\documentclass[pdflatex,sn-mathphys-num]{sn-jnl}
\usepackage{lipsum}
\usepackage{amsfonts}
\usepackage{graphicx}
\usepackage{epstopdf}

\usepackage[numbers]{natbib}

\usepackage{graphicx}%
\usepackage{multirow}%
\usepackage{amsmath,amssymb,amsfonts}%
\usepackage{mathrsfs}%
\usepackage[title]{appendix}%
\usepackage{xcolor}%
\usepackage{textcomp}%
\usepackage{manyfoot}%
\usepackage{booktabs}%
\usepackage{listings}%
\usepackage[utf8]{inputenc}

\usepackage{graphicx}

\usepackage[T1]{fontenc}    
\usepackage{booktabs}       

\usepackage{amsfonts,amssymb,bbm}
\usepackage{amsmath}
\usepackage{enumitem}
\usepackage{graphicx}
\usepackage{xcolor}
\usepackage{adjustbox}
\usepackage[UKenglish]{isodate}
\usepackage{graphicx}
\usepackage{silence}
\usepackage[format=default,font=small,labelfont=it]{caption}
\usepackage{subcaption}
\usepackage{hyperref} 
    \hypersetup{
    	colorlinks = true,
    	linkcolor = blue,
    	anchorcolor = blue,
    	citecolor = blue,
    	filecolor = blue,
    	urlcolor = blue
    }

\usepackage{float}
\usepackage{enumitem}
\usepackage{multirow}
\usepackage{fancyvrb}
\usepackage{rotating}
    \usepackage{tikz}
    \usepackage{tikz-cd} 
    \pgfdeclarelayer{edgelayer}
    \pgfdeclarelayer{nodelayer}
    \pgfsetlayers{edgelayer,nodelayer,main}
    \tikzstyle{new style 0}=[fill={rgb,255: red,255; green,94; blue,247}, draw=black, shape=circle]
    \tikzstyle{pointy}=[fill=white, draw=black, shape=circle]
    \tikzstyle{pointy}=[->]
\usepackage{fontawesome}

\usepackage{amsopn}

\ifpdf
\hypersetup{
  pdftitle={An Example SIURO Article},
  pdfauthor={D. Doe, P. T. Frank, and J. E. Smith}
}

\usepackage{lscape}

\makeatletter
\newcommand{\pushright}[1]{\ifmeasuring@#1\else\omit\hfill$\displaystyle#1$\fi\ignorespaces}
\newcommand{\pushleft}[1]{\ifmeasuring@#1\else\omit$\displaystyle#1$\hfill\fi\ignorespaces}
\makeatother

\DeclareMathOperator{\tr}{tr}

\newcommand{\bE}{\mathbb{E}}
\newcommand{\bP}{\mathbb{P}}

\newcommand{\cY}{\mathcal{Y}}
\newcommand{\cB}{\mathcal{B}}

\renewcommand{\phi}{\varphi}

\newcommand{\rr}{\mathbb{R}}

\newcommand{\process}[3]{{#1}({#2})_{#3}}
\newcommand{\myset}[2]{\lbrace{#1}, \ldots, {#2}\rbrace}
\newcommand{\mySet}[2]{\left\lbrace {#1} \middle\vert {#2} \right\rbrace}

\newcommand{\N}{\mathbb{N}}

\newcommand{\R}{\mathbb{R}}

\newcommand{\eqdef}{\ensuremath{\,\raisebox{-1pt}{$\stackrel{\mbox{\upshape\tiny def.}}{=}$}}\,}

\newcommand{\sdim}{{N_\mathrm{sim}}} 
\newcommand{\pdim}{{N_\mathrm{pos}}} 
\newcommand{\edim}{{N}} 
\newcommand{\ddim}{{N'}} 
\newcommand{\rdim}{{N_\mathrm{ref}}} 
\newcommand{\tdim}{{N_\mathrm{time}}} 
\newcommand{\pldim}{{N_\mathrm{p}}} 
\newcommand{\ldim}{{d_K}} 

\usepackage{amsthm}

\newtheorem{definition}{Definition}
\newtheorem{proposition}{Proposition}
\newtheorem{assumption}{Assumption}
\newtheorem{theorem}{Theorem}
\newtheorem{lemma}{Lemma}

\newtheorem{remark}{Remark}

\NewDocumentCommand{\luca}{mo}{
    \IfValueF{#2}{
                        {{\scriptsize
                            \textcolor{green}{ 
                            \textbf{L:}
                            \textit{{#1}}
                            }
                        }}
        }
    \IfValueT{#2}{
                        \marginnote{{\scriptsize
                            \textcolor{green}{ 
                            \textbf{L:}
                            \textit{{#1}}
                            }
                        }}
        }
                    }
\NewDocumentCommand{\giulia}{mo}{
    \IfValueF{#2}{
                        {{\scriptsize
                            \textcolor{red}{ 
                            \textbf{GL:}
                            \textit{{#1}}
                            }
                        }}
        }
    \IfValueT{#2}{
                        \marginnote{{\scriptsize
                            \textcolor{red}{ 
                            \textbf{GL:}
                            \textit{{#1}}
                            }
                        }}
        }
}
\NewDocumentCommand{\anastasis}{mo}{
    \IfValueF{#2}{
                        {{\scriptsize
                            \textcolor{violet}{ 
                            \textbf{A:}
                            \textit{{#1}}
                            }
                        }}
        }
    \IfValueT{#2}{
                        \marginnote{{\scriptsize
                            \textcolor{violet}{ 
                            \textbf{A:}
                            \textit{{#1}}
                            }
                        }}
        }
                    }
                    
\NewDocumentCommand{\cody}{mo}{
    \IfValueF{#2}{
                        {{\scriptsize
                            \textcolor{orange}{ 
                            \textbf{A:}
                            \textit{{#1}}
                            }
                        }}
        }
    \IfValueT{#2}{
                        \marginnote{{\scriptsize
                            \textcolor{orange}{ 
                            \textbf{A:}
                            \textit{{#1}}
                            }
                        }}
        }
                    }

\NewDocumentCommand{\yannick}{mo}{
    \IfValueF{#2}{
                        {{\scriptsize
                            \textcolor{cyan}{ 
                            \textbf{Y:}
                            \textit{{#1}}
                            }
                        }}
        }
    \IfValueT{#2}{
                        \marginnote{{\scriptsize
                            \textcolor{cyan}{ 
                            \textbf{Y:}
                            \textit{{#1}}
                            }
                        }}
        }
                    } 

\definecolor{darkgreen}{rgb}{0.0, 0.2, 0.13}
\NewDocumentCommand{\xuwei}{mo}{
    \IfValueF{#2}{
                        {{\scriptsize
                            \textcolor{darkgreen}{ 
                            \textbf{X:}
                            \textit{{#1}}
                            }
                        }}
        }
    \IfValueT{#2}{
                        \marginnote{{\scriptsize
                            \textcolor{darkgreen}{ 
                            \textbf{X:}
                            \textit{{#1}}
                            }
                        }}
        }
                    }

\usepackage{appendix}

\usepackage{cleveref}
\newcounter{termcounter}
\renewcommand{\thetermcounter}{\Roman{termcounter}}
\crefname{term}{term}{terms}
\creflabelformat{term}{#2\textup{(#1)}#3}

\makeatletter
\def\term{\@ifnextchar[\term@optarg\term@noarg}
\def\term@optarg[#1]#2{%
  \textup{#1}%
  \def\@currentlabel{#1}%
  \def\cref@currentlabel{[][2147483647][]#1}%
  \cref@label[term]{#2}}
\def\term@noarg#1{%
  \refstepcounter{termcounter}%
  \textup{(\thetermcounter)}%
  \cref@label[term]{#1}}
\makeatother

\crefname{lemma}{lemma}{lemmata}
\Crefname{lemma}{Lemma}{Lemmata}
\crefname{assumption}{assumption}{assumptions}
\Crefname{assumption}{Assumption}{Assumptions}
\crefname{example}{Example}{Examples}
\crefname{proposition}{Proposition}{Proposition}

\usepackage{comment}
\usepackage{microtype}

\theoremstyle{thmstyletwo}%
\theoremstyle{thmstylethree}%

\usepackage[table]{xcolor}
\usepackage{booktabs}
\usepackage{array}
\usepackage{adjustbox}

\begin{document}

\title{Transformers Can Solve Non-Linear and Non-Markovian Filtering Problems in Continuous Time For Conditionally Gaussian Signals}

\author[3,4]{\fnm{Blanka} \sur{Horvath}}
\email{horvath@maths.ox.ac.uk}

\author*[1,2]{\fnm{Anastasis} \sur{Kratsios}}\email{kratsioa@mcmaster.ca}

\author[3,4]{\fnm{Yannick} \sur{Limmer}}\email{yannick.limmer@maths.ox.ac.uk}

\author[1,2]{\fnm{Xuwei} \sur{Yang}}\email{henryyangxuwei@gmail.com}

\affil*[1]{\orgname{McMaster University}, \country{Canada}}

\affil*[2]{\orgname{The Vector Institute}, \country{Canada}}

\affil*[3]{\orgname{University of Oxford}, \country{UK}}

\affil*[4]{\orgname{Oxford-Man Institute}, \country{UK}}




\abstract{The use of attention-based deep learning models in stochastic filtering, e.g.\ transformers and deep Kalman filters, has recently come into focus; however, the potential for these models to solve stochastic filtering problems remains largely unknown. The paper provides an affirmative answer to this open problem in the theoretical foundations of machine learning by showing that a class of continuous-time transformer models, called \textit{filterformers}, can approximately implement the conditional law of a broad class of non-Markovian and conditionally Gaussian signal processes given noisy continuous-time (possibly non-Gaussian) measurements. Our approximation guarantees hold uniformly over sufficiently regular compact subsets of continuous-time paths, where the worst-case 2-Wasserstein distance between the true optimal filter and our deep learning model quantifies the approximation error. Our construction relies on two new customizations of the standard attention mechanism: The first can losslessly adapt to the characteristics of a broad range of paths since we show that the attention mechanism implements bi-Lipschitz embeddings of sufficiently regular sets of paths into low-dimensional Euclidean spaces; thus, it incurs no ``dimension reduction error''. The latter attention mechanism is tailored to the geometry of Gaussian measures in the $2$-Wasserstein space. Our analysis relies on new stability estimates of robust optimal filters in the conditionally Gaussian setting.}




\maketitle

\section{Introduction}
\label{s:Introduction_Deep_Filtering}

In a wide variety of scientific domains, from medicine~\cite{sameni2006filtering} and evolutionary biology~\cite{lillacci2010parameter} to mathematical finance~\cite{javaheri2003filtering,wells2013kalman}, one is often interested in estimating an unobservable \textit{signal} process $X_{\cdot}$ based on noisy \textit{observations} $Y_{\cdot}$. This leads to the classical problem of \textit{optimal (stochastic) filtering}, which seeks the best reconstruction of the signal process given the observed data. 
For instance, in mathematical finance, market sentiment can act as a signal, and the impact of the sentiment on market prices can act as observations. Although this problem is well studied and admits a unique solution in the form of an infinite-dimensional recursion, e.g.~\cite{stratonovich1959optimum,stratonovich1960application,Sirjaev1965Filtering,zakai1969optimal,BainFilteirnBook}, the resulting non-linear filters are nearly always computationally intractable and so are traditional approximation schemes; e.g.~particle filters~\cite{del1996nonlinear,del1998measure,del1999interacting,ning2023iterated} or linearized surrogates~\cite{gonon2020linearized}. This challenge has sparked the exploration of deep learning approaches to stochastic filtering, motivated by the success of neural networks in most challenging computational problems, which have shown promising empirical performance~\cite{KrishnanShalitSontag15,Balint_FinalPaper_2023,bach2024filtering}. Nevertheless, the fundamental question: \textit{``Can neural networks solve the stochastic filtering problem?''} remains open.

This paper provides an \textit{affirmative} answer to this question for a broad range of non-Markovian signal processes evolving according to non-linear dynamics, given general (possibly non-Gaussian) observations, in continuous time up to an arbitrarily small approximation error.  
However, only the latter of these two can be directly measured. 
More precisely, here both of the processes are often assumed to have continuous paths and are defined on a filtered probability space $(\Omega,\mathcal{F},(\mathcal{F}_t)_{t\in [0:T]},\mathbb{P})$.  
The objective of the \textit{robust stochastic filtering problem}, studied by \cite{clark1978design,Kushner_1979_RobustDiscreteSTate,Davis_1980MultiplicativeFunctionalNonLinearFiltering,DavisFiltering,clark2005robust,CrisanDielhFrizOberhauser_2013_AAP}, is to identify a continuous function $f_t:C([0:t],\mathbb{R}^{d_Y})\rightarrow \mathcal{P}(\mathbb{R}^{d_X})$, with $d_X, d_Y \in \N$ for $t>0$, satisfying 
\begin{align}
    \label{goal:Robust_Filtering}
        f_t(Y_{[0:t]})
    =
        \mathbb{P}(X_t\in \cdot|\mathcal{F}_t^Y)
    ,
\end{align}
where $\mathbb{P}(X_t\in \cdot|\mathcal{F}_t^Y)$ is the conditional law of the $d_X$-dimensional signal process $X_t$ given the $\sigma$-algebra generated by $(Y_s)_{s\in [0:t]}$. The key innovation in \eqref{goal:Robust_Filtering} is the continuity, and the uniqueness, of $f_t$.  In contrast, a Borel $f_t$ satisfying \eqref{goal:Robust_Filtering} exists by elementary measure-theoretic \citep[Theorem 6.3]{Kallenberg_2021__FoundationsOfModernProbability}. 
The continuity of $f_t$ is qualified by equipping the set of Borel probability measures on $\mathbb{R}^{d_X}$, denoted by $\mathcal{P}(\mathbb{R}^{d_X})$, with the weak topology and by equipping the set of continuous paths from $[0:t]$ to $\mathbb{R}^{d_Y}$, denoted by $C([0:t],\mathbb{R}^{d_Y})$, with the uniform norm.

When one has access to a \textit{robust representation}~\eqref{goal:Robust_Filtering}, then they can reliably predict the conditional law of $X_t$ even subject to imperfections on the observed historical data in $y_{[0:t]}\in C([0:t],\mathbb{R}^{d_Y})$.  These robust representations are particularly invaluable in mathematical finance, where continuous streams of data are often noisy.  More broadly, stochastic filters are indispensable in situations where latent parameters influence or obscure market factors.  Applications include the computation of optimal investment under partial information \cite{lakner1998optimal, BjorkDavis}, the estimation of volatility from observed intra-day stock prices \cite{Barndorff-NielsenShephard02}, estimation of interest rates \cite{chen1993maximum,duan1999estimating,BabbsNowman99,JavaheriLautierGalli03}, estimation of spot price estimation for commodity futures \cite{schwartz1997stochastic,schwartz2000short,ManoliuTompaidis02,LautierGalli04,Cody2007parameter}, hedging of credit derivatives \cite{frey2010pricing,FreySchmidt_2012_PricingHedging}, estimation of equilibria under asymmetric information \cite{ccetin2018financial}, and calibration of option pricing models \cite{LindstroemStroejbyBrodenWiktorssonHolst08,LindstroemWiktorsson14}.  Pathwise, or so-called robust, formulations were later derived in \cite{clark2005robust,crisan2013robust}.  Indeed, the abilities of stochastic filters to estimate unobservable variables, track market dynamics, and improve the precision of financial models have made them a staple tool in the quantitative finance toolbox, with several books treating the role of stochastic filtering in finance \cite{harvey1990forecasting,wells2013kalman,bhar2010stochastic,date2011linear,remillard2013statistical}.  

Although the stochastic filtering problem is a well-understood mathematical problem.   Many of the fundamental questions in stochastic filtering, such as existence and dynamics for the evolution of the conditional law of signal process $X_t$ given the realizations of the measurement process $Y_{\cdot}$, were solved in a series of classical \cite{stratonovich1959optimum,stratonovich1960application,zakai1969optimal,Sirjaev1965Filtering,Sirjaev_1978Filtering} and contemporary~\cite{zhang2024numerical} papers.  Nevertheless, the infinite-dimensionality of general stochastic filters, being measure-valued path-dependent processes, makes the problem computationally intractable. 
The exception to this rule is so-called finite-dimensional filters. These encapsulate rare situations in which the dynamics of $X_{\cdot}$ and $Y_{\cdot}$ give way to optimal filters that are finitely parameterized.  Examples include the \textit{Kalman-Bucy filters} \cite{Kalman60,KalmanBucy61} where closed-forms are derived under the assumption of Markovian Gaussian noise and affine OU-type dynamics of $X_{\cdot}$ and $Y_{\cdot}$, the \cite{wonham1964some} filter where all involved quantitative are finite-state Markov processes, and the Be\v{n}es filter \cite{benevs1970existence} which relies a particular set of one-dimensional dynamics.

The computational intractability of the general filtering problem leads to the use of approximately optimal filters.  These approaches include particle filters which aim at dynamically approximating optimal filters using an evolving interacting particle system \cite{del1997nonlinear,djuric2003particle,DelMoralBook_2013}, linear relaxations of the optimal filtering functional for affine processes which can be numerically computed by solving specific stochastic Riccati equations \cite{LukasJosef_LinearizedFiltering}, or by projection of the optimal filter onto finite-dimensional manifolds of exponential families \cite{brigo1998differential,brigo1999approximate,armstrong2023optimal} where it can be tracked using finitely parameterized representations.  The latter two approaches work well if the coupled system $(X_{\cdot},Y_{\cdot})$ follows the postulated affine dynamics, and the former particle-filtering approach is well-suited to low-dimensional settings. Nevertheless, one would like to have access to (approximate) optimal finite-dimensional filters for a broader range of situations.    

One promising path is via deep learning, as deep neural networks can solve a variety of previously intractable numerical problems. The expressivity of neural networks on classical learning problems has since inspired the use of deep learning approaches to numerical stochastic filtering problems, e.g.~\cite{Balint_FinalPaper_2023,Ariel_2023Filtering,pmlr-v80-ryder18a,de2019gru,herrera2021neural,krach2022optimal,bishop2023recurrent,heiss2024nonparametric}, and the development of measure-valued deep learning models e.g.~\cite{Acciaio2022_GHT}. Central to these developments was the emergence of the Deep Kalman filter (DKF) \cite{KrishnanShalitSontag15}.

\paragraph{Connections to Kalman Filters and the \textit{``Deep} Kalman Filter''.}
The original \textit{Deep Kalman Filter}(DKF)~\cite{KrishnanShalitSontag15},  which served as the starting point for this study is a Bayesian deep learning model which is loosely related to the classical Kalman filter through the serendipitous naming and through its \textit{conditionally Gaussian} nature---although DKF, in contrast to the latter, is not restricted to linear settings and hence could be more broadly applicable.  
The name (DKF) draws attention already for its seemingly oxymoronic title\footnote{The expressive power of neural networks would be unnecessary for solving a linear Gaussian (Kalman) filtering problem, which is already well-known to be resolved by the classical \textit{Kalman-Bucy} filter~\cite{KalmanBucy61}.}. 
It is the non-linearity (stemming from the neural network) which has led to its well-documented success in non-linear settings across a range of domains, from video compression~\cite{jaimungal2022reinforcement} to censoring~\cite{hosseinyalamdary2018deep} and mathematical finance~\cite{jaimungal2022reinforcement}. 
Given this (highly) non-linear nature of DKFs and the multitude of domains they have been applied to, it is a natural question to ask whether \textit{Deep Kalman Filters} applicability is indeed mathematically justifiable beyond the classical linear settings, which would usually be typical for Kalman (-Bucy) filters. In this paper, we answer this question affirmatively by solving the following rephrased problem 
\[
    \mbox{\textit{Which classes of (robust) non-linear filtering problems can DKF-type approximately solve?}}
\]
\paragraph{Main Result.}
This paper presents an attention-guided modification of the original DKF model~\cite{KrishnanShalitSontag15}, which we dubbed the \textit{Filterformer} (FF).  Our main result shows that this new model can uniformly approximate (Theorem~\ref{thrm:MAIN}) the robust representation $f_t$ in~\eqref{goal:Robust_Filtering} of the optimal filters in a broad class of \textit{non-linear} and \textit{non-Markovian} signal- and observation processes.  Specifically, we consider the partially-observed systems in~\citep[Chapter 12]{LipShiryaev_II_2001}, which are significant generalizations of the linear Gaussian structure considered in the classical Kalman(-Bucy) filter~\cite{Kalman60,KalmanBucy61}, while maintaining only the core structure required for conditional Gaussianity of the optimal filter. 
We guarantee that Filterformer is \textit{universal} in the class of optimal filters for any partially-observed system
\begin{align}
\label{eq:filtering_dynamics__signal}
dX_t & = [a_0(t,Y_{[0:t]}) + a_1(t,Y_{[0:t]}) X_t]dt + \sum_{i=1}^2\, b_i(t,Y_{[0:t]})dW^{(i)}_t \\
\label{eq:filtering_dynamics__observation}
dY_t & = [A_0(t,Y_{[0:t]}) + A_1(t,Y_{[0:t]}) X_t]dt + \sum_{i=1}^2\, B_i(t,Y_{[0:t]})dW^{(i)}_t
\end{align}
where $a_0$ and $A_0$ respectively take values in $\mathbb{R}^{d_X}$ and $\mathbb{R}^{d_Y}$, and where $a_1,A_1,b_1,b_2,B_1,B_2$ are matrix-valued of respective dimensions $d_X\times d_X$, $d_Y\times d_X$, $d_X\times d_X$, $d_X\times d_Y$, $d_Y\times d_X$, and $d_Y\times d_Y$, and the entries of $a_0, a_1, b_0, b_1, A_0, A_1, B_0, B_1$ are measurable nonanticipative functionals on the measurable space $\smash{([0:T] \times C([0:T], \R^{d_Y}), \cB_{[0:T]} \times \cB_T^{d_Y})}$, and where $\smash{\cB_t^d  \eqdef  \sigma(C([0: t], \R^d))}$ denotes the $\sigma$-algebra generated by continuous paths on $[0: t]$ to $\R^d$.   
We will always assume that the filtration $\mathbb{F}$ is right-continuous and that the probability space $(\Omega,\mathcal{F},\mathbb{F},\mathbb{P})$ supports independent $\mathbb{F}$-adapted Brownian motions, $\smash{W^{(1)}_{\cdot}\eqdef (W_t^{(1)})_{0\le t\le T}}$ and $\smash{W^{(2)}_{\cdot}\eqdef (W^{(2)}_t)_{0\le t\le T}}$, of respective dimensions $d_X$ and $d_Y$, for positive integers $d_X$ and $d_Y$; all technical requirements are detailed in \Cref{ass:technical_evolution_conditions} below.

{{
\paragraph{Technical Contributions.}
We provide Lipschitz-stability guarantees for the robust optimal filter associated with the partially-observed system~\eqref{eq:filtering_dynamics__signal}–\eqref{eq:filtering_dynamics__observation}, in the spirit of~\cite{clark2005robust,crisan2013robust}, but in our case the target distribution is quantified using the $2$-\textit{Wasserstein} distance $\mathcal{W}_2$ (Proposition~\ref{prop:LocLipReg_OptimFilter}). Once the Lipschitz regularity of the robust filter is established in the $\mathcal{W}_2$ sense, we prove a \textit{new} universal approximation theorem showing that our \textit{Filterformer} model can approximate continuous mappings from a suitable class of paths to multivariate Gaussian measures (Proposition~\ref{prop:Universal_Approximation_Theorem__PathToNd}).  
Unlike the Gaussian measure-valued approximation result in~\cite{arabpour2024low}, we do not need the Gaussian measures to be non-degenerate, nor do we rely on an information-geometric distance to quantify the approximation error; but do so using the standard $\mathcal{W}_2$ distance.  

Finally, we highlight the speed of our approximation rates, which we find noteworthy  
considering that the input space is infinite-dimensional (cf.~\cite{lanthaler2023curse}) and no smoothness requirements were stated for the target function (cf.~\cite{adcock2022near}).  Our fast convergence rates are made possible by our pathwise-attention mechanism which, as we show, can adaptively and \textit{losslessly} encode suitable compact sets $K$ of continuous paths by implementing bi-Lipschitz embedding of $K$ into finite-dimensional space (Proposition~\ref{prop:Lossless_Encoding}) whose embedding dimension is \textit{independent of the approximation error}.
}}


\subsection{The Filterformer Model}
\label{s:Introduction_Deep_Filtering__ss:ModelDetails}
The proposed Filterformer (FF), illustrated in \Cref{fig:FFModel}, generates predictions on the finite-dimensional metric space $(\mathcal{N}_{d_X},\mathcal{W}_2)$ whose points are non-singular Gaussian measures, and whose distance between points is quantified by the $2$-Wasserstein distance.  These predictions are generated from the continuous input paths via the following three-phase process.  
\begin{figure}[hpt!]
    \centering
        \centering
        \includegraphics[width=1\linewidth]{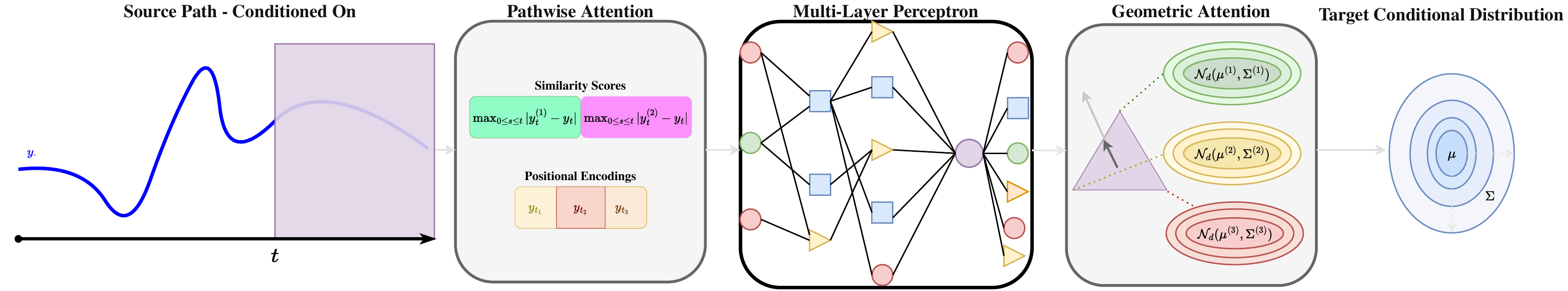}
    \caption{Architecture of a Filterformer which can approximate the conditional law of the signal process $X_{\cdot}$ given source paths, possibly (but not necessarily) taken by the observation process.}
    \label{fig:FFModel}
\end{figure}

\begin{enumerate}[wide,label=\emph{Phase \arabic*}.]
    \item The observed infinite-dimensional continuous path is encoded into a finite-length real vector.  Surprisingly, our proposed encoding layer, coined the \textit{pathwise attention mechanism}, adaptively implements a stable and lossless encoding for a broad class of compact subsets of $K$ paths in $C([0:T],\mathbb{R}^{d_Y})$.  By \textit{losslessness}, we mean that it is injective for a broad class of compact subsets of $K$, including classes of piecewise linear paths, any finite (training) set, and any $K$ which is isometric to a closed Riemannian manifold.  By stability, we mean it in the sense of constructive approximation theory, e.g.~\cite{MR4433109,MR4577400}, namely that the map encoding layer is Lipschitz and so is its inverse.  \textit{Stability} is desirable since it implies that minor numerical errors, e.g.\ rounding, do not lead to drastically different downstream predictions.  By \textit{adaptivity} we mean that the encoding layer is devoid of any projection onto a finite-dimensional (Schauder) basis, and its parameters can be chosen to suit the specific geometry of the compact set of paths $K$ on which the optimal filter is approximated.
\begin{figure}[H]
    \centering
        \centering
        \includegraphics[width=.35\linewidth]{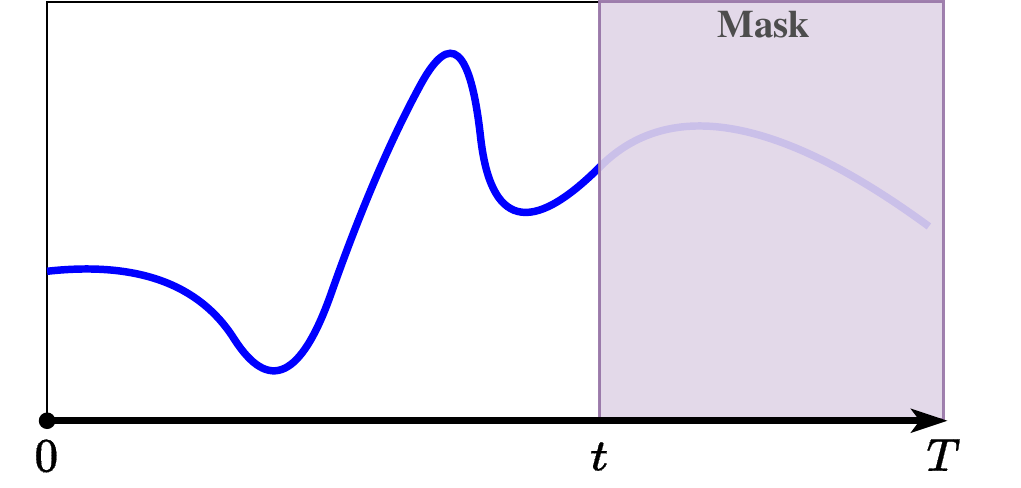}
    \caption{The \textit{Filterformer} model 
    has two inputs: a time $t$ and a continuous path $y_{\cdot}$ defined up to some time $T\ge t$.  The first parameter $t$ acts as a mask hiding the future evolution of the observation (input) path $y_{\cdot}$ beyond the current time $t$.  We note that there is no loss of generality in assuming that the path $y_{\cdot}$ is defined up to time $T$, since any path can be trivially extended beyond the current time $t$ by setting $y_{s}=y_t$ for all $s\in [t,T]$; e.g.\ as in the Functional It\^{o} Calculus \cite{MR3939653,MR2782350,MR3059194}.
    \\
    The mask, i.e.\ the parameter $t$, fills the analogous role to recursions in classical stochastic filters.  By varying $t$, the prediction of the FFs evolves into the future without having to retrain the FF model.}
    \label{fig:Mask}
\end{figure}
    \item The FF then processes the vectors encoding the observed paths and generates ``deep features'' which are then passed along to the output layer for prediction generation.  This phase leverages the efficient approximation capacity of MLPs, e.g.~\cite{yarotsky2018optimal,pmlr-v125-kidger20a,kratsios2021universal,LuShenYangZhang_2021_UATRegularTargets}, to flexibly process the encoded features in a task-specific manner, by adaptive to specific dynamics of the coupled system~\eqref{eq:filtering_dynamics__signal}-\eqref{eq:filtering_dynamics__observation}.
    
    \item In its final processing phase, the FF decodes the ``deep features'' generated by the MLP to $\mathcal{N}_{d_X}$-valued predictions, which are then used to approximately implement the predictions of the optimal filter~\eqref{goal:Robust_Filtering}.  This decoder is a modified instance of the \textit{geometric attention mechanism} of \cite{Acciaio2022_GHT}.  Intuitively, this layer translates Euclidean data to a point in a generalized geodesic convex hull of a finite number containing the image of $K$ under the optimal filter.
\end{enumerate}

\subsection{Organization of Paper}
\Cref{s:Setting} formalizes the basic required regularity conditions on the coupled system~\eqref{eq:filtering_dynamics__signal}-\eqref{eq:filtering_dynamics__observation} and it overviews any background material.  \Cref{s:Model} rigorously introduces the relevant FF model.  \Cref{s:Main_Results} contains the paper's main theoretical results, as well as an overview of the proof methodology.  Detailed proofs of our results are relegated to \Cref{s:Proof}.

\section{The Setting}
\label{s:Setting}
This section formalizes the setting in which our analysis takes place.  We first specify the dynamics of the stochastic processes in the filtering problem.  Next, the source space of paths considered is introduced.  Subsequently, we formulate the target space of probability measures wherein the optimal filter lies.

\subsection{Regularity Conditions on the Partially-Observed System}
We maintain the following basic regularity assumptions on the dynamics of the coupled system~\eqref{eq:filtering_dynamics__signal}-\eqref{eq:filtering_dynamics__observation}.
\begin{assumption}[Regularity Conditions: Dynamics]
\label{ass:technical_evolution_conditions}
We will assume the following uniform bounds:
\begin{align*}
    \vert \process{a_1}{t, y}{i, j}\vert \le L
    \mbox{ and }
    \vert \process{A_1}{t, y}{k, j}\vert \le L
\end{align*}
for $y \in C([0:t], \R^{d_Y})$, $t \in [0:T]$, $i,j=1,\dots,d_X$, and $k=1,\dots,d_Y$.  
Further, we require the following integrability conditions:\footnote{In particular, the integrability conditions imply that $\bE\big[\sum_{j = 1}^{d_X}(X_0)_j^4\big] < \infty$.}
\begin{enumerate}[label=(\roman*),leftmargin=1cm]
    \item $\int_0^T \bE[\process{a_0}{t, Y_{[0:t]}}{i}^4 + \process{b_1}{t, Y_{[0:t]}}{ij}^4 + \process{b_2}{t, Y_{[0:t]}}{ij}^4] dt < \infty$,
    \item $\int_0^T \bE[\process{b_2}{t, Y_{[0:t]}}{ij}^4] dt < \infty$,
\end{enumerate}
hold for $i,j=1,\dots,d_X$, and $k,l=1,\dots,d_Y$ and all paths $y_{[0:T]} \in C([0:T], \R^{d_Y})$. We define 
\begin{align*}
    B \circ B \eqdef B_1B_1^\top + B_2B_2^\top, \quad 
    b \circ B \eqdef b_1B_1^\top + b_2B_2^\top, \quad 
    b \circ b \eqdef b_1b_1^\top + b_2b_2^\top.
\end{align*}
and require that the matrix $B \circ B$ is uniformly non-singular, that is, its inverse is uniformly bounded; there are constants $L_1, L_2 \in \R$ as well as a non-decreasing right-continuous function $K: [0:T] \to [0:1]$ such that for every $x, y \in C([0:T], \R^{d_Y})$ and every $k=1,2$, and $i,j=1,\dots,d$ it holds for all $0 \leq t\leq T$ that 
\begin{enumerate}[label=(\roman*), resume,leftmargin=1cm]
    \item $\vert (B_k)_{i,j}(t, x) - (B_k)_{i,j}(t, y) \vert^2 
        \le
        L_1 \int_0^{t} \vert x_s - y_s \vert^2 dK(s) + L_2 \vert x_t - y_t \vert^2$,
    \item $(B_k)_{i,j}(t, x)^2  
        \le 
        L_1 \int_0^t \vert (1 + \vert x_s \vert^2 ) dK(s) + L_2 (1 +\vert x_t \vert^2)$,
    \item $\int_0^T \bE[\vert \process{A_1}{(t, Y_{[0:t]})}{i, j}(X_t)_j \vert] dt < \infty$,
    \item $\bE[\vert (X_t)_j \vert] < \infty, \quad \forall t \in [0:T]$,
    \item $\bP\big( \int_0^T (\process{A_1}{t, Y_{[0:t]}}{i, j} \bE[(X_t)_j \vert \cY_t])^2dt  < \infty \big) = 1$,
\end{enumerate}
for indices $i=1,\dots,d_Y$ and $j=1,\dots,d_X$, where $\vert \cdot \vert$ is the Euclidean distance on $\R^{d_Y}$.  

    We further impose the following assumptions on the dynamics \eqref{eq:filtering_dynamics__signal}, \eqref{eq:filtering_dynamics__observation}:
    \begin{enumerate}[label=(\roman*), resume,leftmargin=1cm]
        \item Local Lipschitz continuity in the path component uniformly in time\footnote{
            By \emph{local Lipschitz continuity uniformly in time} we mean that there exists a constant that is a local Lipschitz-constant for the path at all times $t \in [0,T]$.
        }, as well as global Lipschitz continuity in the time component of $a_0$, $a_1$, $b \circ b$, $b \circ B$, $ A_0$, $A_1$, $(b \circ B)$, $(B \circ B)^{-1}$ with respect to the $l^2$-norm, or Frobenius norm if matrix-valued. \label{asm:Regulartiy__1}
        \item Positive semi-definiteness of $(b \circ b)(t, y_{[0:t]})- (b \circ B)(B \circ B)(b \circ B)^\top(t, y_{[0:t]})$ for all times $t \in [0:T]$ and paths $y_{[0:T]} \in C^1([0:T], \R^{d_X})$. \label{asm:Regulartiy__2}
        \item \label{asm:Regularity__3} Let $G_t(y_{[0:t]})$ be a solution of $\partial_tG_t(y_{[0:t]}) = \Tilde{a}_1(t, y_{[0:t]})G_t(y_{[0:t]})$ with $G_0(y_{[0:t]}) = I_{d_X}$ for any path $y_{[0:T]} \in C^1([0:T], \R^{d_X})$, where 
        \begin{align*}
            \Tilde{a}_1^\top(t, y_{[0:t]}) \eqdef a_1(t, y_{[0:t]}) - (b \circ B)(B \circ B)A_1(t, y_{[0:t]}).
        \end{align*}
        There exist constants $K_1, K_2 > 0$ s.t. uniformly $K_1 \leq \tr(G_t(y_{[0:t]})) \leq K_2$ as well as $K_3 > 0$ s.t. $\tr(A_1^\top(B \circ B)^{-1}A_1(t, y_{[0:t]})) \leq K_3$. 
    \end{enumerate}

    Eventually we require for the initial conditions that 
    \begin{enumerate}[label=(\roman*), resume,leftmargin=1cm]
        \item the conditional distribution of $X_0$ given $Y_0$ is normal,
        \item the covariance matrix $\Sigma_0  \eqdef  \mathbb{E}[(X_0 - \mu_0)(X_0 - \mu_0)^\intercal \vert Y_0]$ is positive definite, where $\mu_0  \eqdef  \mathbb{E}[X_0 \vert Y_0]$,
        \item \label{asm:Regularity__LipInit} Lipschitz continuity of $\mu_0, \Sigma_0$ in $Y_0$ w.r.t. the $\Vert \cdot \Vert_2$-norm holds. 
    \end{enumerate}
\end{assumption}

\subsection{Source Spaces: Regular Compact Subsets of Path Space}
\label{s:Setting__ss:InputSpace}
Fix a time-horizon $T>0$.  For every $0<t\le T$, let $C([0:t],\mathbb{R}^{d_Y})$ denote the set of continuous ``paths'', i.e.\ functions, from $[0:t]$ to the $d$-dimensional Euclidean space $\mathbb{R}^{d_Y}$.  We equip this set with the uniform norm, making it a Banach space, where the norm of a path $y_{\cdot}\in C([0:t],\mathbb{R}^{d_Y})$ is defined by 
$
        \|y\|_t
    \eqdef 
        \max_{0\le s\le t}\, \|y_s\|_2
$, 
where $\|\cdot\|_2$ is the Euclidean norm on $\mathbb{R}^{d_Y}$.  
For every $0<t\le T$, similarly to the Horizontal Extension of a path in the Functional It\^{o} calculus \citep[Section 5.2.1]{fournierRama}, we may canonically embed any path $y\in C([0:t],\mathbb{R}^{d_Y})$ into a path $\bar{y} \in C([0:T],\mathbb{R}^{d_Y})$ via by extending the ``frozen version after time $t$'' defined by
\[
    \bar{y}(s)\eqdef \begin{cases}
        y(s) &: \mbox{ if } 0\le s \le t\\
        y(t) & : \mbox{ if } t<s\le T.
    \end{cases}
\]
Moreover, note that the map $\bar{\cdot}: C([0:t],\mathbb{R}^{d_Y}) \to C([0:T],\mathbb{R}^{d_Y})$ is a linear isometric embedding of the Banach spaces, since
$
\|\bar{y}\|_T = \max_{0\le s\le T}\, \|y(s)\|_2
= \max_{0\le s\le t}\, \|y(s)\|_2 = \|y\|_t
$, 
since $y(s)=y(t)$ for all $t<s\le T$.  Conversely, the restriction $\bar{y}_{[0:t]}$ of a path $\bar{y}\in C([0:T],\mathbb{R}^{d_Y})$ to any shorter time-interval $[0:t]$ for $0\le t\le T$ defines a non-expansive linear operator from $C([0:T],\mathbb{R}^{d_Y})$ to $C([0:t],\mathbb{R}^{d_Y})$ since
$
        \|\bar{y}_{[0:t]}\|_t 
    = 
       \max_{0\le s\le t}\, \|\bar{y}_{[0:t]}(s)\|_2 
    \le
        \max_{0\le s\le T}\, \|\bar{y}_{[0:t]}(s)\|_2 
    =
        \max_{0\le s\le T}\, \|\bar{y}(s)\|_2 
    =
        \|\bar{y}\|_T
$. 
Therefore, in what follows, we will always consider the domain of our path space to be the Banach space $C([0:T],\mathbb{R}^{d_Y})$.  
As show in \citep[Proposition A.3]{lanthaler2023curse}, even in finite dimensions, there are rather regular functions which cannot be approximated by standard deep neural networks parameterized determined by a number of trainable parameters which is a polynomial in the reciprocal approximation error.  In other words, the curse of dimensionality is generally unavoidable when approximating rather regular functions between finite-dimensional Banach spaces, and thus the problem can only be exacerbated in finite-dimensions; see \citep{galimberti2022designing}.   

The curse of dimensionality cannot be broken, but can typically be avoided either by restricting the class of functions being approximated, e.g.\ in~\cite{adcock2022near,marcati2023exponential,siegel2023characterization,gonon2023approximation}, or the regularity of the compact subsets of the input space on which the uniform approximation is to hold; e.g.\ in~\cite{kratsios2021universal,lu2021deep,kratsios2022universal}.  Recent advances in infinite-dimensional approximation of functions not taking values in a linear space -- as is the case in our results -- show that such restrictions are a sufficient requirement for obtaining universality, see \cite{kratsios2023transfer}. However, the necessity of such restrictions on compact subsets is still unknown.  The explicit effects of restricting the (fractal) dimension and diameter of compact sets of which a deep neural network approximation is to hold is explicitly studied in \citep[Proposition 3.10]{Acciaio2022_GHT}.

We adopt the second approach, since we cannot restrict the function class, which is determined by the stochastic filtering problem.  
We therefore restrict to compact subsets $K$ of the input space which are regular, in that they are either isometric to some compact Riemannian manifold or they are comprised of piecewise linear paths with finitely many pieces.  In particular, we exclude fractal-like subsets of the path space $C([0:T],\mathbb{R}^{d_Y})$ which need not be compressible into finitely many dimensions without ``loosing information'' (i.e.\ for which there is no bi-Lipschitz embedding into a finite-dimensional normed space).  We denote the Riemannian volume of a connected $(C^2)$ Riemannian manifold $(\mathcal{M},g)$ by $\operatorname{Vol}(\mathcal{M},g)$, its geodesic distance function is $d_g$.
\begin{assumption}[Domain Regularity]
\label{ass:reg_compact}
Fix a compact $K\subseteq C([0:T],\mathbb{R}^{d_Y})$ and a ``latent dimension $\ldim\in \mathbb{N}_+$''.
Suppose that any of the following hold:
\begin{enumerate}[label=(\roman*),leftmargin=1cm]
    \item  \label{ass:reg_compact_finite}\emph{Finite Domains:} $K$ is finite.
    \item \label{ass:reg_compact_linear} \emph{Piecewise Linear Domains:} There are $\pldim \in \mathbb{N}_+$, $0=t_0<\dots<t_\pldim=T$, and $C_K>0$ such that $y_{\cdot}\in K$ if and only if: $y_0=0$, $\max_{i=1,\dots,\pldim}\,|y(t_i)|\le C_K$, and  
            for $i=0,\dots,\pldim-1$ there is a $d\times d$ matrix $A^{(i)}$ and a $b^{(i)}\in \mathbb{R}^{d_Y}$ satisfying $
                y(t)=A^{(i)}\,t+b^{(i)}
            $
            $for all 
                t\in (t_i,t_{i+1})
            .
            $
    \item \label{ass:reg_compact_riemann}\emph{Smooth Domains:} There exists a compact and connected $\ldim$-dimensional Riemannian manifold $(\mathcal{M},g)$ for which $(\mathcal{M},d_g)$ is isometric to $(K,\|\cdot\|_T)$.
\end{enumerate}
\end{assumption}
Evidently, their domains $K$ satisfying \Cref{ass:reg_compact} \ref{ass:reg_compact_finite} are easily exhibited, and they correspond to any \textit{training set}.  One can interpret results in this case as paralleling interpolation results for deep learning, e.g.\ \cite{CristaJosefMartin_2020_UnivInterpol,vershynin2020memory,KratsiosDevarnotDokmanic_SmallTransformers_JMLR2023}, but allowing for some arbitrarily small error as in \cite{park2021provable}.  Compact subsets of the path-space $C([0:T],\mathbb{R}^{d_Y})$ satisfying \Cref{ass:reg_compact} \ref{ass:reg_compact_riemann} are also plentiful.  Indeed, using Banach-Mazur theorem \cite{BanachMazurTheoremBook1955} one can show that $C([0:T],\mathbb{R}^{d_Y})$ contains isometric copies of any compact metric space, and in particular, every compact and connected Riemannian manifold (see Proposition~\ref{prop:Non_Vaccousness} in the supplement).

We obtain explicit estimates for the case where $(K,d_{\cdot})$ satisfies \Cref{ass:reg_compact} \ref{ass:reg_compact_riemann}, if the manifold $\mathcal{M}$ is topologically regular, in that sense that it is \textit{aspherical}.
This means that their homotopy groups $\pi_i(K)$ all vanish for all indices $i\ge 2$, e.g.\ if $\mathcal{M}$ is a torus. See e.g. \citep[Chapter 7]{spanier1989algebraic} for definitions and details on homotopy groups.  
Domains $K$ satisfying \Cref{ass:reg_compact} \ref{ass:reg_compact_linear} include any piecewise linear interpolation of real-world financial time-series data, or simulated data.  This is because only a finite number of inflection points can be observed, and prices are not unbounded.  In this way, \Cref{ass:reg_compact} summarizes a rich family of compact subsets of the path-space $C([0:T],\mathbb{R}^{d_Y})$ which possess a sufficient degree of structure to be losslessly encoded into finitely many dimensions by our model's encoding layer, called its \textit{pathwise attention mechanism}.

\section{The Model: The Filterformer}
\label{s:Model}

\subsection{Phase 1 - Encoding via Pathwise Attention}
\label{s:Model__ss:PathwiseAttention}

We consider an attention-type feature encoding which modifies the \textit{attention mechanism} of \cite{bahdanau2016neural} and the graph attention mechanism of \cite{velivckovic2017graph} to the context of continuous-time path sources/inputs.  Our attention mechanism, called \textit{pathwise attention}, encodes key information about a novel path $y_{\cdot}$ and, like those attention mechanisms, allows us to adaptively encode any newly observed path in terms of its similarity scores to ``known/previously generated paths''.  

\begin{figure}[hpt!]
    \centering
    \begin{subfigure}[t]{.45\textwidth}
        \centering
        \includegraphics[width=.8\linewidth]{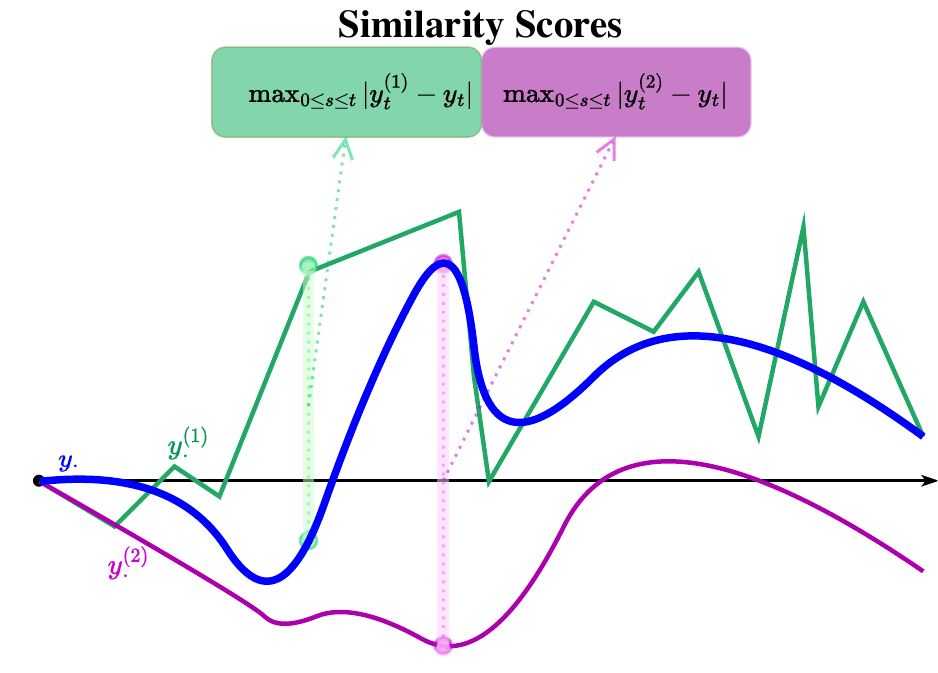}
        \caption{\label{fig:Similarity_Score} The \textit{similarity score} component $\operatorname{sim}_T^{\theta_0}$ encodes a path into finite-dimensional vectorial data, by relating its similarity to a dictionary of previously observed ``contextual paths''.}  
    \end{subfigure}
    \centering
    \begin{subfigure}[t]{.45\textwidth}
        \centering
        \includegraphics[width=.8\linewidth]{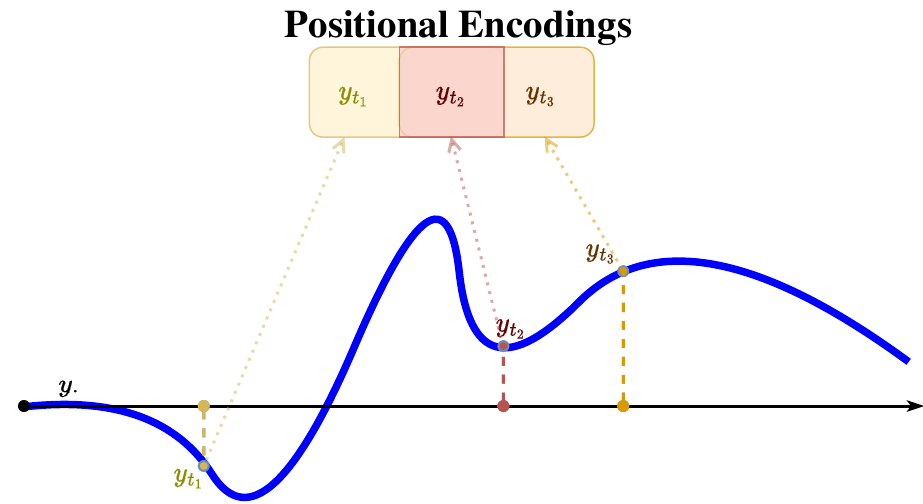}
        \caption{\label{fig:PositionalEncoding} The \textit{positional encoding} component $\operatorname{post}_T^{\theta_0}$ of the attention mechanism allows the similarity scores to further adapt to any novel path by extracting sample-positions on the novel paths.}
    \end{subfigure}
    \caption[1]{Our \textit{pathwise attention mechanism} $\operatorname{attn}_t^{\theta}$, in \Cref{defn:PathAttention}, relies on two methods for adaptive feature extraction.  The first component of the mechanism, in~\ref{fig:Similarity_Score}, in is a similarity score which ranks the similarity of any novel path $x_{\cdot}$ to a dictionary of ``recognized/saved'' reference paths.  The second component of the mechanism, in~\ref{fig:PositionalEncoding}, samples points on the novel path $x_{\cdot}$ and perturbs the similarity scores by them.  The adeptness of the pathwise attention mechanism, unlike basis-based methods, is rooted in the fact that nearly all of its parameters are trainable, which allows it to be tasks-specific (similarly to the advantage of deep MLPs have over regression via basis-functions).}
    \label{fig:Attention_Is_All_You_Need}
\end{figure}

\begin{definition}[Similarity Score]
\label{defn:Sim_Score}
Fix positive integers $\rdim$, $\sdim$, $d_Y$, and a ``time horizon'' $T> 0$.  
Given a set of $\rdim$ distinct paths $\{y^{(n)}\}_{n=1}^\rdim$ in $C([0,1],\mathbb{R}^{d_Y})$, an $\sdim \times \rdim$-matrix  $B$, $b\in \mathbb{R}^{\sdim}$, an $\sdim \times \sdim$-matrix $A$, and $a\in \mathbb{R}^{\sdim}$.  Denote $\theta_0\eqdef (A,B,a,b,\{y^{(n)}\}_{n=1}^\rdim)$.

The \textit{similarity score}, is a map $\operatorname{sim}_T^{\theta_0}:C([0,1],\mathbb{R}^{d_Y})\rightarrow\mathbb{R}^\sdim$, mapping any $y_{\cdot}\in C([0,1],\mathbb{R}^{d_Y})$ to
\begin{equation}
\label{eq:sim_score}
    \operatorname{sim}_T^{\theta_0}:
        y_{\cdot} 
    \mapsto 
        \operatorname{Softmax}
            \Big(
                    A\operatorname{ReLU}\bullet 
                \Big(
                        B
                        \big(\|y_{\cdot}-y^{(n)}_{\cdot}\|_t\big)_{n=1}^\rdim
                    +
                        b
                \Big)
                    +
                        a
            \Big)
    .
\end{equation}
\end{definition}

Illustrated in \Cref{fig:Attention_Is_All_You_Need}, our pathwise attention mechanism extracts features by reinterpreting two of the key features of attention mechanisms from natural language processing (NLP).  First, it utilizes \textit{similarity scores} which rank the likeness of any novel path against a dictionary of references paths.  These reference paths can either arise from real/stress historical (market) scenarios, they can be generated synthetically, or a combination of either.  These similarity scores are illustrated in \Cref{fig:Similarity_Score}, and they serve as pathwise analogues of contextual keys and queries in attention mechanisms in NLP.

\begin{definition}[Positional Encoding]
\label{defn:PositionalEncoding}
In the notation of \Cref{defn:Sim_Score}, fix a positive integer $\pdim$, and set of query times $0\le t_1<\dots< t_\tdim\le T$, an $\pdim\times {d_Y}$ matrix $V$, and an $\pdim\times \tdim$ matrix $U$.  Set $\theta_1\eqdef (V,U,\{t_n\}_{n=1}^\tdim)$.
\noindent The \textit{positional encoding}, is a map $\operatorname{post}_T^{\theta_0}:C([0,1],\mathbb{R}^{d_Y})\rightarrow\mathbb{R}^{\pdim \times d_Y}$, defined by 
\begin{equation}
\label{eq:pos_encoding}
    \operatorname{post}_T^{\theta_1}
        (y_{\cdot} )
    \eqdef 
        U \,
        (\oplus_{j=1}^\tdim \, y_{t_j})
        + V
    ,
\end{equation}
for any $y_{\cdot}\in C([0,1],\mathbb{R}^{d_Y})$.
\end{definition}

The \textit{positional encoding} of any path, illustrated in \Cref{fig:PositionalEncoding}, encodes snapshots of it at various times.  These indicate changes in any sample path, and they fully determine piecewise linear paths with a finite number of pieces (as are generated in simulation studies).  This then combines with the weights of the similarity scores to produce the feature encodings of our pathwise attention mechanism.

\begin{definition}[Pathwise Attention]
\label{defn:PathAttention}
In the notation of \Cref{defn:Sim_Score,defn:PositionalEncoding}, set $\sdim=\pdim$.  A \textit{pathwise attention}, with parameter $\theta\eqdef (\theta_0,\theta_1,C)$, is the map $\operatorname{attn}_T^{\theta}:[0,T]\times C([0,T],\mathbb{R}^{d_Y})\rightarrow \mathbb{R}^{\edim + 1}$ with representation
\begin{align}
    \label{eq:featuremap}
        \operatorname{attn}_T^{\theta}(t,y_{\cdot})
    \eqdef 
        \Big(
                t
            ,
                C\,
                \operatorname{vec}\big(
                        \operatorname{sim}_T^{\theta_0}(y_{\cdot}) 
                    \odot
                        \operatorname{post}_T^{\theta_1}(y_{\cdot})
                \big)
        \Big)
\end{align}
for any $y_{\cdot}\in C([0,1],\mathbb{R}^{d_Y})$ and $0\le t\le T$, where $C$ is an $\edim \times (\sdim d_Y)$-dimensional matrix.
\end{definition}

\begin{proposition}[Pathwise Attention Losslessly Encodes Regular Domains]
\label{prop:Lossless_Encoding}
\sloppy Let $K\subseteq C([0:T],\mathbb{R}^{d_Y})$ satisfy \Cref{ass:reg_compact}.  Then, there exists an $\edim \in \mathbb{N}_+$ and a parameter $\theta$ as in \Cref{defn:PathAttention} such that, $\operatorname{attn}_T^{\theta}$ restricts to a bi-Lipschitz embedding of $(K,\|\cdot\|_T)$ into $(\mathbb{R}^{\edim+1},\|\cdot\|_2)$.

Estimates for the encoding dimension $\edim$ are recorded in \Cref{tab:ATTENTION_Rates} on a case-by-base basis.
\end{proposition}

\begin{table}[!ht]
    \centering
    \caption{{Complexity Estimates for Lossless Pathwise Attention Encoding of Regular Compact Domains.}}
    \label{tab:ATTENTION_Rates}
    \begin{tabular}{lllll}
    \toprule
        Type of Compact Domain 
        & Encoding Dimension ($\edim$) & Ass.~\ref{ass:reg_compact} \\
    \midrule
        Finite
        & $\mathcal{O}(\log(\# K))$ 
        & \ref{ass:reg_compact_finite}
        \\
        P.W. Lin. ($\pldim$ Pieces) 
        & $\mathcal{O}(\pldim d)$
        & \ref{ass:reg_compact_linear}
        \\
        Iso. Comp. Riemann.
        & Finite
        & \ref{ass:reg_compact_riemann}
        \\
        Iso. Comp. Riemann $\&$ Aspherical 
        & $\smash{\mathcal{O}\big(\mathfrak{N}_{d_g}(\mathcal{M},\operatorname{Vol}(\mathcal{M},g)^{1/\ldim} )\big)}$
        & \ref{ass:reg_compact_riemann}
        \\
    \bottomrule
    \end{tabular}
    \caption*{Fix an $A\subset \mathcal{M}$ and $\delta>0$.  The quantity $\mathfrak{N}_{d_g}(A,\delta)$ denotes the minimum number, $I$, of points $\{x_i\}_{i=1}^I \subseteq A$ for which every $x\in A$ is contained in a geodesic ball of radius $\delta$ about some $x_i$, for $i\in \{1,\dots,I\}$.}
\end{table}

\subsection{Phase 2 - Multi-Layered Perceptrons (MLPs) Transformation}
\label{s:Model__ss:Review__MLPs}
Let us briefly recall the structure of a standard MLP, before formalizing our FF. 
In what follows, we fix a real-valued \textit{activation function} $\sigma$ defined on $\mathbb{R}$ satisfying:
\begin{assumption}[{\cite{pmlr-v125-kidger20a}~Condition}]
\label{ass:KLCondition}
The activation function $\sigma$ is continuous, non-affine, and there exists an $t\in \mathbb{R}$ such that $\sigma$ is differentiable at $t$ and $\sigma^{\prime}(t)\neq 0$.
\end{assumption}
Examples of activation functions satisfying this condition are the $\operatorname{PReLU}$ function $t\mapsto \max\{0,t\}+a\min\{0,t\}$, for a hyperparameter $a\in \mathbb{R}$, the $\tanh$ function used in the numerical PDE literature \cite{de2021approximation}, the sine function used in SIRENs \cite{sitzmann2020implicit}, and the Swish map $t\mapsto \frac{t}{1+e^{-\beta t}}$ of \cite{ramachandran2017searching}, where $0\le \beta$ is a hyperparameter.  

Fix an encoding dimension $\edim$ and a target dimension $\ddim$; both of which are positive integers.  A \textit{multi-layer perceptron} (MLP), also called feedforward neural network, from $\mathbb{R}^\edim$ to $\mathbb{R}^\ddim$ with activation function $\sigma$ is a map $\hat{f}:\mathbb{R}^\edim\rightarrow\mathbb{R}^\ddim$ with iterative representation: for each $x\in \mathbb{R}^\edim$
\begin{equation}
\label{eq:MLP}
    \begin{aligned}
    \hat{f}(x) & \eqdef A^{(J)}\,x^{(J)}+b^{(J)},
    \\
    x^{(j+1)} & \eqdef  
        \sigma
        \bullet 
        \big( 
            A^{(j)}
            \,
            x^{(j)}
            +
        b^{(j)}\big)
        \qquad 
            \mbox{for }  
        j=0,\dots,J-1,
    \\
    x^{(0)} & \eqdef  x.
    \end{aligned}
\end{equation}
where for $j=0,\dots,J-1$, each $A^{(j)}$ is a $d_{j+1}\times d_j$ matrix, $b^{(j)}\in \mathbb{R}^{d_{j+1}}$, $d_0=\edim$ and $d_{J+1}=\ddim$.

\subsection{Phase 3 - Decoding via Geometric Attention Mechanism}
\label{s:Model__ss:GeometricAttentionMechanism}

Fix $\ddim \in \mathbb{N}_+$.  In what follows, we denote the orthogonal projection of $\mathbb{R}^\ddim$ onto the $\ddim$-simplex 
$
    \Delta_\ddim
    \eqdef 
    \{
        w \in [0,1]^\ddim:\,
        \sum_{n=1}^\ddim\,w_n = 1
    \}
$
by $P_{\Delta_\ddim}$. The closedness and convexity of $\Delta_\ddim$ implies that $P_{\Delta_\ddim}$ is a well-defined $1$-Lipschitz map, see e.g.~\citep[Proposition 4.8]{CombettesBauschke_Book_ConvexAnalysisMonotonOperators_2017}.  

\begin{definition}[Geometric Attention Mechanism]
\label{defn:GeometricAttentionMechanism}
Let $\ddim,d_X \in \mathbb{N}_+$.  A geometric attention mechanism is a map $\operatorname{g-attn}_\ddim: \mathbb{R}^\ddim \rightarrow \mathcal{N}_{d_X}$ with representation
\[
        \operatorname{g-attn}_\ddim^{\vartheta}(v)
    =
        \mathcal{N}_{d_X}\Biggl(
                \sum_{n=1}^\ddim\,
                    P_{\Delta_\ddim}(v)_n
                    \cdot
                    m^{(n)}
            ,
                \sum_{n=1}^\ddim\,
                    P_{\Delta_\ddim}(v)_n
                    \cdot
                    (A^{(n)})^{\top}A^{(n)}
        \Biggr)
\]
where $m^{(1)},\dots,m^{(\ddim)}\in \mathbb{R}^{d_X}$, $A^{(1)},\dots,A^{(\ddim)} \in \R^{d_X\times d_X}$; and $\vartheta\eqdef (m^{(n)},A^{(n)})_{n=1}^\ddim$.
\end{definition}

\subsection{The FF Model}
\label{s:Model__ss:FFModel}
We may now formalize the Filterformer model of \Cref{fig:FFModel}.  
\begin{definition}[Filterformer (FF)]
Let $d_X, d_Y \in\mathbb{N}_+$, $T>0$, and an activation function $\sigma:\mathbb{R}\rightarrow \mathbb{R}$.  
A function $\hat{F}:C([0:T],\mathbb{R}^{d_Y})\rightarrow \mathcal{N}_{d_X}$ is called a Filterformer if it admits the representation
\begin{equation}
\label{eq:FF_formalization}
        \hat{F}
    =
            \operatorname{g-attn}_\ddim^{\vartheta}
        \circ 
            \hat{f}
        \circ 
            \operatorname{attn}_T^{\theta}
\end{equation}
where $\operatorname{g-attn}_\ddim^{\vartheta}$ is as in \Cref{defn:GeometricAttentionMechanism}, $\operatorname{attn}_T^{\theta}$ is as in \Cref{defn:PathAttention}, and $\hat{f}$ is an MLP as in~\eqref{eq:MLP}; such that the composition~\eqref{eq:FF_formalization} is well-defined.
\end{definition}

\section{Main Result}
\label{s:Main_Results}

We are now ready to state our main theorem, which shows that the FF~\eqref{eq:FF_formalization} is indeed capable of asymptotically optimally filtering the coupled system~\eqref{eq:filtering_dynamics__signal}-\eqref{eq:filtering_dynamics__observation}.  Our guarantees are of a non-asymptotic form; in that they depend on the complexity of the network.
\begin{theorem}[The FF Can Approximate the Optimal Filter]
\label{thrm:MAIN}
\sloppy Let $d_X, d_Y \in \mathbb{N}_+$ and $K\subset C([0:T],\mathbb{R}^{d_Y})$ satisfy \Cref{ass:reg_compact}.  Suppose that the coupled system~\eqref{eq:filtering_dynamics__signal}-\eqref{eq:filtering_dynamics__observation} satisfies \Cref{ass:technical_evolution_conditions}.  
For every $T>0$ there exists a FF $\hat{F}:[0:T]\times C([0:T]\times\mathbb{R}^{d_Y})\rightarrow\mathcal{N}_{d_X}$ satisfying the uniform estimate
\[
    \max_{0\le t\le T,\,y_{\cdot}\in K}\,
        \mathcal{W}_p\big(
                \mathbb{P}(X_t \in \cdot \vert y_{[0:t]})
            ,
                \hat{F}(t,y_{\cdot})
        \big)
    < 
        \varepsilon
    ,
\]
for every $1\le p\le 2$.  Furthermore, \Cref{tab:ComplexityMain} records the complexity estimates of $\hat{F}$.
\end{theorem}

\begin{table}[ht!]
    \centering
    \caption{{Complexity Estimates for the FF model $\hat{F}$ in~\Cref{thrm:MAIN}.}}
    \label{tab:ComplexityMain}
    \begin{tabular}{lllll}
    \toprule
        $\sigma$ Regularity & Depth & Width & Encode ($\edim$) & Decode ($\ddim$)\\
    \midrule
        $\operatorname{ReLU}$ & 
        $
        \smash{
            \mathcal{O}\big(
                \varepsilon^{-\ddim}
            \big)
        }
    	$
        & 
        $        
        \smash{
            \mathcal{O}\big(
                \varepsilon^{-\ddim}
            \big)
        }
        $
        & $\mathcal{O}(1)$ & $\mathcal{O}\big(\varepsilon^{-1}\big)$
        \\
        Smooth \& Non-poly.
         & 
         $
         \smash{
         \mathcal{O}\big(
            \varepsilon^{-4\ddim-1}
         \big)
    	 }
         $
         &
         $ \ddim + \edim +3$
         & $\mathcal{O}(1)$ & $\mathcal{O}\big(\varepsilon^{-1}\big)$
         \\
         Poly. $\&$ Non-affine
         & 
         $
         \smash{
         \mathcal{O}\big(
            \varepsilon^{-8\ddim-6}
         \big)
    	 }
         $
         & 
         $\ddim+\edim+4$
         & $\mathcal{O}(1)$ & $\mathcal{O}\big(\varepsilon^{-1}\big)$
         \\
         Non-Smooth $\&$ Non-poly.
            &
        Finite
        &
        $\ddim+\edim+3$
         & $\mathcal{O}(1)$ & $\smash{\mathcal{O}\big(\varepsilon^{-1}\big)}$
         \\
    \bottomrule
    \end{tabular}
\end{table}

\subsection{Proof Overview}
\label{s:Theoretical_Analysis}
This section overviews the derivation of \Cref{thrm:MAIN}.

Results in~\cite{LipShiryaev_II_2001} to show that  $\mathbb{P}(X_t\in \cdot|\mathcal{Y}_t)$ is conditionally Gaussian, they denote an explicit formulation of the resulting distribution which can then analyze directly. 
In particular, \Cref{ass:technical_evolution_conditions} allows the application of \citep[Theorem 12.6]{LipShiryaev_II_2001} and we obtain that for each $0\le t\le T$ and each path $y_{[0:t]}\in C([0:t], \R^{d_Y})$ the probability measure $f_t(y_{[0:t]})$ is a $d_X$-dimensional Gaussian measure.  We denote the mean and covariance of this measure by $\mu(y_{[0:t]})$ and $\Sigma(y_{[0:t]})$.  Thus, 
\begin{align}
        \mathcal{N}(\mu(y_{[0:t]}),\Sigma(y_{[0:t]})) 
    \sim 
        \mathbb{P}(X_t\in \cdot|y_{[0:t]})
    \eqdef 
        f_t(y_{[0:t]})
    , \label{id:LipsterShiryaevRepresentation}
\end{align}
for each $0\le t\le T$ and each path $y_{[0:t]}\in C([0:t], \R^{d_Y})$.

A function is uniformly approximable by a deep learning model if it is continuous.  If there is no clear favourable structure in the function being approximated, e.g.\ smoothness \cite{yarotsky2020phase,gonon2023approximation} or neural network-like structure \cite{cheridito2021efficient,mhaskar2016deep}, then the best available approximation rates are those corresponding to those where the target function is Lipschitz; see \cite{DaubechiesDeVoreFoucartHaninPetrova_CA_KnotFreeReLU_2022}. 
Thus, the first step to obtaining approximability guarantees for the optimal filter by deep neural networks, which depend on relatively few trainable parameters, is to show that the optimal filter is a locally-Lipschitz function of observed paths.

\begin{proposition}[Local Lipschitz-Continuity of the Optimal Filter]
\label{prop:LocLipReg_OptimFilter}
    Under \Cref{ass:technical_evolution_conditions}, $f_t$ from \eqref{goal:Robust_Filtering} is locally Lipschitz-continuous. In particular, for every time $t \in [0:T]$, path $y^{(1)}_\cdot \in C^{\color{blue} 1}([0:T], \R^{d_Y})$ and $\epsilon > 0$ there exists constant $C \ge 0$ such that for all times $s \in [0:T]$ and paths $y^{(2)}_\cdot \in C^{\color{blue} 1}([0:T], \R^{d_Y})$ with $\vert t -s \vert < \epsilon$, $\| y^{(1)}_{[0:T]} - y^{(2)}_{[0:T]} \|_T < \epsilon$ holds
    \begin{align*}
        \mathcal{W}_2(\mathbb{P}(X_t \in \cdot \vert y^{(1)}_{[0:t]}), \mathbb{P}(X_s \in \cdot \vert y^{(2)}_{[0:s]}))
        \leq C 
            \big(
                \Vert y^{(1)}_{[0:T]} - y^{(2)}_{[0:T]} \Vert_T 
            + 
                \vert t - s \vert
            \big).
    \end{align*}
\end{proposition}

Having established the regularity of the optimal filter for the coupled system~\eqref{eq:filtering_dynamics__signal}-\eqref{eq:filtering_dynamics__observation}, we need to verify that maps sharing the same domain (input space), codomain (output space), and regularity as the optimal filter are approximable by the proposed FF model.  This is the content of the next proposition, which acts as its standalone approximation theorem for the FF.

\begin{table}[!ht]
    \centering
    \caption{{Complexity Estimates for transformer-type model $\hat{F}$ in \Cref{prop:Universal_Approximation_Theorem__PathToNd}.}}
    \label{tab:approximationrates__nonFlatVersion}
    \begin{tabular}{@{}lllll@{}}
    \toprule
        $\sigma$ Regularity & Depth & Width & Encode ($\edim$) & Decode ($\ddim$) \\
    \midrule
        $\operatorname{ReLU}$ & 
        $
        \smash{
            \mathcal{O}\big(
                (LV( L))^{-\ddim}\,\varepsilon^{-\ddim}
            \big)
        }
        $
        & 
        $
        \smash{
            \mathcal{O}\big(
                (LV(L))^{-\ddim}\,\varepsilon^{-\ddim}
            \big)
        }
        $
        & $\mathcal{O}(1)$ & $\mathcal{O}\big(L\varepsilon^{-1}\big)$
        \\
        Smooth $\&$ Non-poly.
         & 
         $
         \smash{
        \mathcal{O}\big(
            L^{4\ddim+1}\, \varepsilon^{-4\ddim-1}
        \big)
    	}
         $
         &
         $ \ddim + \edim +2$
         & $\mathcal{O}(1)$ & $\mathcal{O}\big(L\varepsilon^{-1}\big)$
         \\
         Poly. $\&$ Non-affine
         & 
         $
         \smash{
        \mathcal{O}\big(
            L^{8\ddim+6}\, \varepsilon^{-8\ddim-6}
        \big)
    	}
         $
         & 
         $\ddim+\edim+3$
         & $\mathcal{O}(1)$ & $\mathcal{O}\big(L\varepsilon^{-1}\big)$
         \\
         Non-Sm. $\&$ Non-poly.
            &
        Finite
        &
        $\ddim+\edim+2$
         & $\mathcal{O}(1)$ & $\smash{\mathcal{O}\big(L\varepsilon^{-1}\big)}$
         \\
    \bottomrule
    \end{tabular}
    \caption*{Where $V(t)$ is the inverse of $s\mapsto s^4\,\log_3(t+2)$ on $[0,\infty)$ evaluated at $131 t$.}
\end{table}

\begin{proposition}[Approximation Capacity of FFs]
\label{prop:Universal_Approximation_Theorem__PathToNd}
Let $d_X, d_Y \in \mathbb{N}_+$, $L>0$, and $K\subset C([0:T],\mathbb{R}^{d_Y})$ satisfying \Cref{ass:reg_compact}.  
For every $0<\varepsilon< 1/2$ and every $L$-Lipschitz function $f:([0,T]\times K,|\cdot| \times \|\cdot\|_{\infty})\rightarrow (\mathcal{N}_{d_X},\mathcal{W}_2)$ there exists a FF $\smash{\hat{F}}$ 
satisfying the uniform estimate 
\[
    \sup_{t, x\in [0:T]\times K}\,
        \mathcal{W}_p\big(
        f(t, x_\cdot) ,
        \hat{F}(t,x_\cdot) 
        \big)
    < 
        \varepsilon
    ,
\]
for every $1\le p\le 2$.  Furthermore, the depth, width, encoding dimension $(\edim)$ and Decoding dimension $(\ddim)$ are all recorded in \Cref{tab:approximationrates__nonFlatVersion} depending on the activation function $\sigma$.
\end{proposition}
\begin{proof}\textbf{of \Cref{thrm:MAIN}}\hfill\\
Since the coupled system~\eqref{eq:filtering_dynamics__signal}-\eqref{eq:filtering_dynamics__observation} satisfies \Cref{ass:technical_evolution_conditions}, \Cref{prop:LocLipReg_OptimFilter} implies that the optimal filter $f_T$ in \eqref{goal:Robust_Filtering} 
is a locally-Lipschitz map from $([0:T],  C([0:T],\mathbb{R}^{d_Y}),\vert \cdot \vert \times \|\cdot\|_T)$ to $(\mathcal{N}_{d_X},\mathcal{W}_2)$.  Since $K\subseteq C([0:T],\mathbb{R}^{d_Y})$ is compact then, $f|_K$ is Lipschitz.  
Since $K\subset C([0:T],\mathbb{R}^{d_Y})$ satisfies \Cref{ass:reg_compact} and $0<\varepsilon<1/2$ then, \Cref{prop:Universal_Approximation_Theorem__PathToNd} implies that there is a FF $\hat{F}:C([0:T],\mathbb{R}^{d_Y})\rightarrow\mathcal{N}_{d_X}$ with representation~\eqref{eq:FF_formalization} and complexity recorded in \Cref{tab:approximationrates__nonFlatVersion} satisfying
\[
	\sup_{t, y\in [0:T]\times K}\,
		\mathcal{W}_p\big(
			f(t,y)
		,
			\hat{F}(t,y)
		\big)
	< 
		\varepsilon
	.
\]
This completes the proof of \Cref{thrm:MAIN}.
\end{proof}

\section{Experimental Ablation}
\label{s:Experiments}
We benchmark Filterformer against the natural filtering baselines in the conditionally Gaussian regime: the Kalman filter~\cite{Kalman60}, the ensemble Kalman filter~\cite{houtekamer2005ensemble}, and the particle filter~\cite{del1997nonlinear}.  All parameters are estimated either using the classical EM algorithm (Kalman-type filters) or maximum likelihood (particle filters) from the training data.
Code is publically available at: \url{https://github.com/AnastasisKratsios/Filterformer_Demo}.

\paragraph{Experimental setup.}
We work in the scalar case $d_X=d_Y=1$ on $[0,T]$ with $T=1$. Fix constants
$
a_0,A_0,B_0,\alpha,\beta\in\mathbb{R},
\qquad
B_0\neq 0,
$
and, for a context length $0\le \varkappa<T$, define
$
I_t^\varkappa(y_{[0:t]})
 \eqdef 
\int_{(t-\varkappa)\vee 0}^t e^{t-s}y_s\,ds.
$
We then consider the path-dependent system
\begin{align}
dX_t
&=
\big(
a_0+\alpha I_t^\varkappa(Y_{[0:t]})X_t
\big)\,dt
+
\beta I_t^\varkappa(Y_{[0:t]})\,dW_t,
\\
dY_t
&=
A_0\,dt+B_0\,dW_t.
\end{align}
In this setting, by~\citep[Theorem 12.7]{LipShiryaev_II_2001}, the (optimal) conditional covariance is available in closed form and the (optimal) conditional mean in semi-closed form, providing exact targets for an interpretable evaluation; namely, at any time $0\le t\le T$, the optimal $\mu_t\eqdef \mu_t(y_{[0:t]})$ and $\Sigma_t\eqdef \Sigma_t(y_{[0:t]})$ map any continuous path $y_{[0:t]}$ to
\begin{equation}
\label{eq:optimal_dynamics}
\begin{aligned}
d\mu_t(y_{[0:t]})
&=
\Big[
a_0+\alpha\,I_t^\varkappa(y_{[0:t]})\,\mu_t(y_{[0:t]})
\Big]dt 
+
\frac{
\beta B_0\,I_t^\varkappa(y_{[0:t]})+A_1\Sigma_t(y_{[0:t]})
}{
B_0^2
}
\\
&\,
\times
\Big[
dy_t-\big(A_0+A_1\mu_t(y_{[0:t]})\big)dt
\Big]
\\
\Sigma_t(y_{[0:t]})
&=
\Sigma_0\exp\Big(2\alpha\int_0^t I_r^\varkappa(y_{[0:r]})\,dr\Big)
\end{aligned}
\end{equation}
where
$
I_t^\varkappa(y_{[0:t]})
 \eqdef 
\int_{(t-\varkappa)\vee 0}^t e^{t-s}\,y_s\,ds
$,
and where the above formula for $\Sigma_t$ holds in the special case $A_1=0$.

\paragraph{Details.}
Default training setup consist of $N=10^3$ i.i.d.\ training samples simulated observation paths, while the test set consists of $10^2$ independently generated paths, all sampled on the uniform grid $t_m=m/M$ with $M=5\cdot 10^2$; all dynamics are sampled from the true dynamics~\eqref{eq:filtering_dynamics__signal}-\eqref{eq:filtering_dynamics__observation} according to the standard Euler-Maruyama scheme with $1/M$-time discretization step-size.
Throughout, the data-generating model uses the parameter values
$
C=0.2,
\,
\eta=0.75,
\,
a_0=0.1,
\,
A_0=0,
\,
B_0=0.5,
\,
\alpha=0.4,
\,
\beta=0.2,
\,
\mu_0=0,
\,
\Sigma_0=1,
\,
x_0^{\mathrm{std}}=1,
\,
y_0=0.
$
At each time $t_m$, the model receives the truncated observation window
$
(t_m,Y_{[(t_m-C)\vee 0:t_m]})
$
and predicts the corresponding filter statistics
$
(\mu_{t_m},\Sigma_{t_m}).
$
Test data are generated independently in the same manner, with the effective evaluation horizon controlled by the parameter $\eta=0.75$. We report test performance only, since our interest is in the model's ability to predict beyond its training data, i.e.\ to generalize.

\begin{figure}[H]
    \centering
    \includegraphics[width=1.0\linewidth]{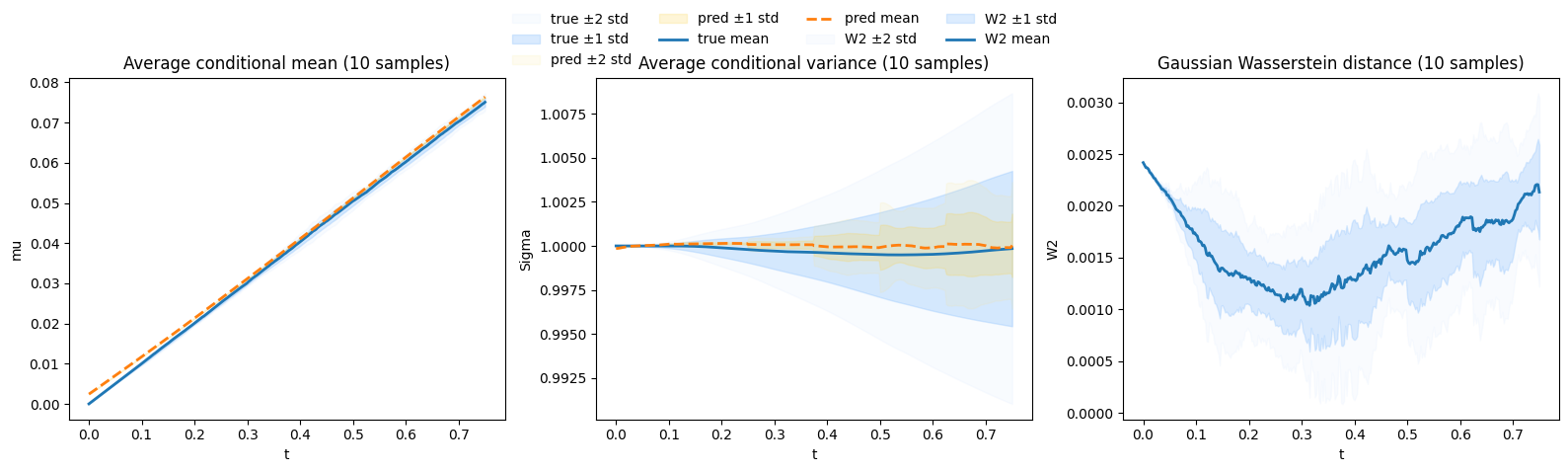}
    \caption{Typical mean and covariance tracking under the default experimental configuration, with two standard deviations.}
    \label{fig:typical_fig}
\end{figure}

\begin{figure}[H]
    \centering
    \includegraphics[width=1.0\linewidth]{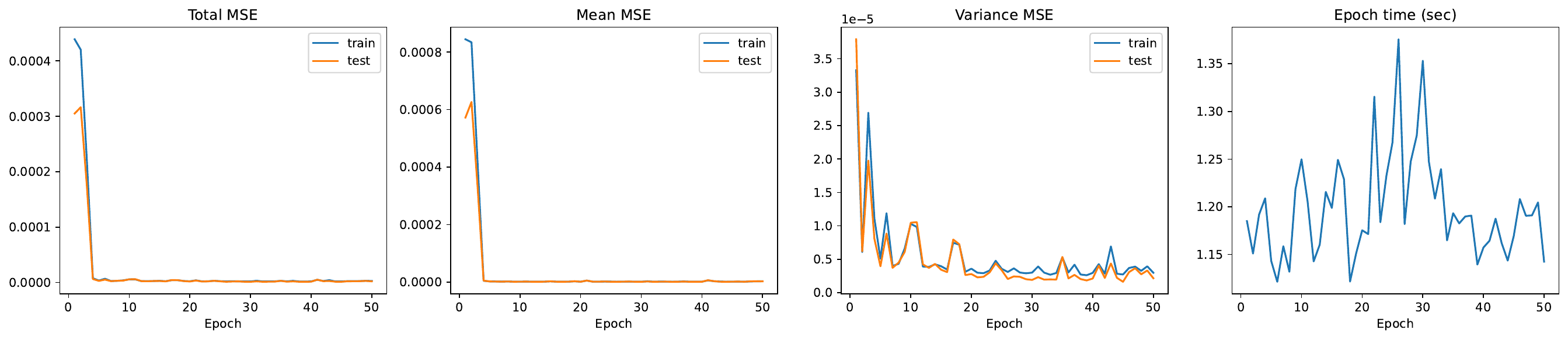}
    \caption{Typical reduction in the training (in-sample) mean, covariance, and $W_2$ errors of the FilterFormer model on Gaussian training data, over the training set. We also report the evolution of the FilterFormer's predictive performance as a function of the training epochs.
    \hfill\\
    We emphasize that the test data is never seen during training; it is used only for evaluation and illustrative purposes.}
    \label{fig:typical_training}
\end{figure}

\subsection{Signal-to-Noise Ratio (SnR) Ablation}
In this experiment, we vary $B_0$ while keeping the remaining model and training parameters fixed, thereby moving from a high signal-to-noise regime (small $|B_0|$) to a low signal-to-noise regime (large $|B_0|$).  This ablation scrutinizes the robustness of Filterformer to increasingly noisy observations, and quantifies how accurately it can recover the conditional mean and covariance as the observation channel becomes less informative.

\begin{table}[ht!]
\centering
\renewcommand{\arraystretch}{1.15}
\setlength{\tabcolsep}{8pt}
\begin{adjustbox}{max width=\textwidth}
\begin{tabular}{>{\raggedright\arraybackslash}p{2.5cm}cccc}
\toprule
Model & $\mathrm{MSE}_{\mu}$ ($\pm 1$ std) & $\mathrm{MSE}_{\Sigma}$ ($\pm 1$ std) & $W_2$ ($\pm 1$ std) & Training Time (s) \\
\midrule
\rowcolor[HTML]{E5EFE5} FilterFormer & 0.000045 $\pm$ 0.000075 & 0.000031 $\pm$ 0.000027 & 0.005060 $\pm$ 0.002705 & 117.94 \\
\rowcolor[HTML]{EAECE5} Particle Filter  & 0.000443 $\pm$ 0.000611 & 0.001011 $\pm$ 0.001363 & 0.023149 $\pm$ 0.012659 & 0.35 \\
\rowcolor[HTML]{EEE8E5} Ensemble KF  & 0.003905 $\pm$ 0.000281 & 0.007869 $\pm$ 0.000609 & 0.067522 $\pm$ 0.001821 & 0.12 \\
\rowcolor[HTML]{F3E5E5} Kalman Filter & 0.343675 $\pm$ 0.373735 & 1.002088 $\pm$ 0.035195 & 1.126412 $\pm$ 0.134505 & 0.13 \\
\bottomrule
\end{tabular}
\end{adjustbox}
\caption{\textbf{Low SnR $B_0=1$:} Test-set summary across all models, ordered by increasing $W_2$.}
\label{tab:sorted_model_summary_w2__largenoise}
\end{table}

For small observation noise ($B_0=0.1$), all methods benefit from the more informative observation channel, but FilterFormer still achieves the best overall reconstruction of the conditional law, with the smallest $W_2$, mean, and covariance errors.

\begin{table}[ht!]
\centering
\renewcommand{\arraystretch}{1.15}
\setlength{\tabcolsep}{8pt}
\begin{adjustbox}{max width=\textwidth}
\begin{tabular}{>{\raggedright\arraybackslash}p{2.5cm}cccc}
\toprule
Model & $\mathrm{MSE}_{\mu}$ ($\pm 1$ std) & $\mathrm{MSE}_{\Sigma}$ ($\pm 1$ std) & $W_2$ ($\pm 1$ std) & Training Time (s) \\
\midrule
\rowcolor[HTML]{E5EFE5} FilterFormer & 0.000001 $\pm$ 0.000002 & 0.000039 $\pm$ 0.000022 & 0.003105 $\pm$ 0.000735 & 114.69 \\
\rowcolor[HTML]{EAECE5} Particle Filter  & 0.000443 $\pm$ 0.000612 & 0.001006 $\pm$ 0.001352 & 0.023140 $\pm$ 0.012656 & 0.32 \\
\rowcolor[HTML]{EEE8E5} Ensemble KF  & 0.003901 $\pm$ 0.000272 & 0.007850 $\pm$ 0.000527 & 0.067494 $\pm$ 0.001695 & 0.12 \\
\rowcolor[HTML]{F3E5E5} Kalman Filter & 0.005134 $\pm$ 0.006407 & 1.000155 $\pm$ 0.003532 & 1.000568 $\pm$ 0.002602 & 0.11 \\
\bottomrule
\end{tabular}
\end{adjustbox}
\caption{\textbf{High SnR $B_0=0.1$:} Test-set summary across all models, ordered by increasing $W_2$.}
\label{tab:sorted_model_summary__smallnoise}
\end{table}

For large observation noise ($B_0=1$), the filtering task becomes harder across the board, yet FilterFormer remains the most accurate method, indicating strong robustness to degradation in the signal-to-noise ratio.

\subsection{Non-Markovianness Ablation}
In this experiment, we vary the context length $\varkappa$ while keeping the remaining model and training parameters fixed, thereby moving from a regime which is close to Markovian when $\varkappa$ is small to a strongly path-dependent regime when $\varkappa$ is large.  This ablation examine whether Filterformer is genuinely exploiting historical information in the observation path, rather than merely learning a short-memory surrogate of the filter.

\begin{table}[ht!]
\centering
\renewcommand{\arraystretch}{1.15}
\setlength{\tabcolsep}{8pt}
\begin{adjustbox}{max width=\textwidth}
\begin{tabular}{>{\raggedright\arraybackslash}p{2.5cm}cccc}
\toprule
Model & $\mathrm{MSE}_{\mu}$ ($\pm 1$ std) & $\mathrm{MSE}_{\Sigma}$ ($\pm 1$ std) & $W_2$ ($\pm 1$ std) & Training Time (s) \\
\midrule
\rowcolor[HTML]{E5EFE5} FilterFormer & 0.000001 $\pm$ 0.000001 & 0.000001 $\pm$ 0.000001 & 0.000884 $\pm$ 0.000273 & 106.44 \\
\rowcolor[HTML]{EAECE5} Particle Filter  & 0.000443 $\pm$ 0.000612 & 0.001005 $\pm$ 0.001352 & 0.023139 $\pm$ 0.012656 & 0.37 \\
\rowcolor[HTML]{EEE8E5} Ensemble KF  & 0.003901 $\pm$ 0.000272 & 0.007849 $\pm$ 0.000526 & 0.067492 $\pm$ 0.001692 & 0.16 \\
\rowcolor[HTML]{F3E5E5} Kalman Filter & 0.086661 $\pm$ 0.096147 & 0.999902 $\pm$ 0.001331 & 1.031959 $\pm$ 0.042739 & 0.13 \\
\bottomrule
\end{tabular}
\end{adjustbox}
\caption{\textbf{Short Memory $\varkappa=0.01$ ($\approx$ Markovian):} Test-set summary across all models, ordered by increasing $W_2$.}
\label{tab:sorted_model_summary_w2__memory_short}
\end{table}

For short memory ($\varkappa=0.01$), where the problem is close to Markovian, FilterFormer attains near-exact recovery of the filter statistics, showing that it does not sacrifice performance even in the weakly path-dependent regime.

\begin{table}[ht!]
\centering
\renewcommand{\arraystretch}{1.15}
\setlength{\tabcolsep}{8pt}
\begin{adjustbox}{max width=\textwidth}
\begin{tabular}{>{\raggedright\arraybackslash}p{2.5cm}cccc}
\toprule
Model & $\mathrm{MSE}_{\mu}$ ($\pm 1$ std) & $\mathrm{MSE}_{\Sigma}$ ($\pm 1$ std) & $W_2$ ($\pm 1$ std) & Training Time (s) \\
\midrule
\rowcolor[HTML]{E5EFE5} FilterFormer & 0.000064 $\pm$ 0.000110 & 0.000043 $\pm$ 0.000032 & 0.005717 $\pm$ 0.002885 & 128.56 \\
\rowcolor[HTML]{EAECE5} Particle Filter  & 0.000443 $\pm$ 0.000612 & 0.001010 $\pm$ 0.001360 & 0.023147 $\pm$ 0.012659 & 0.32 \\
\rowcolor[HTML]{EEE8E5} Ensemble KF  & 0.003904 $\pm$ 0.000279 & 0.007865 $\pm$ 0.000583 & 0.067516 $\pm$ 0.001785 & 0.12 \\
\rowcolor[HTML]{F3E5E5} Kalman Filter & 0.086949 $\pm$ 0.099083 & 1.001964 $\pm$ 0.028719 & 1.032481 $\pm$ 0.042009 & 0.13 \\
\bottomrule
\end{tabular}
\end{adjustbox}
\caption{\textbf{Long Memory $\varkappa=1$:} Test-set summary across all models, ordered by increasing $W_2$.}
\label{tab:sorted_model_summary_w2__memory_long}
\end{table}

For long memory ($\varkappa=1$), where the latent dynamics depend on a much longer observation history, FilterFormer's error increases modestly but it still clearly outperforms all baselines, consistent with its ability to exploit genuinely non-Markovian structure.

\section{Discussion and Future Work}
\label{s:Discussion}
Our main result, namely Theorem~\ref{thrm:MAIN}, showed that there are Filterformers which can solve the approximately stochastic filtering problem in continuous time for the coupled system~\eqref{eq:filtering_dynamics__signal}-\eqref{eq:filtering_dynamics__observation}, to arbitrary accuracy.  This complements the experimental work of \cite{KrishnanShalitSontag15,KrishnanShalitSontag17,Balint_FinalPaper_2023,bach2024filtering}, which shows that deep learning models seemingly can approximately filter and that such models can be trained to offer state-of-the-art empirical performance.  

\section{Proofs}
\label{s:Proof}
This section contains the derivations of our paper's results.

\subsection{Stability of the Optimal Filter - Proof of \texorpdfstring{\Cref{prop:LocLipReg_OptimFilter}}{Lipschitz Continuity}}
We now prove the local Lipschitz stability of the optimal filter~\eqref{goal:Robust_Filtering}, for the coupled system~\eqref{eq:filtering_dynamics__signal}-\eqref{eq:filtering_dynamics__observation}, as a function of its observed/input path.  The proof builds on a series of Lemmata, which demonstrate the local Lipschitz continuity of the parameters defining the optimal filter.  

We first observe that, under the dynamics~\eqref{eq:filtering_dynamics__signal}-\eqref{eq:filtering_dynamics__observation}, and \Cref{ass:technical_evolution_conditions} grants us access to \citep[Theorem 12.7]{LipShiryaev_II_2001}, which implies that $\mu_t\eqdef \mu(y_{[0:t]})$ and $\Sigma_t\eqdef \Sigma(y_{[0:t]})$ are
\allowdisplaybreaks
\begin{align}
    d\mu_t 
        &=[a_0(t,y_{[0:t]}) + a_1(t,y_{[0:t]}) \mu_t]dt +[(b\circ B)(t,y_{[0:t]})+\Sigma_t A_1^{\top}(t,y_{[0:t]})]\label{id:ClosedFormMu}
        \\
        & \pushright{
             \times (B\circ B)^{-1}(t,y_{[0:t]}) [dy_t - (A_0(t,y_{[0:t]}) + A_1(t,y_{[0:t]}))\mu_t dt]  \notag
        }
        \\
    \partial_t\Sigma_t
        & = a_1(t,y_{[0:t]}) \Sigma_t + \Sigma_t a_1^{\top}(t,y_{[0:t]}) + (b\circ b)(t,y_{[0:t]})
        \label{id:ClosedFormSigma} \\
        & \pushright{
              - [(b\circ B)(t,y_{[0:t]}) + \Sigma_t A_1^{\top}(t,y_{[0:t]}) ]
              (B\circ B)^{-1} (t,y_{[0:t]})} \notag \\
        & \pushright{\times [(b\circ B)(t,y_{[0:t]}) + \Sigma_t A_1^{\top}(t,y_{[0:t]})]^{\top} }
         \notag
\end{align}
with initial conditions $\mu_0\eqdef\mathbb{E}[X_0|Y_0]$ and $\Sigma_0\eqdef \mathbb{E}[(X_0-\mu_0)(X_0-\mu_0)^{\top}|Y_0]$; where
\begin{align*}
    B \circ B \eqdef B_1B_1^\top + B_2B_2^\top, \quad 
    b \circ B \eqdef b_1B_1^\top + b_2B_2^\top, \quad 
    b \circ b \eqdef b_1b_1^\top + b_2b_2^\top.
\end{align*}
Furthermore, \citep[Theorem 12.7]{LipShiryaev_II_2001} implies that $\Sigma_t$ is positive definite for all $t \in [0:T]$ since we have assume that $\Sigma_0$ in \Cref{ass:technical_evolution_conditions}.
To clarify the argument's direction, we first record the proof of \Cref{prop:LocLipReg_OptimFilter}.  We subsequently derive the lemmata on which it relies.

\begin{proof}\textbf{of \Cref{prop:LocLipReg_OptimFilter}}\hfill\\
    Fix $t \in [0:T]$, a path $y^{(1)}_\cdot \in C^{\color{blue} 1}([0:T], \R^{d_Y})$ and $\epsilon > 0$.
    By \Cref{prop:LocLipReg_OptimFilter_Time} we can fix a $K_1 > 0$ such that for every time $0 < s < T$ holds
    \begin{align*}
        \mathcal{W}_2(\mathbb{P}(X_t \in \cdot \vert y^{(1)}_{[0:t]}), \mathbb{P}(X_s \in \cdot \vert y^{(1)}_{[0:s]})) \leq K_1 \vert t-s \vert_t.
    \end{align*}
    By \Cref{prop:LocLipReg_OptimFilter_Path} we can fix a $K_2 > 0$ such that for every path $y^{(2)}_\cdot \in C^{\color{blue} 1}([0:T], \R^{d_Y})$ with $\| y^{(1)}_{[0:T]} - y^{(2)}_{[0:T]} \|_T < \epsilon$ holds
    \begin{align*}
            \mathcal{W}_2(
                \mathbb{P}(X_u \in \cdot \vert y^{(1)}_{[0:u]})
            ,
                \mathbb{P}(X_u \in \cdot \vert y^{(2)}_{[0:u]})
            )
        \leq
            K_2 \Vert y^{(1)}_{[0:T]} - y^{(2)}_{[0:T]} \Vert_T.
    \end{align*}
    
    We see by the triangular inequality of the $\mathcal{W}_2$-distance (that holds as we handle normal distributions, see \eqref{id:LipsterShiryaevRepresentation}),
    \begin{align*}
        &
            \mathcal{W}_2(\mathbb{P}(X_t \in \cdot \vert y^{(1)}_{[0:t]}), \mathbb{P}(X_s \in \cdot \vert y^{(2)}_{[0:s]})) \\
        & \quad \leq 
                \mathcal{W}_2(
                    \mathbb{P}(X_t \in \cdot \vert y^{(1)}_{[0:t]})
                ,
                    \mathbb{P}(X_s \in \cdot \vert y^{(1)}_{[0:s]})
                )
            + 
                \mathcal{W}_2(
                    \mathbb{P}(X_s \in \cdot \vert y^{(1)}_{[0:s]})
                ,
                    \mathbb{P}(X_s \in \cdot \vert y^{(2)}_{[0:s]})
                )
            \\
        & \quad \leq \max\{ K_1, K_2 \} 
            \big(
                \Vert y^{(1)}_{[0:T]} - y^{(2)}_{[0:T]} \Vert_T 
            + 
                \vert t - s \vert
            \big)
    \end{align*}
    what concludes the proof.
\end{proof}

\subsubsection{Lipschitz-Continuity in the Path-Component of the Optimal Filter}

\begin{lemma}[Local Lipschitz - the Path-Component of the Optimal Filter]
\label{prop:LocLipReg_OptimFilter_Path}
    \sloppy Under \Cref{ass:technical_evolution_conditions}, $f$ from \eqref{goal:Robust_Filtering} is locally Lipschitz-continuous uniformly in time. In particular, for every path $y^{(1)}_\cdot \in C^{\color{blue} 1}([0:T], \R^{d_Y})$ and $\epsilon > 0$ there exists constant $C \ge 0$ such that for all paths $y^{(2)}_\cdot \in C^{\color{blue} 1}([0:T], \R^{d_Y})$ with $\| y^{(1)}_{[0:t]} - y^{(2)}_{[0:t]} \|_t < \epsilon$ holds
    \begin{align*}
        \mathcal{W}_2(\mathbb{P}(X_t \in \cdot \vert y^{(1)}_{[0:t]}), \mathbb{P}(X_t \in \cdot \vert y^{(2)}_{[0:t]})) \leq C \Vert y^{(1)}_{[0:T]} - y^{(2)}_{[0:T]} \Vert_T, \quad \forall t \in [0:T].
    \end{align*}
\end{lemma}
\begin{proof}\textbf{of Lemma~\ref{prop:LocLipReg_OptimFilter_Path}}
    By \eqref{id:LipsterShiryaevRepresentation}, \Cref{lem:LoewnerBound}, and \Cref{lem:LipschitzManifold_Nd} there exists a non-negative constant $\bar{K} < \infty$ satisfying the following for every pair of paths $y^{(1)}_{\cdot},y^{(2)}_{\cdot} \in C([0:T],\|\cdot\|_2)$ and all time $t \in [0:T]$
    \begin{equation*}
        \mathcal{W}_2(\mathbb{P}(X_t \in \cdot \vert y^{(1)}_{[0:t]}), \mathbb{P}(X_t \in \cdot \vert y^{(2)}_{[0:t]}))
        \leq 
            \bar{K}
                \,
            \sqrt{\| \mu(y_{[0:t]}^{(1)}) - \mu(y_{[0:t]}^{(2)}) \|_2^2 + \|\Sigma(y_{[0:t]}^{(1)})-\Sigma(y_{[0:t]}^{(2)})\|_2^2}.
    \end{equation*}
    We see from \Cref{lem:LipschitzMu} and \Cref{lem:LipschitzSigma}, there exists a non-negative constant $K < \infty$, satisfying
    \begin{align*}
        \begin{split}
                \|\Sigma(y_{[0:t]}^{(1)})-\Sigma(y_{[0:t]}^{(2)})\|_2^2
            & \leq
                K^2 \Vert y^{(1)}_{[0:t]} - y^{(2)}_{[0:t]} \Vert_t^2, 
            \\ 
                \|\mu(y_{[0:t]}^{(1)})-\mu(y_{[0:t]}^{(2)})\|_2^2
            & \leq
                K^2 \Vert y^{(1)}_{[0:t]} - y^{(2)}_{[0:t]} \Vert_t^2, 
        \end{split}
        \qquad \forall t \in [0:T]
    \end{align*}
    Therefore, for each $y_{\cdot}^{(1)},y^{(2)}_{\cdot}\in C([0:T],\mathbb{R}^{d_Y})$ we have  
    \begin{align*}
            \mathcal{W}_2(\mathbb{P}(X_t \in \cdot \vert y^{(1)}_{[0:t]}), \mathbb{P}(X_t \in \cdot \vert y^{(2)}_{[0:t]}))
        \leq
            \sqrt{2} K \bar{K}\Vert y^{(1)}_{[0:t]} - y^{(2)}_{[0:t]} \Vert_t,
        \qquad \forall t \in [0:T]
    \end{align*}
    and the statement follows. 
\end{proof}

\begin{lemma}[Local Lipschitz-continuity of $\mu$ in the path]
\label{lem:LipschitzMu}
    Under \Cref{ass:technical_evolution_conditions} \ref{asm:Regulartiy__1} $\mu_\cdot$ from \eqref{id:ClosedFormMu} is locally Lipschitz-continuous uniformly in time. This means that for every path $y^{(1)}_\cdot \in C^{\color{blue} 1}([0:T], \R^{d_Y})$ there exists constant $K \in \R$, $\epsilon > 0$ such that for all paths $y^{(2)}_\cdot \in C^{\color{blue} 1}([0:T], \R^{d_Y})$ with $\| y^{(1)}_{[0:t]} - y^{(2)}_{[0:t]} \|_t < \epsilon$ holds
    \begin{align*}  
            \big\| \mu^{(1)}_t - \mu^{(2)}_t \big\|_2 
        \leq
            K \big\| y^{(1)}_{[0:t]} - y^{(2)}_{[0:t]} \big\|_t,
        \qquad \forall t \in [0:T].
    \end{align*} 
\end{lemma}
\begin{proof}\textbf{of Lemma~\ref{lem:LipschitzMu}}
Define 
\allowdisplaybreaks
\begin{align} 
  \Phi_1(s, y_{[0:s]}) & \eqdef a_1(s, y_{[0:s]}) 
   - [(b\circ B)(s, y_{[0:s]}) + \Sigma_t A_1^{{\top}}(s, y_{[0:s]}) ] \label{id:Phi1} \\ 
 & \hspace{3cm} \times  (B\circ B)^{-1}(s, y_{[0:s]}) [A_0(s, y_{[0:s]}) + A_1(s, y_{[0:s]})], \notag \\
  \Phi_2(s, y_{[0:s]}) & \eqdef a_0(s, y_{[0:s]}) + [(b\circ B)(s, y_{[0:s]}) + \Sigma_t A_1^{{\top}}(s, y_{[0:s]}) ] (B\circ B)^{-1}(s, y_{[0:s]}) \label{id:Phi2}. 
\end{align}
Note that conditions of \Cref{lem:LipschitzPhi1} and \Cref{lem:LipschitzPhi2} are satisfied, therefore, $\Phi_1$ and $\Phi_2$ are globally Lipschitz w.r.t. time and locally Lipschitz w.r.t. the path. 

By \eqref{id:ClosedFormMu} and for $t \in [0,T]$, $\mu_t$ can be written as 
\begin{align} 
 \mu_t = & \mu_0 + \int_0^t a_0(s, y_{[0:s]}) ds  + \int_0^t \Phi_2(s, y_{[0:s]}) d y_{[0:s]} 
 + \int_0^t \Phi_1(s, y_{[0:s]}) \mu_s ds .   \notag 
\end{align}
We further have that  
\begin{align*}
    \mu^{(1)}_t - \mu^{(2)}_t 
& =   
        \mu_0^{(1)} - \mu_0^{(2)} 
    + 
        \int_0^t a_0(s, y_{[0:s]}^{(1)}) -  a_0(s, y_{[0:s]}^{(2)} )  ds \\
    & \quad +   
        \int_0^t \Phi_2(s, y^{(1)}_{[0: s]}) d y^{(1)}_{[0:s]}
    - 
        \int_0^t \Phi_2(s, y^{(2)}_{[0 : s]})  d  y^{(2)}_{[0 : s]} \\
    & \quad +
        \int_0^t \Phi_1(s, y^{(1)}_{[0 : s]}) -   \Phi_1(s, y^{(2)}_{[0 : s]}) \mu^{(2)}_s ds
    +
        \int_0^t  \Phi_1(s, y^{(1)}_{[0 : s]})\, (\mu^{(1)}_s - \mu^{(2)}_s ) ds .
\end{align*} 
By the triangle inequality, Jensen's inequality, and Cauchy-Schwarz inequality, we have 
\begin{align} 
 \big\| \mu^{(1)}_t - \mu^{(2)}_t \big\|_2 \leq & \gamma(t) + \int_0^t \big\| \Phi_1(s, y^{(1)}_{[0 : s]}) \big\|_2 \cdot 
 \big\| \mu^{(1)}_s - \mu^{(2)}_s \big\|_2 ds ,\notag 
\end{align} 
where 
\begin{align*} 
\gamma(t) \eqdef & \big\| \mu_0^{(1)} - \mu_0^{(2)} \big\|_2 
  + \underbrace{\Big\| \int_0^t a_0(s, y_{[0:s]}^{(1)}) -  a_0(s, y_{[0:s]}^{(2)} )  ds \Big\|_2}_{\term{t:gammaA}} \\  
& + \underbrace{ \Big\| \int_0^t \Phi_2(s, y^{(1)}_{[0: s]}) d y^{(1)}_{[0:s]}
  - \int_0^t \Phi_2(s, y^{(2)}_{[0 : s]})  d  y^{(2)}_{[0 : s]} \Big\|_2 }_{\term{t:gammaB}} \\ 
  & + \underbrace{\Big\| \int_0^t \Phi_1(s, y^{(1)}_{[0 : s]}) -   \Phi_1(s, y^{(2)}_{[0 : s]}) \mu^{(2)}_s ds \Big\|_2}_{\term{t:gammaC}} .
\end{align*} 
For an application of Gr\"onwall's inequality, it is sufficient to show $\Vert \gamma(t) \Vert_2 < \infty$. We have due to $a_0$ being Lipschitz with constant $K_1 < \infty$,
\begin{align*}
    \eqref{t:gammaA} 
    \leq \int_0^t \Vert a_0(s, y_{[0:s]}^{(1)}) - a_0(s, y_{[0:s]}^{(2)}) \Vert_2 ds  
    \leq \int_0^t K_1 \Vert y_{[0:s]}^{(1)} - y_{[0:s]}^{(2)} \Vert_s ds  
    \leq t K_1 \Vert y^{(1)}_{[0:t]} - y^{(2)}_{[0:t]} \Vert_t.
\end{align*}
Further, we see that 
\begin{align} 
 \eqref{t:gammaB}
 \leq & 
  \underbrace{\Big\| \int_0^t \Phi_2(s, y^{(2)}_{[0:s]}) - \Phi_2(s, y^{(1)}_{[0:s]}) d y^{(2)}_{[0:s]} \Big\|_2}_{\term{t:gammaE}}
  + 
  \underbrace{ \Big\| \int_0^t \Phi_2(s, y^{(1)}_{[0:s]}) 
   d ( y^{(2)}_{[0:s]} - y^{(1)}_{[0:s]} ) \Big\|_2 }_{\term{t:gammaD}} . \notag 
\end{align} 
Using It\^{o}'s isometry, we have 
\allowdisplaybreaks 
\begin{align*} 
    \eqref{t:gammaE}^2 
    & =  \int_0^t 
        \big\| \Phi_2(s, y^{(2)}_{[0:s]}) - \Phi_2(s, y^{(1)}_{[0:s]})\big\|_2^2 
    d \, (B\circ B)(s, y^{(2)}_{[0:s]}) \\
    & \leq 
        K_2^2 \Vert y^{(2)}_{[0:t]} - y^{(1)}_{[0:t]} \Vert^2_t
        \int_0^t d \, (B\circ B)(s, y^{(2)}_{[0:s]})
\end{align*} 
where the last inequality holds due to (path-)Lipschitz continuity of $\Phi_2$ uniformly in time with constant $K_2 < \infty$, given by \Cref{lem:LipschitzPhi2}. Next, let $\lbrace [s_{i-1}:s_{i}] \vert i \in \myset{1}{I}\rbrace$, $I \in \N$ be a fine enough partition of $[0:t]$, we see due to the (time-) Lipschitz continuity of $B \circ B$ together with the definition of the Riemann-Stieltjes integral that 
\begin{align*}
    \eqref{t:gammaE}^2 
    & \leq \epsilon_I + K_2^2 \Vert y^{(1)}_{[0:t]} - y^{(2)}_{[0:t]} \Vert^2_t \cdot \sum_{i = 1}^I \Vert (B\circ B)(y^{(2)}_{[0:s_i]}) - (B\circ B)(y^{(2)}_{[0:s_{i-1}]}) \Vert_2  \\
    & \leq \epsilon_I + K_2^2 T \Vert y^{(1)}_{[0:t]} - y^{(2)}_{[0:t]} \Vert_t^2.
\end{align*}
As $\epsilon_I \to 0$ for $I \to \infty$, we obtain $\eqref{t:gammaE} \leq K_2\sqrt{T} \Vert y^{(1)}_{[0:t]} - y^{(2)}_{[0:t]} \Vert_t$.

%

Integrating by parts, we find that
\begin{align} 
    \eqref{t:gammaD}  
    & = \Big\| \Phi_2(t, y^{(1)}_{[0:t]}) (y^{(2)}_{[0:t]} - y^{(1)}_{[0:t]}) 
        - \Phi_2(0, y^{(1)}_0) (y^{(2)}_0 - y^{(1)}_0)    
        \notag 
    {- \int_0^t ( y^{(2)}_{[0:s]} - y^{(1)}_{[0:s]} ) \, d \Phi_2(s, y^{(1)}_{[0:s]})  \Big\|_2} \notag \\ 
    & \leq \Big\| \Phi_2(t, y^{(1)}_{[0:t]}) (y^{(2)}_{[0:t]} - y^{(1)}_{[0:t]}) 
        - \Phi_2(0, y^{(1)}_0) (y^{(2)}_0 - y^{(1)}_0) \Big\|_2 
        \notag 
    {+ \Big\|  \int_0^t ( y^{(2)}_{[0:s]} - y^{(1)}_{[0:s]} ) \, d \Phi_2(s, y^{(1)}_{[0:s]})  \Big\|_2} 
    \notag \\ 
    & \leq 2 \max_{0 \leq s \leq T} \Big\| \Phi_2(s, y^{(1)}_{[0:s]}) \Big\|_2 
        \cdot \Big\|  y^{(2)}_{[0:s]} - y^{(1)}_{[0:s]} \Big\|_t 
        + \underbrace{\Big\|  \int_0^t ( y^{(2)}_{[0:s]} - y^{(1)}_{[0:s]} ) \, d \Phi_2(s, y^{(1)}_{[0:s]})  \Big\|_2}_{\term{t:gammaF}}. \notag 
\end{align} 
\Cref{lem:LipschitzPhi2} yields (time-)Lipschitz continuity of $\Phi_2$, implying $\max_{s \in [0:T]} \| \Phi_2(s, y^{(1)}_{[0:s]}) \|_2 =: K_3 < \infty$, together with (path-) Lipschitz continuity of $\Phi_2$ uniformly in time with constant $K_4 < \infty$ yields
\begin{align*}
    \eqref{t:gammaF} 
    & \leq \epsilon_I + \sum_{i = 1}^I \big\| y^{(2)}_{[0:s_i]} - y^{(1)}_{[0:s_i]} \big\|_2  \big\Vert \Phi_2(s_i, y^{(1)}_{[0:s_i]}) - \Phi_2(s_{i-1}, y^{(1)}_{[0:s_{i-1}]}) \big\Vert_2 
    \leq \epsilon_I + K_4 t \Vert y^{(2)}_{[0:t]} - y^{(1)}_{[0:t]} \Vert_t ,
\end{align*}
which holds for $\epsilon_I \to 0$ for $I \to \infty$; we conclude that 
\begin{align*}
\eqref{t:gammaB} \leq \left(K_2 \sqrt{T} + K_3 + K_4 T\right) \Vert y^{(2)}_{[0:t]} - y^{(1)}_{[0:t]} \Vert_t.
\end{align*}
Next, we consider
\begin{align*}
    \eqref{t:gammaC} 
    \leq \int_0^t \big\| \Phi_1(s, y^{(1)}_{[0 : s]}) - \Phi_1(s, y^{(2)}_{[0 : s]}) \big\|_2 \big\|\mu^{(2)}_s \big\|_2 ds 
    \leq  K_5 t \Vert y^{(1)}_{[0:t]} - y^{(2)}_{[0:t]} \Vert_t \max_{s \in [0:t]} \big\|\mu^{(2)}_s \big\|_2
\end{align*}
what holds due to (path-)Lipschitz continuity of $\Phi_1$ uniformly in time with constant $K_5$ which is obtained by \Cref{lem:LipschitzPhi1}. We conclude, since $\max_{s \in [0:T]} \|\mu^{(2)}_s \|_2  \eqdef  K_6 < \infty$ due to \Cref{lem:TimeLipschitzMu}, that $\gamma(t) \leq K \Vert y^{(1)}_{[0:t]} - y^{(2)}_{[0:t]} \Vert_t$ with $K  \eqdef  K_1 + K_2 \sqrt{T} + K_3 + K_4 T + K_5 K_6 T$.

With that, we are able to apply Gr\"{o}nwall's inequality and obtain
\begin{align} 
\big\| \mu^{(1)}_t - \mu^{(2)}_t \big\|_2  
    & \leq \gamma(t) + \int_0^t \gamma(s) \big\| \Phi_1(s, y^{(1)}_{[0 : s]}) \big\|_2 
    \exp\Big( \int_s^t \big\| \Phi_1(r, y^{(1)}_{[0 : r]}) \big\|_2 d r \Big) ds  
    \notag   \\
    & \leq K \Vert y^{(1)}_{[0:t]} - y^{(2)}_{[0:t]} \Vert_t 
        {\left( 
            1
            + \int_0^t \big\| \Phi_1(s, y^{(1)}_{[0 : s]}) \big\|_2 
              \exp\Big( \int_s^t \big\| \Phi_1(r, y^{(1)}_{[0 : r]}) \big\|_2 d r \Big) ds   
        \right)}
     \notag  
\end{align} 
Note that $\Phi_1$ is (time-)Lipschitz continuous, as shown in \Cref{lem:LipschitzPhi1}, therefore implying that $\max_{s \in [0:T]} \| \Phi_1(s, y^{(1)}_{[0:s]}) \|_2 < \infty$. As constants $K_1$, $K_2$, $K_3$, $K_3$, $K_4$, $K_5$, and $K_6$ were chosen independently of $t \in [0:T]$ and \Cref{ass:technical_evolution_conditions} \ref{asm:Regularity__LipInit} holds, we conclude the proof. 
\end{proof}\setcounter{termcounter}{0}

\begin{lemma}\label{lem:LipschitzPhi1}
    Let \Cref{ass:technical_evolution_conditions} \ref{asm:Regulartiy__1} hold. Then,
    \begin{enumerate}[label=(\roman*),leftmargin=1cm]
        \item for any path $y \in C([0:T]), \R^d)$ there is a constant $K \in \R$ s.t. for all $t > s \in [0:T]$ holds
        \begin{align*}
            \| \Phi_1(t, y_{[0:t]}) - \Phi_1(s, y_{[0:s]}) \|_2 \leq K \vert t-s \vert. 
        \end{align*}
        where $\Phi_1$ is defined as in \eqref{id:Phi1}.
        \item Also, for every path $y^{(1)}_\cdot \in C^{\color{blue} 1}([0:T], \R^{d_Y})$ there exists a constant $C\ge 0$ such that for all $\epsilon > 0$, paths $y^{(2)}_\cdot \in C^{\color{blue} 1}([0:T], \R^{d_Y})$ with $\| y^{(1)}_{[0:T]} - y^{(2)}_{[0:T]} \|_T < \epsilon$, and times $t \in [0:T]$ holds
        \begin{align*}  
            \big\| \Phi_1(t, y^{(1)}_{[0:t]}) - \Phi_1(t, y^{(2)}_{[0:t]}) \big\|_2 
            \leq C \big\| y^{(1)}_{[0:t]} - y^{(2)}_{[0:t]} \big\|_t.
        \end{align*} 
    \end{enumerate}
\end{lemma}
\begin{proof}\textbf{of Lemma~\ref{lem:LipschitzPhi1}}
\begin{enumerate}[wide]
    \item The statement holds if $a_1$, $ (b\circ B)(B\circ B)^{-1}A_0$, $ (b\circ B)(B\circ B)^{-1}A_1$, $ \Sigma_\cdot A_1(B\circ B)^{-1}A_0$, $ \Sigma_\cdot A_1(B\circ B)^{-1}A_1$ are globally Lipschitz continuous w.r.t. the time component. As $a_1$, $ (b \circ B)$, $ \Sigma_\cdot$, $ (B\circ B)^{-1}$, $ A_0$, $ A_1$ satisfy this already (see \Cref{lem:TimeLipschitzSigma}) and with this are also bounded on the interval $[0:T]$, their product is globally Lipschitz continuous w.r.t. time as well. 
    \item The assertion holds if $a_1$, $ (b\circ B)(B\circ B)^{-1}A_0$, $ (b\circ B)(B\circ B)^{-1}A_1$, $ \Sigma_\cdot A_1(B\circ B)^{-1}A_0$, $ \Sigma_\cdot A_1(B\circ B)^{-1}A_1$ are locally Lipschitz continuous w.r.t. the path component uniformly in time. As $a_1$, $ (b \circ B)$, $ \Sigma_\cdot$, $ (B\circ B)^{-1}$, $ A_0$, $ A_1$ are locally Lipschitz continuous w.r.t. the path component (see \Cref{lem:LipschitzSigma}) uniformly in time, this also follows for their product. 
\end{enumerate}
\end{proof}

\begin{lemma}\label{lem:LipschitzPhi2}
    Let \Cref{ass:technical_evolution_conditions} \ref{asm:Regulartiy__1} be fulfilled. Then,
    \begin{enumerate}[label=(\roman*),leftmargin=1cm]
        \item for any path $y \in C([0:T]), \R^d)$ there is a constant $C\ge 0$ s.t. for all $t > s \in [0:T]$ holds
        \begin{align*}
            \| \Phi_2(t, y_{[0:t]}) - \Phi_2(s, y_{[0:s]}) \|_2 \leq C \vert t-s \vert. 
        \end{align*}
        where $\Phi_2$ is defined as in \eqref{id:Phi2}.
        \item Additionally, for every path $y^{(1)}_\cdot \in C^{\color{blue} 1}([0:T], \R^{d_Y})$ there exists constant $C\ge 0$, such that for all $\epsilon > 0$,  paths $y^{(2)}_\cdot \in C^{\color{blue} 1}([0:T], \R^{d_Y})$ with $\| y^{(1)}_{[0:t]} - y^{(2)}_{[0:t]} \|_t < \epsilon$, and times $t \in [0:T]$ holds
        \begin{align*}  
            \big\| \Phi_2(t, y^{(1)}_{[0:t]}) - \Phi_2(t, y^{(2)}_{[0:t]}) \big\|_2 
            \leq C \big\| y^{(1)}_{[0:t]} - y^{(2)}_{[0:t]} \big\|_t.
        \end{align*} 
    \end{enumerate}
\end{lemma}
\begin{proof}\textbf{of Lemma~\ref{lem:LipschitzPhi2}}
    \begin{enumerate}[wide]
        \item The assertion holds if $a_0$, $(b \circ B)(B\circ B)^{-1}$, $\Sigma_\cdot A_1(B\circ B)^{-1}$ are globally Lipschitz continuous w.r.t. the time component. As $a_0$, $(b \circ B)$, $\Sigma_\cdot$, $ (B\circ B)^{-1}$, $ A_1$ satisfy this already (see \Cref{lem:TimeLipschitzSigma}) and with this are also bounded on the interval $[0:T]$, their product is globally Lipschitz continuous w.r.t. time as well. 
        \item The statement holds if $a_0$, $(b \circ B)(B\circ B)^{-1}$, $\Sigma_\cdot A_1(B\circ B)^{-1}$ are locally Lipschitz continuous w.r.t. the path component. As $a_0$, $(b \circ B)$, $\Sigma_\cdot$, $ (B\circ B)^{-1}$, $ A_1$ are locally Lipschitz continuous w.r.t. the path component (see \Cref{lem:LipschitzSigma}) this also follows for their product. 
    \end{enumerate}
\end{proof}

\begin{lemma}[Local Lipschitz-continuity of $\Sigma$]\label{lem:LipschitzSigma}
    Under \Cref{ass:technical_evolution_conditions} \ref{asm:Regulartiy__1}, $\Sigma_\cdot$ from \eqref{id:ClosedFormSigma} is locally Lipschitz-continuous uniformly in time. This means, for every path $y^{(1)}_\cdot \in C^{\color{blue} 1}([0:T], \R^{d_Y})$ there exists constant $C \in \R$ such that for all $\epsilon > 0$, paths $y^{(2)}_\cdot \in C^{\color{blue} 1}([0:T], \R^{d_Y})$ with $\| y^{(1)}_{[0:t]} - y^{(2)}_{[0:t]} \|_t < \epsilon$, and times $t \in [0:T]$ holds
\begin{align*}  \big\| \Sigma^{(1)}_t - \Sigma^{(2)}_t \big\|_2 
 \leq C \big\| y^{1}_{[0:t]} - y^{2}_{[0:t]} \big\|_t ,
 \end{align*} 
 where $\Vert \cdot \Vert_2$ on the left-hand side refers to the Frobenius norm.
\end{lemma} 
\begin{proof}\textbf{of Lemma~\ref{lem:LipschitzSigma}}
From \eqref{id:ClosedFormSigma} and for $t \in [0:T]$, we obtain the integral representation of $\Sigma_t$
\allowdisplaybreaks
\begin{align} 
        \Sigma_t
    =
        & \Sigma_0 + \int_0^t
                a_1(t,y_{[0:s]}) \Sigma_s
            +
                \Sigma_s a_1^{\top}(s,y_{[0:s]})
            +
                (b\circ b)(s,y_{[0:s]}) \hspace{3cm}
            \notag\\
                &\pushright{- [
                    (b\circ B)(s,y_{[0:s]})
                +
                    \Sigma_s A_1^{\top}(s,y_{[0:s]})
                ]
                (B\circ B)^{-1} (s,y_{[0:s]}) } \hspace{1cm} \notag \\
                &\pushright{\times[
                    (b\circ B)(s,y_{[0:s]}) 
                +
                    \Sigma_s A_1^{\top}(s,y_{[0:s]})
                ]^{\top}
        ds.} \notag  
\end{align} 
Let $\Sigma^{(1)}_t$, $\Sigma^{(2)}_t$ be the covariance matrices corresponding to paths $y^{(1)}_\cdot, y^{(2)}_\cdot \in C^{\color{blue} 1}([0:T], \R^{d_Y})$ where $y^{(1)}_\cdot$ is arbitrary and $y^{(2)}_\cdot$ s.t. $\| y^{(1)}_{[0:t]} - y^{(2)}_{[0:t]} \|_t < \epsilon$. The difference $\Sigma^{(1)}_t - \Sigma^{(2)}_t$ satisfies the integral representation
\allowdisplaybreaks 
\begin{multline*} 
\Sigma^{(1)}_t - \Sigma^{(2)}_t
 = \Sigma^{(1)}_0 - \Sigma^{(2)}_0 
  + \int_0^t 
 \Xi_1(s, y^{(1)}_{[0:s]}, y^{(2)}_{[0:s]}, \Sigma^{(1)}_s, \Sigma^{(2)}_s)  d s 
\\ 
 + \int_0^t (\Sigma^{(1)}_s - \Sigma^{(2)}_s) 
\Xi_2(s, y^{(1)}_{[0:s]}, y^{(2)}_{[0:s]}, \Sigma^{(1)}_s, \Sigma^{(2)}_s) 
 + 
\Xi_2^\top (s, y^{(1)}_{[0:s]}, y^{(2)}_{[0:s]}, \Sigma^{(1)}_s, \Sigma^{(2)}_s)  (\Sigma^{(1)}_s - \Sigma^{(2)}_s) ds , 
\end{multline*} 
where we denote 
\allowdisplaybreaks
\begin{align} 
        & \Xi_1(s, y^{(1)}_{[0:s]}, y^{(2)}_{[0:s]}, \Sigma^{(1)}_s, \Sigma^{(2)}_s)\notag \\ 
    & \quad\quad\quad \eqdef  
            (a_1(s, y^{(1)}_{[0:s]}) - a_1(s, y^{(2)}_{[0:s]}) ) \Sigma^{(1)}_s
        +
            \Sigma^{(1)}_s (a_1^\top(s, y^{(1)}_{[0:s]}) - a_1^\top(s, y^{(2)}_{[0:s]})) \notag \\
        & \quad\quad\quad\quad +
            (b\circ b)(s, y^{(1)}_{[0:s]}) - (b\circ b) (s, y^{(2)}_{[0:s]}) \notag\\ 
        & \quad\quad\quad\quad -
            \Sigma^{(2)}_s [
                    A_1^\top (B\circ B)^{-1} (b\circ B)(s, y^{(1)}_{[0:s]}) 
                - 
                    A_1^\top (B\circ B)^{-1} (b\circ B)(s, y^{(2)}_{[0:s]})
            ] \notag\\ 
        & \quad\quad\quad\quad - 
            [
                    (b\circ B) (B\circ B)^{-1} A_1(s, y^{(1)}_{[0:s]}) 
                -
                    (b\circ B) (B\circ B)^{-1} A_1(s, y^{(2)}_{[0:s]}) 
             ] \Sigma^{(2)}_s \notag\\ 
        & \quad\quad\quad\quad - 
            \Sigma^{(2)}_s [
                    A_1^\top (B\circ B)^{-1} A_1(s,y^{(1)}_{[0:s]}) 
                -
                    A_1^\top (B\circ B)^{-1} A_1(s,y^{(2)}_{[0:s]})
            ] \Sigma^{(1)}_s \notag\\ 
        & \quad\quad\quad\quad - 
            [
                    (b\circ B) (B\circ B)^{-1} (b\circ B)^\top(s, y^{(1)}_{[0:s]})
                -
                    (b\circ B) (B\circ B)^{-1} (b\circ B)^\top(s, y^{(2)}_{[0:s]}) 
            ] \label{Xi1}
\end{align} 
and 
\allowdisplaybreaks 
\begin{multline} 
\Xi_2(s, y^{(1)}_{[0:s]}, y^{(2)}_{[0:s]}, \Sigma^{(1)}_s, \Sigma^{(2)}_s) \eqdef  a_1^\top(s, y^{(2)}_s) 
 + A_1^\top(s, y^{(1)}_{[0:s]}) 
 (B\circ B)^{-1}(s, y^{(1)}_{[0:s]})  \\ 
 \times(b\circ B) (s, y^{(1)}_{[0:s]})  
 + A_1^\top(s, y^{(1)}_{[0:s]}) 
 (B\circ B)^{-1}(s, y^{(1)}_{[0:s]}) A_1(s, y^{(1)}_{[0:s]}) \Sigma^{(1)}_s
 .   \label{Xi2} 
\end{multline} 
By the triangular inequality, Jensen's inequality, and Cauchy-Schwarz inequality, we have 
\begin{align} 
 \big\| \Sigma^{(1)}_t - \Sigma^{(2)}_t \big\|_2 
 \leq & \big\| \Sigma^{(1)}_0 - \Sigma^{(2)}_0 \big\|_2 
  + \int_0^t \big\| \Xi_1(s, y^{(1)}_{[0:s]}, y^{(2)}_{[0:s]}, \Sigma^{(1)}_s, \Sigma^{(2)}_s) \big\|_2 ds \notag \\ 
  & + 2 \int_0^t \big\| \Sigma^{(1)}_s - \Sigma^{(2)}_s \big\|_2 \cdot \big\| \Xi_2(s, y^{(1)}_{[0:s]}, y^{(2)}_{[0:s]}, \Sigma^{(1)}_s, \Sigma^{(2)}_s)  \big\|_2 
   ds . \notag 
\end{align}  
By Gr\"{o}nwall's inequality, it holds that 
\begin{multline} 
    \big\| \Sigma^{(1)}_t - \Sigma^{(2)}_t \big\|_2 
    \leq  
    \big\| \Sigma^{(1)}_0 - \Sigma^{(2)}_0 \big\|_2 + \int_0^t \big\| \Xi_1(s, y^{(1)}_{[0:s]}, y^{(2)}_{[0:s]}, \Sigma^{(1)}_s, \Sigma^{(2)}_s) \big\|_2 ds 
     \\ 
  + \int_0^t 
 \Big( \big\| \Sigma^{(1)}_0 - \Sigma^{(2)}_0 \big\|_2 + \int_0^s \big\| \Xi_1(r, y^{(1)}_{[0:r]}, y^{(2)}_{[0:r]}, \Sigma^{(1)}_r, \Sigma^{(2)}_r) \big\|_2 d r 
 \Big)  \cdot  \\ 
 \cdot 2 \big\|  \Xi_2(s, y^{(1)}_{[0:s]}, y^{(2)}_{[0:s]}, \Sigma^{(1)}_s, \Sigma^{(2)}_s) \big\|_2 \cdot \\ \cdot \exp \Big( 2 \int_s^t \big\|  \Xi_2(r, y^{(1)}_{[0:r]}, y^{(2)}_{[0:r]}, \Sigma^{(1)}_r, \Sigma^{(2)}_r) \big\|_2 
 d r  \Big) d s . \label{Gronwall:Sigma1-Sigma2} 
\end{multline}  
From \Cref{lem:LipschitzA1}, we know that there exists a constant $K \in \R$ such that
\begin{align} 
 \big\|  \Xi_1(s, y^{(1)}_{[0:s]}, y^{(2)}_{[0:s]}, \Sigma^{(1)}_s, \Sigma^{(2)}_s)  \big\|_2 
  \leq K \big\| y^{(1)}_{[0:s]} - y^{(2)}_{[0:s]}  \big\|_s . 
  \label{Xi1Lip}
\end{align} 
Combining \eqref{Gronwall:Sigma1-Sigma2} and \eqref{Xi1Lip}, we obtain
\allowdisplaybreaks 
\begin{multline}
        \big\| \Sigma^{(1)}_t - \Sigma^{(2)}_t \big\|_2  
    \leq  
        \big( \big\| \Sigma^{(1)}_0 - \Sigma^{(2)}_0 \big\|_2  + KT \big\| y^{(1)} - y^{(2)} \big\|_T \big)  \cdot 
        \Big[ 1+  \int_0^t 2 \big\| \Xi_2(s, y^{(1)}_{[0:s]}, y^{(2)}_{[0:s]}, \Sigma^{(1)}_s, \Sigma^{(2)}_s)  \big\|_2 \cdot 
        \\
        \cdot \exp \Big( 2 \int_s^t \big\|  \Xi_2(r, y^{(1)}_{[0:r]}, y^{(2)}_{[0:r]}, \Sigma^{(1)}_r, \Sigma^{(2)}_r) \big\|_2 
 d r  \Big) d s \Big] .  
 \label{Sigma1-2bddXi2}
\end{multline} 
Since $a_1$, $A_1$, $(B\circ B)^{-1}$ and $b\circ B$ are globally (time-) Lipschitz continuous, $\Xi_2$ is uniformly bounded, according to \eqref{Xi2}. Therefore, from \eqref{Sigma1-2bddXi2} we obtain a constant $C \in \R$ such that 
\begin{align} 
 \big\| \Sigma^{(1)}_t - \Sigma^{(2)}_t \big\|_2  
\leq 
C ( \big\| \Sigma^{(1)}_0 - \Sigma^{(2)}_0 \big\|_2 + \big\| y^{(1)} - y^{(2)} \big\|_T ) .  \notag  
\end{align}
We conclude the proof due to \Cref{ass:technical_evolution_conditions} \ref{asm:Regularity__LipInit}.
\end{proof} 

\begin{lemma}\label{lem:LipschitzA1}
Let $\Xi_1$ be defined as in \eqref{Xi1}, and let \Cref{ass:technical_evolution_conditions} \ref{asm:Regulartiy__1} hold.
Then, for every path $y^{(1)}_\cdot \in C^{\color{blue} 1}([0:T], \R^{d_Y})$ there exists constant $K \in \R$, such that for all  $\epsilon > 0$, paths $y^{(2)}_\cdot \in C^{\color{blue} 1}([0:T], \R^{d_Y})$ with $\| y^{(1)}_{[0:t]} - y^{(2)}_{[0:t]} \|_t < \epsilon$, and times $s \in [0:T]$ holds
\begin{align*} 
         \big\|  \Xi_1(s, y^{(1)}_{[0:s]}, y^{(2)}_{[0:s]}, \Sigma^{(1)}_s, \Sigma^{(2)}_s)  \big\|_2 
     \leq
         K \big\| y^{(1)}_{[0:s]} - y^{(2)}_{[0:s]}  \big\|_s . 
\end{align*} 
\end{lemma}
\begin{proof}\textbf{of Lemma~\ref{lem:LipschitzA1}}
    Since $a_1$, $b \circ b$, $b \circ B$, $A_1$, $(B \circ B)^{-1}$ are locally Lipschitz continuous in the path component uniformly over all times, we conclude that $a_1$, $b \circ b$, $A_1^\top(B \circ B)^{-1}(b \circ B)$, $(b \circ B)(B \circ B)^{-1}A_1$, $A_1^\top (B \circ B)^{-1}A_1$, and $(b \circ B)(B \circ B)^{-1}(b \circ B)^\top$ have the same property. By this, together with the triangular inequality, there exists a constant $K_1 \in \R$
    \begin{align*}
    &
            \big\Vert \Xi_1(s, y^{(1)}_{[0:s]}, y^{(2)}_{[0:s]}, \Sigma^{(1)}_s, \Sigma^{(2)}_s) \big \Vert  
    \\& \quad
        \leq  
                2K_1 \big\Vert y^{(1)}_{[0:s]} - y^{(2)}_{[0:s]}  \big\Vert_s \Big(\big\Vert\Sigma^{(1)}_s\big\Vert_2 + \big\Vert\Sigma^{(2)}_s \big\Vert_2\Big) 
            +
                K_1 \big\Vert y^{(1)}_{[0:s]} - y^{(2)}_{[0:s]}  \big\Vert_s \big\Vert\Sigma^{(1)}_s\big\Vert_2\big\Vert\Sigma^{(2)}_s\big\Vert_2
    \\ &
            \pushright{
            +
                K_1 \big\Vert y^{(1)}_{[0:s]} - y^{(2)}_{[0:s]}  \big\Vert_s 
            }
    \\& \quad  
        =
            K_1 \big\Vert y^{(1)}_{[0:s]} - y^{(2)}_{[0:s]}  \big\Vert_s
            \Big(
                    1
                +
                    2 \Vert\Sigma^{(1)}_s \big\Vert_2
                +
                    2 \Vert\Sigma^{(2)}_s \big\Vert_2 
                + 
                    \big\Vert\Sigma^{(1)}_s\big\Vert_2 \big\Vert\Sigma^{(2)}_s \big\Vert_2
            \Big).
    \end{align*}
With that, it is left to show that $\Vert\Sigma^{(1)}_s \Vert_2 < \infty$ and $\Vert\Sigma^{(2)}_s \Vert_2 < {K}_2$ for some $K_2 \in \R$. The former is trivially fulfilled due to the continuity of $a_1, b \circ b, b \circ B, A_1, (B \circ B)^{-1}$. 
Since $\Sigma_t$ is continuous in the path-component and the closure of the set 
 $
    \{y \in C^{1}([0:T]) \,:\, \Vert y_{[0:t]} - y^{(1)}_{[0:t]} \Vert_2 < \epsilon \}
$ 
is compact, we know $\Vert\Sigma^{(2)}_s \Vert_2$ is bounded.
\end{proof}

\subsubsection{Lipschitz-Continuity in the Time-Component of the Optimal Filter}

\begin{lemma}[Global Lipschitz-Cont. in the Time-Comp. of the Optimal Filter]
\label{prop:LocLipReg_OptimFilter_Time}
    \sloppy Under \Cref{ass:technical_evolution_conditions}, $f_t$ from \eqref{goal:Robust_Filtering} is locally Lipschitz-continuous in the time-component. In particular, for every path $y_\cdot \in C^{\color{blue} 1}([0:T], \R^{d_Y})$, $t \in [0:T]$, and  $\epsilon > 0$ there exists constant $C \ge 0$ such that for all times $s \in [0:T]$ with $\| t - s \|_t < \epsilon$ holds
    \begin{align*}
        \mathcal{W}_2(\mathbb{P}(X_t \in \cdot \vert y_{[0:t]}), \mathbb{P}(X_s \in \cdot \vert y_{[0:s]})) \leq C \vert t - s \vert.
    \end{align*}
\end{lemma}
\begin{proof}\textbf{of Lemma~\ref{prop:LocLipReg_OptimFilter_Time}}
    As in the proof of \Cref{prop:LocLipReg_OptimFilter_Path}, we argue that by \eqref{id:LipsterShiryaevRepresentation}, \Cref{lem:LoewnerBound}, and \Cref{lem:LipschitzManifold_Nd} there exists a non-negative constant $\bar{K} < \infty$ satisfying the following for every path $y_{\cdot} \in C([0:T],\|\cdot\|_2)$ and all times $t, s \in [0:T]$
    \begin{equation*}
        \mathcal{W}_2(\mathbb{P}(X_t \in \cdot \vert y_{[0:t]}), \mathbb{P}(X_s \in \cdot \vert y_{[0:s]}))
        \leq 
            \bar{K}
                \,
            \sqrt{\| \mu(y_{[0:t]}) - \mu(y_{[0:s]}) \|_2^2 + \|\Sigma(y_{[0:t]})-\Sigma(y_{[0:s]} \|_2^2}.
    \end{equation*}
    We see from \Cref{lem:TimeLipschitzMu} and \Cref{lem:TimeLipschitzSigma} that there exists a non-negative Lipschitz-constant $K < \infty$ for both $\mu$ and $\Sigma$ such that $
        \mathcal{W}_2(\mathbb{P}(X_t \in \cdot \vert y_{[0:t]}), \mathbb{P}(X_s \in \cdot \vert y_{[0:s]})) \leq \sqrt{2}\bar{K}K \vert t - s \vert,
    $ 
    and the statement follows. 
\end{proof}

\begin{lemma}[Global Lipschitz-continuity of $\mu$ in time]\label{lem:TimeLipschitzMu}
    Let the \Cref{ass:technical_evolution_conditions} \ref{asm:Regulartiy__1} be fulfilled. Then, for any path $y \in C([0:T]), \R^d)$ there is a constant $K \in \R$ s.t. for all $t > s \in [0:T]$ holds: 
    $
        \| \mu_t - \mu_s \|_2 \leq K \vert t-s \vert 
    $
    where $\mu$ is defined as in \eqref{id:ClosedFormMu}.
\end{lemma}
\begin{proof}\textbf{of Lemma~\ref{lem:TimeLipschitzMu}}
    With the notation of \Cref{lem:LipschitzMu}, we have
    \begin{equation*}
            \| \mu_t - \mu_s \|_2 
        \leq 
            \underbrace{\Big\Vert \int_s^t a_0(u, y_{[0:u]})du \Big\Vert_2}_{\term{t:muA}}
          + \underbrace{\Big\Vert\int_s^t \Phi_2(u, y_{[0:u]})dy_{[0:u]} \Big\Vert_2}_{\term{t:muB}}
          + \underbrace{\Big\Vert \int_s^t \Phi_1(u, y_{[0:u]})\mu_u du \Big \Vert_2}_{\term{t:muC}}.
    \end{equation*}
    We see that $\eqref{t:muA} \leq K \vert t - s \vert$ as $\max_{u \in [0:T]}\Vert a_0(u, y_{[0:u]}) \Vert_2 < \infty$ due to (time-)Lipschitz-continuity of $a_0$.

    By the use of integration by parts,
    \begin{align*}
        \eqref{t:muB}  
        & = 
        \Big\| 
              \Phi_2(t, y_{[0:t]}) y_{[0:t]} 
            - \Phi_2(s, y_{[0:s]}) y_{[0:s]} 
            - \int_0^t y_{[0:u]} \, d \Phi_2(u, y_{[0:u]})
        \Big\|_2 
        \\
        & \leq 
            \underbrace{
            \Big\| 
                  \Phi_2(t, y_{[0:t]}) y_{[0:t]} 
                - \Phi_2(0, y_0) y_0 
            \Big\|_2
            }_{\term{t:muB1}}
        +
            \underbrace{
            \Big\|
                \int_0^t y_{[0:u]} \, d \Phi_2(u, y_{[0:u]})
            \Big\|_2 
            }_{\term{t:muB2}},
    \end{align*}
    we see that $\eqref{t:muB1} \leq K \Vert y^{(1)}_{[0:t]} - y^{(2)}_{[0:t]} \Vert_t$ due to (path-)Lipschitz continuity of $\Phi_2$. 
    Let $\mySet{[u_{i-1}:u_{i}]}{i \in \myset{1}{I}}, I \in \N$ be a fine enough partition of $[s:t]$, then we see due to global (time-)Lipschitz continuity of $\Phi_2$ and the definition of the Riemann-Stieltjes integral that 
    \begin{align*}
        \eqref{t:muB2} 
        & \leq \epsilon_I + \sum_{i = 1}^I   \Vert y_{[0:u_i]} \Vert_2  \big\| \Phi_2(u_i, y_{[0:u_i]}) - \Phi_2(u_{i-1}, y_{[0:u_{i-1}]})  \big\|_2 
        \leq \epsilon_I + \Vert y_{[0:t]} \Vert_t \sum_{i = 1}^I K \vert u_i - u_{i-1} \vert
    \end{align*}
    and, since $\epsilon_I \to 0$ for $I \to \infty$, it follows $\eqref{t:muB} \leq K \vert t - s \vert$.  Since $\eqref{t:muA} < \infty$ and $\eqref{t:muB} < \infty$, we can apply Gr\"onwall's inequality to 
    \begin{align*}
            \| \mu_t \|_2 
        \le &
            \overbrace{\Big\Vert \int_0^t a_0(u, y_{[0:u]})du \Big\Vert_2
          + \Big\Vert\int_0^t \Phi_2(u, y_{[0:u]})dy_{[0:u]} \Big\Vert_2}^{ \eqdef  \alpha(t)} 
          + \int_0^t \big\Vert \Phi_1(u, y_{[0:u]})\big\Vert_2 \big\Vert \mu_u \big\Vert_2 du,
    \end{align*}
    which yields
    \begin{align*}
            \| \mu_t \|_2 
        & \leq 
            \alpha(t) + \int_0^t \alpha(s)\big\Vert \Phi_1(s, y_{[0:s]}) \big \Vert_2
            \exp\left(
                \int_s^t \big\Vert \Phi_1(u, y_{[0:u]}) \big \Vert_2 du 
            \right) ds 
        \leq 
            K t 
    \end{align*}
    where the last equality holds due to $\alpha(s) \leq K s$ for all $s \in [0:T]$ and (time-)Lipschitz continuity of $\Phi_1$, implying $\max_{u \in [0:t]}\Vert \Phi_1(u, y_{[0:u]}) \Vert_2 < \infty$. 
    With that, we have
    \begin{align*}
        \eqref{t:muC} \leq  \int_s^t \big\Vert \Phi_1(u, y_{[0:u]})\big\Vert_2 \big\Vert \mu_u \big \Vert_2du   \leq \max_{u \in [0:t]}\Vert \Phi_1(u, y_{[0:u]}) \Vert_2 K T \vert t - s \vert
    \end{align*}
    what proves the claim.
\end{proof}

\begin{lemma}[Global Lipschitz-continuity of $\Sigma$ in time]\label{lem:TimeLipschitzSigma}
    Let \Cref{ass:technical_evolution_conditions} \ref{asm:Regulartiy__1} hold. Then, 
    having $\Sigma$ is defined as in \eqref{id:ClosedFormSigma}, 
    for any path $y \in C([0:T]), \R^d)$ there is a constant $K \in \R$ s.t. for all $t > s \in [0:T]$ holds
    \begin{align*}
        \| \Sigma_t - \Sigma_s \|_2 \leq K \vert t-s \vert. 
    \end{align*}
\end{lemma}
\begin{proof}\textbf{of Lemma~\ref{lem:TimeLipschitzSigma}}
    Note that $\Sigma_\cdot$ is continuous in time. From \eqref{id:ClosedFormSigma}, we obtain
\allowdisplaybreaks
\begin{align*} 
\Vert \Sigma_t  - \Sigma_s \Vert_2 
    = & \int_s^t
          \Vert a_1(t,y_{[0:u]}) \Sigma_u \Vert_2
        + \Vert \Sigma_u a_1^{\top}(u,y_{[0:u]}) \Vert_2
        + \Vert (b\circ b)(u,y_{[0:u]}) \Vert_2 \\
      & \quad \quad
        + \Vert  (b\circ B)(u,y_{[0:u]})
                 (B\circ B)^{-1} (u,y_{[0:u]})
                 (b\circ B)(u,y_{[0:u]})^{\top} \Vert_2 \\
      & \quad \quad
        + \Vert  (b\circ B)(u,y_{[0:u]})
                 (B\circ B)^{-1} (u,y_{[0:u]})
                 \Sigma_u A_1^{\top}(u,y_{[0:u]})^{\top} \Vert_2 \\
      & \quad \quad
        + \Vert  \Sigma_u A_1^{\top}(u,y_{[0:u]}) 
                 (B\circ B)^{-1} (u,y_{[0:u]})
                 (b\circ B)(u,y_{[0:u]})^{\top} \Vert_2 \\
      & \quad \quad
        + \Vert  \Sigma_u A_1^{\top}(u,y_{[0:u]}) 
                 (B\circ B)^{-1} (u,y_{[0:u]})
                 \Sigma_u A_1^{\top}(u,y_{[0:u]})^{\top} \Vert_2
    du .
\end{align*} 
As all components are continuous, we can bound them by their maximal value on $[0:T]$. 
\end{proof}

\subsubsection{Lower Loewner order bound}

\begin{lemma}[Lower Loewner order bound]\label{lem:LoewnerBound}
    Let $\Sigma$ be defined as in \Cref{id:ClosedFormSigma} and \Cref{ass:technical_evolution_conditions} \ref{asm:Regulartiy__2}, \ref{asm:Regularity__3}, hold. There is an $r > 0$, s.t. for all times $t \in [0:T]$ and paths  $y_\cdot \in C^{\color{blue} 1}([0:T], \R^{d_Y})$ holds that 
     $r I_d \preccurlyeq \Sigma_t$ (where $\preccurlyeq$ is the Loewner bound).
\end{lemma}
\begin{proof}\textbf{of Lemma~\ref{lem:LoewnerBound}}
    By the Courant-Fisher Theorem, see \cite{hornjohnson2012}, the statement is equivalent to finding a lower bound on the eigenvalues of $\Sigma$, which we will show by finding an upper bound on the Eigenvalues of $\Sigma^{-1}$, which exists as it is positive definite.
    Following the proof of \cite[Theorem 12.7]{LipShiryaev_II_2001}, we note that 
    \begin{align*}
            \partial_t\Sigma_t^{-1} 
        = &
            - \Tilde{a}_1^\top(t, y_{[0:t]})\Sigma_t^{-1}
            - \Sigma_t^{-1}\Tilde{a}_1(t, y_{[0:t]})
            + A_1^\top(B \circ B)^{-1}A_1(t, y_{[0:t]}) \\
        &   - \Sigma_t^{-1}\left[
                  (b \circ b)(t, y_{[0:t]})
                - (b \circ B)(B \circ B)(b \circ B)^\top(t, y_{[0:t]})
            \right]\Sigma_t^{-1}
    \end{align*}
    with $\Tilde{a}_1^\top(t, y_{[0:t]}) \eqdef a_1(t, y_{[0:t]}) - (b \circ B)(B \circ B)A_1(t, y_{[0:t]})$. 
    Now, let $G_t(y_{[0:t]})$ be a solution of $\partial_tG_t(y_{[0:t]}) = \Tilde{a}_1(t, y_{[0:t]})G_t(y_{[0:t]})$ with $G_0(y_{[0:t]}) = I_{d_X}$. 
    Then, as pointed out in \cite[Theorem 12.7]{LipShiryaev_II_2001}, we arrive due to $(b \circ b)(t, y_{[0:t]})- (b \circ B)(B \circ B)(b \circ B)^\top(t, y_{[0:t]})$ being positive semi-definite at 
    \begin{equation*}
            \tr( G_t(y_{[0:t]}) \Sigma_t^{-1} G_t(y_{[0:t]})^\top ) 
        \leq 
                \tr(\Sigma_0^{-1})
            +
                \int_0^T
                    \tr\left(
                        G_s(y_{[0:s]})
                        A_1^\top(B \circ B)^{-1}A_1(s, y_{[0:s]})
                        G_s(y_{[0:s]})^\top
                    \right)
                ds.
    \end{equation*}
    As $G_t(y_{[0:t]})$ is positive definite, we conclude by submultiplicativity of positive semi-definite matrices that 
    \begin{align*}
            \tr(\Sigma_t^{-1})
        &\leq
            \tr(
                G_t(y_{[0:t]})^{-1}
                G_t(y_{[0:t]}) \Sigma_t^{-1} G_t(y_{[0:t]})^\top
                (G_t(y_{[0:t]})^{-1})^\top
            ) \\
        &\leq 
            \tr(G_t(y_{[0:t]})^{-1})^2
            \tr(
                G_t(y_{[0:t]}) \Sigma_t^{-1} G_t(y_{[0:t]})^\top
            )
    \end{align*}
    and $%
            \tr\left(
                G_s(y_{[0:s]})
                A_1^\top(B \circ B)^{-1}A_1(s, y_{[0:s]})
                G_s(y_{[0:s]})^\top
            \right)
        \leq 
            \tr\left(
                G_s(y_{[0:s]})
            \right)^2
            \tr\left(
                A_1^\top(B \circ B)^{-1}A_1(s, y_{[0:s]})
            \right)
    $.  
    As we assumed that there exist constants $K_1, K_2 > 0$ s.t. uniformly $K_1 \leq \tr(G_t(y_{[0:t]})) \leq K_2$ as well as $K_3 > 0$ s.t. $\tr(A_1^\top(B \circ B)^{-1}A_1(t, y_{[0:t]})) \leq K_3$ we conclude $\tr(\Sigma_t^{-1}) < K$ uniformly for some constant $K > 0$. As the trace is the sum of all Eigenvalues, this proves the claim.
\end{proof}


\subsection{Stable and Lossless Encoding by Pathwise Attention - Proof of \texorpdfstring{\Cref{prop:Lossless_Encoding}}{Proposition 2}}
\label{s:Proof__ss:Encoding}

In the case where the ``latent geometry'' of $K$ has additional structure, we may guarantee that the attention mechanism $\operatorname{attn}_T$ is linearly-stable with linearly-stable inverse, how to generate the reference paths $y^{(1)}$, $\dots$, $y^{(\rdim)}$ in $C([0:T],\mathbb{R}^{d_Y})$, and quantitative estimates on how many paths must be generated.  

\begin{lemma}[Stable Lossless Feature Maps -- Riemannian Case]
\label{lem:Good_Feature_Map}
\sloppy Under \Cref{ass:reg_compact} \ref{ass:reg_compact_riemann}, there exists a constant $C_{\ldim}>0$, depending only on $\ldim$, and a $\delta>0$ depending only on $(M,g)$ such that any $\delta$-packing $\{y^{(n)}\}_{n=1}^\ldim$ of $K$, there are matrices $A,B,b,V,U,C$ and vectors $a,b$ for which the parameter $\theta$, as in \Cref{defn:PathAttention}, defining a pathwise attention mechanism $\operatorname{attn}_T^{\theta}:C([0:T],\mathbb{R}^{d_Y})\rightarrow \mathbb{R}^{\edim}$ which restricts to a bi-Lipschitz embedding of $(K,\|\cdot\|_T)$ into $(\mathbb{R}^{\edim+1},\|\cdot\|_2)$.
Furthermore, $\|\theta\|_0 \in \mathcal{O}(\edim^2)$.

\noindent Moreover, if $(\mathcal{M},g)$ is aspherical then:
\begin{enumerate}[label=(\roman*),leftmargin=1cm]
    \item[(i)] $\delta \in \mathcal{O}( \operatorname{Vol}(\mathcal{M},g)^{1/\ldim} )$
    \item[(ii)] $\|\theta\|_0 \in \mathcal{O}(\edim^2)$,
    \item[(iii)] $C_{\ldim}\in \mathcal{O}(\ldim^{3(\ldim+1)/2})$
\end{enumerate}
The attention mechanism $\operatorname{attn}_T$ satisfies \citep[Setting 3.6 (i)]{kratsios2023transfer} for any Borel probability measure 
on $K$.
\end{lemma}

\begin{proof}\textbf{of Lemma~\ref{lem:Good_Feature_Map}}\hfill
\begin{enumerate}[label=\emph{Step \arabic*:}, ref=\arabic*, wide]
    \item \label{step:GFM_1}\emph{Estimating $\rdim$ on the Riemannian Manifold $(\mathcal{M},g)$.}
Since $(K,\|\cdot\|_T)$ be isometric to $(\mathcal{M},d_g)$ where $d_g$ is the geodesic distance on an aspherical compact $\ldim$-dimensional Riemannian manifold $(\mathcal{M},g)$, then Gromov's systolic inequality \citep[Theorem 0.1.A]{gromov1983filling} implies that the systole of $\mathcal{M}$ denoted by $\operatorname{sys}(\mathcal{M})$ satisfies
\begin{equation}
\label{eq:Systolic}
        \operatorname{sys}(\mathcal{M}) 
    \le
    \,\tilde{C}_{\ldim}
    \,
        \operatorname{Vol}(\mathcal{M},g)^{1/\operatorname{dim}(\mathcal{M})}
    ,
\end{equation}
where $\operatorname{Vol}(\mathcal{M},g)$ denotes the Riemannian volume of $(\mathcal{M},g)$ and $\tilde{C}_{\ldim}>0$ is a universal constant only depending on $\ldim$ satisfying
$
        0
    <
        \Tilde{C}_{\ldim}
    <
        6 (\ldim+1)\ldim^{\ldim}\,\sqrt{(\ldim+1)!}
\,
$
.
By Stirling's approximation, $\sqrt{(\ldim+1)!} \in \mathcal{O}(\ldim^{(\ldim+1)/2})$; whence it follows that $\tilde{C}_{\ldim}\in 
\mathcal{O}(\ldim^{3(\ldim+1)/2})$.

In the proof of \cite[Theorem 1]{katz2011bi} (circa \cite[Equation (1.1)]{katz2011bi}) we see that if $\delta = \operatorname{Sys}(\mathcal{M})/10$ then given any $\delta$-net $\widetilde{\mathbb{X}}$ in $(\mathcal{M},g)$, the map $\varphi_2:(\mathcal{M},d_g)\mapsto (\mathbb{R}^{\#\widetilde{\mathbb{X}}},\|\cdot\|_2)$ given for any $p\in M$ by
\[
    \varphi_2:p\mapsto \big(d_g(p,u)\big)_{u\in \widetilde{\mathbb{X}}}
\]
is a bi-Lipschitz embedding.  Set $C_{\ldim}\eqdef \tilde{C}_{\ldim}/10$ and $\rdim\eqdef \#\widetilde{\mathbb{X}}$.  Therefore,
\[
        \delta 
    \le
        C_{\ldim}
        \,
        \operatorname{Vol}(\mathcal{M},g)^{1/\ldim}
    .
\]
Enumerate $\widetilde{\mathbb{X}}\eqdef \{u_n\}_{n=1}^\rdim$.
We note that if $K$ is not aspherical but if it is only isometric to a closed Riemannian manifold, then the conclusion still holds, however, without this explicit upper bound on $\delta$.
\item \label{step:GFM_2} \emph{Building the Feature Map with Well-posed Inverse.}
Let $\varphi_1:(\mathcal{M},d_g)\rightarrow (K,\|\cdot\|_T)$ by an any isometry, which we have postulated to exist.  Set $\smash{\mathbb{X}\eqdef \{\varphi(u):\,u\in \widetilde{\mathbb{X}}\}=\{y^{(n)}\}_{n=1}^\rdim}$ where, for $n=1,\dots,\rdim$ we define $y^{(n)}\eqdef \phi_1(u_n)$.  
Define $\smash{\varphi: C([0:T],\mathbb{R}^{d_Y})\rightarrow \mathbb{R}^\rdim}$ by $\smash{y_{\cdot}\mapsto (\|y_{\cdot}-y^{(n)}_{\cdot}\|_{T})_{n=1}^\rdim}$.  Observe that, for every path $y\in K$, the following holds
\allowdisplaybreaks
\begin{align}
\nonumber
        \varphi(y_{\cdot})
    \eqdef 
        \Big(
            \|y_{\cdot}-y^{(n)}_{\cdot}\|_{T}
        \Big)_{n=1}^\rdim
= 
&
        \Big(
            \big\|
                \varphi_1\circ \varphi_1|_{\varphi_1(K)}^{-1}(y_{\cdot})
                -
                \varphi_1\circ \varphi_1|_{\varphi_1(K)}^{-1}(y^{(n)}_{\cdot})
            \big\|_T
        \Big)_{n=1}^\rdim
\\
\label{proofFeatureMap_isometry_assumption}
= &
        \Big(
            d_g\big(
                \varphi_1|_{\varphi_1(K)}^{-1}(y_{\cdot})
                    , 
                \varphi_1|_{\varphi_1(K)}^{-1}(y^{(n)}_{\cdot})
            \big)
        \Big)_{n=1}^\rdim
\\
= &
\label{proofFeatureMap_injectivity}
        \Big(
            d_g\big(
                \varphi_1|_{\varphi_1(K)}^{-1}(y_{\cdot})
                    , 
                u^n
            \big)
        \Big)_{n=1}^\rdim
= 
    \varphi_2\circ \varphi_1|_{\varphi_1(K)}^{-1}(y_{\cdot})
,
\end{align}
where~\eqref{proofFeatureMap_isometry_assumption} holds since $\varphi_2$ is an isometry and~\eqref{proofFeatureMap_injectivity} holds unambiguously since $\varphi_1$ is a bijection from $\mathcal{M}$ to $K$.  Since every isometry is a bi-Lipschitz map, the compositions of bi-Lipschitz maps is again bi-Lipschitz, and since we have just shown that $\varphi|_{K}=\varphi_2\circ \varphi_1|_{\phi(K)}^{-1}$ then, $\phi|_{K}$ is a bi-Lipschitz embedding of $K$ into the $\rdim$-dimensional Euclidean space.

\item \label{step:GFM_3} \emph{Aligning to a Hyperplane in $\mathbb{R}^{\sdim}$ with a Shallow ReLU Neural Network.}
Set $\sdim \eqdef \rdim+1$.  Let $b$ be the zero vector in $\mathbb{R}^{2\rdim}$.
We now consider a variation of the example on~\citep[page 3]{cheridito2021efficient}, the respective $\rdim \times 2\rdim$ and $2\rdim\times \rdim$ block-matrices $A_1$ and $B$
\[
        A_1
    =
        \begin{pmatrix}
            I_\rdim & ~ -I_\rdim
        \end{pmatrix}
    \mbox{ and }
        B
    =
        \begin{pmatrix}
            I_\rdim \\
            -I_\rdim
        \end{pmatrix}
\]
are such that the ReLU neural network $\tilde{\psi}:\mathbb{R}^\rdim\rightarrow \mathbb{R}^\rdim$, with $\tilde{\psi}(u)\eqdef A_1\operatorname{ReLU}\bullet(Bu+b)$ satisfies $\tilde{\psi}(u)=u$, for each $u\in \mathbb{R}^\rdim$.  Consider the $\sdim\times \rdim$ block-matrix $A_2$ and the vector $a\in \mathbb{R}^{\sdim}$ given by
\[
        A_2
    =
        \begin{pmatrix}
            I_\rdim \\
            0 \\
        \end{pmatrix}
    \mbox{ and }
        a_i 
    =
        \begin{cases}
            0 & \mbox{ if } i=1,\dots,\rdim\\
            1 & \mbox{ if } i=\sdim
    .
        \end{cases}
\]
Set $A\eqdef A_2A_1$, $\psi\eqdef A\operatorname{ReLU}\bullet(Bu+b)+a$, and note that $\psi$ is a ReLU neural network.  A direct computation shows that, $\|A\|_0=2\rdim$, $\|B\|_0=2\rdim$, $\|a\|_0=1$, and $\|b\|_0=0$.  
\item \label{step:GFM_4} \emph{Injectivity of Softmax Function on Hyperplane.}
Note that, for each $u\in \mathbb{R}^\rdim$
\[
    \psi(u) = (u_1,\dots,u_N,1)^{\top}
,
\]
therefore bijectively $\psi$ maps $\mathbb{R}^\rdim$ onto the $\rdim$-dimensional hyperplane $H\eqdef \{(u,1):\,u\in \mathbb{R}^\rdim\}$ in $\mathbb{R}^{\sdim}$.  Observe that the Softmax function, given for any $u\in \mathbb{R}^{\sdim}$ by $\smash{\operatorname{Softmax}:u\mapsto (e^{u_j}/\sum_{k=1}^\rdim\,e^{u_k})_{j=1}^\sdim}$, maps $H$ surjectively and continuously onto the image set 
\begin{align*}
    \operatorname{Softmax}(\mathbb{R}^{\sdim})\subseteq \operatorname{int}(\Delta_\sdim)\eqdef \Big\{v\in [0,1]^{\sdim}:\,\sum_{i=1}^{\sdim}\,v_i=1\Big\}; 
\end{align*}
i.e.\ of the interior of the $\sdim$-simplex.  Since, 
$
    R:
        v
    \mapsto
        \big(\ln(v_i)-\ln(v_{\sdim}) + 1\big)_{i=1}^{\sdim}
$ 
is a continuous right-inverse of $\operatorname{Softmax}$ on the interior of the image set $\operatorname{Softmax}(\mathbb{R}^{\sdim})$, the softmax function $\operatorname{Softmax}$ defines a continuous bijection from $H$ onto $\operatorname{Softmax}(\mathbb{R}^{\sdim})$.  
Therefore, the map $\Psi\eqdef \operatorname{Softmax}\circ \psi\circ \phi$ is a continuous and injective when restricted to $K$. 
Since $K$ is compact and $\Psi$ is continuous then $\Psi(K)$ is a compact subset of the interior $\operatorname{int}(\Delta_\sdim)$ of the $\sdim$-simplex, then, 
$
	 C^{\star}\eqdef \max_{u\in \Phi(K)}~\min_{v\in \Delta_{\sdim}\setminus \operatorname{int}(\Delta_{\sdim})}\,\|u-v\|>0.
$
  Therefore, $\Psi(K)$ is contained in the set
$
        \Delta_{\sdim}^{\star}
    \eqdef 
        \{
            t^{\star}(u-\bar{\Delta}_{\sdim})
            +
            \bar{\Delta}_{\sdim}
        \}
$
where $\bar{\Delta}_{\sdim}\eqdef (1/\sdim,\dots,1/\sdim)$ is the barycenter of the $\sdim$-simplex $\Delta_{\sdim}$ for some $t^{\star}\in [0,1)$ such that $\Psi(K)\subseteq \Delta_{\sdim}^{\star}$ (which is possible since $C^{\star}>0$).  Since the map $R$
is locally Lipschitz on $\operatorname{int}(\Delta_\sdim)$, it is Lipschitz on the compact set $\Delta_{\sdim}^{\star}$ then $\Psi$ has a Lipschitz inverse on its image since it is the composition of Lipschitz functions with Lipschitz inverses.  
\item \label{step:GFM_5} \emph{Representation as an Attention Mechanism.}
Set $\theta_0 = (A,B,a,b, \{y^{(n)}_{\cdot}\}_{n=1}^\rdim)$ and observe that the map $\Psi$ is of the form~\eqref{eq:sim_score}.  
We continue by choosing a map of the form \eqref{eq:pos_encoding}. For this, let $\tdim \in \N$, set $\pdim \eqdef \sdim$ and let $U$ be the $\pdim \times \tdim$-dimensional zero matrix, and let $V$ be the $\pdim \times d_Y$-dimensional matrix whose entries are all equal to $1$.  Set $\theta_1\eqdef (U,V,\{n/{\pdim}\}_{n=1}^{\tdim})$.  Define the encoding dimension $\edim \eqdef \sdim$ and let $C$ be the $\edim \times \edim\cdot  d_Y$-dimensional matrix given by $Cx = (x_{n \cdot d_Y})_{n=1}^{\edim}$ for every $x\in \mathbb{R}^{\edim d_Y}$.  Set $\theta\eqdef (\theta_0,\theta_1,C)$ and observe that 
\[
        \operatorname{attn}_T^{\theta}(t,y_{\cdot}) 
    \eqdef
        \Big(
                t
            ,
                C\,
                \operatorname{vec}\big(
                        \operatorname{sim}_T^{\theta_0}(y_{\cdot}) 
                    \odot
                        \operatorname{post}_T^{\theta_1}(y_{\cdot})
                \big)
        \Big)
    = 
        \big(1_{[0:T]}\times \Psi\big)
            (t,y_{\cdot})
\]
for each $y_{\cdot}\in K.$ and every $y\in [0:T]$. 
Consequentially, $\operatorname{attn}_T^{\theta_{\cdot}}$ is a bi-Lipschitz embedding of $[0:T]\times K$ to $\mathbb{R}^{1 + \edim}$.  Note that $\|C\|_0= \pdim$, $\|U\|_0=0$, and $\|V\|_0= \pdim d_Y$.  Thus, $\|\theta\|_0\in \mathcal{O}(\pdim \cdot d_Y + \rdim)$.
\end{enumerate}
\end{proof}

\begin{lemma}[Stable Lossless Feature Maps -- Finite Case]
\label{lem:Good_Feature_Map__FiniteVariant}
Under \Cref{ass:reg_compact} \ref{ass:reg_compact_finite} and let $\rdim \eqdef \#K$.  
Set $\edim=192\lceil\log(\rdim)\rceil$ and $\{y_i\}_{i=1}^k$ enumerate $K$.  
There is a parameter $\theta$ as in \Cref{defn:PathAttention} defining a pathwise attention mechanism $\operatorname{attn}_T^{\theta}:C([0:T],\mathbb{R}^{d_Y})\rightarrow \mathbb{R}^{\edim}$ which restricts to a $(2^{-1/2},(3\rdim/6)^{1/2})$-bi-Lipschitz embedding of $(K,\|\cdot\|_T)$ into $(\mathbb{R}^{1+\edim},\|\cdot\|_2)$.
In particular, $\operatorname{attn}_T$ satisfies \citep[Setting 3.6 (i)]{kratsios2023transfer}. 
\end{lemma}
\begin{proof}\textbf{of Lemma~\ref{lem:Good_Feature_Map__FiniteVariant}}
We argue similarly to the proof of \Cref{lem:Good_Feature_Map}, with only Steps \ref{step:GFM_1} and \ref{step:GFM_2} of its proof being replaced with the following argument.  
\begin{enumerate}[label=\emph{Step \arabic*:}, ref=\arabic*, wide]
\item[\emph{Steps 1-2}] \textit{(Modified): Building the Feature Map.}\hfill\\
Enumerate $\smash{K=\{x_n\}_{n=1}^\rdim}$.  The map $\smash{\phi_1:(K,\|\cdot\|_T)\rightarrow (\mathbb{R}^\rdim,\|\cdot\|_{\infty})}$ given for every $x\in K$ by $\smash{x\mapsto \big(\|x-x_n\|_t\big)_{n=1}^\rdim}$ is an isometric embedding, called the \textit{Kuratowski embedding} (see \cite[page 99]{Heinonen_2001_BookAnalMetSpace}).  The optimal constants for the equivalence of the Euclidean $\|\cdot\|_2$ and max $\|\cdot\|_{\infty}$ norms are given by
\begin{equation}
\label{eq:equivalence_2infty_norms}
    \|u\|_{\infty} \le \|u\|_2
    \quad\mbox{ and }\quad
    \|u\|_2\le \rdim^{1/2}\,\|u\|_{\infty}
\end{equation}
for each $u\in\mathbb{R}^\rdim$.  Therefore, the ``set theoretic identity map'' $\phi_2:(\mathbb{R}^\rdim,\|\cdot\|_{\infty})\rightarrow (\mathbb{R}^\rdim,\|\cdot\|_2)$ is bi-Lipschitz with optimal (shrinking and expansion) constants given by~\eqref{eq:equivalence_2infty_norms}.  
Define $\sdim \eqdef \lceil48\,\ln(\rdim)\rceil$.
By the Johnson-Lindenstrauss Lemma, with ``small'' constant given in the derivation of \citep[Theorem 2.1]{DubhashiPanconesi_2009_Concentration}
\footnote{We use this formulation since the constant in $4*48$ as opposed to $4*200$ in the standard derivation; found for example in \cite{JohnsonLindenstrauss_1984_originalpaper}.}
, exists a linear map $\phi_3:(\mathbb{R}^\rdim,\|\cdot\|_2)\rightarrow (\mathbb{R}^{\sdim},\|\cdot\|_2)$, i.e.\ $\varphi_3(u)=A_0\,u$ for some $\sdim \times \rdim$-matrix $A_0$, satisfying: For every $u,v\in \mathbb{R}^\rdim$
\[
    2^{-1/2}
    \,
    \|u-v\|_2
    \le 
        \|A_0\,u-A_0\,v\|_2
    \le 
    (3/2)^{1/2}
    \,
    \|u-v\|_2
    .
\]
Consequentially, the map $\varphi:(K,\|\cdot\|_T)\rightarrow (\mathbb{R}^{\sdim},\|\cdot\|_2)$ given by $\varphi\eqdef \varphi_3\circ \varphi_2\circ \varphi_1$ is $(2^{-1/2},(3\rdim/2)^{1/2})$-bi-Lipschitz, since the Kuratowksi embedding $\varphi_1$ is an isometry and $\varphi_2$ satisfies~\eqref{eq:equivalence_2infty_norms}; this is because for each $x,\tilde{x}\in K$ we have
\allowdisplaybreaks
\begin{align*}
    2^{-1/2}
    \,
        \|x-\tilde{x}\|_T
    = &
        2^{-1/2}
    \,
        \|\varphi_1(x)-\varphi_1(\tilde{x})\|_{\infty}
\\
    \le & 
        2^{-1/2}
    \,
        \|\varphi_2\circ \varphi_1(x)-\varphi_2\circ \varphi_1(\tilde{x})\|_2
\\
    \le & 
        \|\varphi_3\circ \varphi_2\circ \varphi_1(x)-\varphi_3\circ \varphi_2\circ \varphi_1(\tilde{x})\|_2
\\
    \eqdef &
        \|\varphi(x)-\varphi(\tilde{x})\|_2  
\\
    \eqdef &
        \|\varphi_3\circ \varphi_2\circ \varphi_1(x)-\varphi_3\circ \varphi_2\circ \varphi_1(\tilde{x})\|_2
\\
    \le &
        (3/2)^{1/2}
        \,
        \|\varphi_2\circ \varphi_1(x)-\varphi_2\circ \varphi_1(\tilde{x})\|_2
\\
    \le & 
        (3/2)^{1/2}
        \,
        \rdim^{1/2}\,
        \| \varphi_1(x)-\varphi_1(\tilde{x}) \|_{\infty} 
    = 
        (3\rdim/2)^{1/2}
        \,
        \| x-\tilde{x} \|_T 
.
\end{align*}
\item[\emph{Remaining Steps.}]
The rest of the proof is identical to Steps \labelcref{step:GFM_3,step:GFM_4,step:GFM_5}\footnote{Here $\sdim$ corresponds to $\sdim - 1$ in Step \ref{step:GFM_3}.} of the proof of \Cref{lem:Good_Feature_Map} but with $B$ defined instead as $B\eqdef B_1A_0$, where $B_1=  (I_\sdim\, \,I_\sdim)^{\top}$.
\end{enumerate}
This concludes the proof.
\end{proof}

\begin{lemma}[Stable Lossless Feature Maps -- Linear Case]
\label{lem:biLip_pw_lin}
Fix $\pldim \in \mathbb{N}_+$, $0=t_0<\dots<t_\pldim=T$, and a constant $C_K>0$.
Let $K\subset C([0:T],\mathbb{R}^{d_Y})$ satisfy \Cref{ass:reg_compact} \ref{ass:reg_compact_linear}. 
Then, there is a parameter $\theta$, as in \Cref{defn:PathAttention}, defining a pathwise attention mechanism $\operatorname{attn}_T^{\theta}:C([0:T],\mathbb{R}^{d_Y})\rightarrow \mathbb{R}^{\edim + 1}$, with $\edim \eqdef \pldim \cdot d_Y$ which restricts to a bi-Lipschitz embedding of $(K,\|\cdot\|_T)$ into $(\mathbb{R}^{\edim + 1},\|\cdot\|_2)$.
Furthermore, $\|\theta\|_0\in \mathcal{O}(\pldim \cdot d_Y)$.

In particular, $\operatorname{attn}_T$ satisfies \citep[Setting 3.6 (i)]{kratsios2023transfer}.
\end{lemma}
\begin{proof}\textbf{of Lemma~\ref{lem:biLip_pw_lin}}\hfill\\
Let $\rdim \in \N_+$, set $\sdim \eqdef \pldim$, and define $A, B$ as zero-matrices as well as $a, b$ as zero-vectors according to the dimensions in \Cref{defn:Sim_Score}.
By fixing $\theta_0=(A,B,a,b)$, we observe that: for all $y\in C([0:T],\mathbb{R}^{d_Y})$
\begin{align}
\label{eq:bijection_samplingA}
    \operatorname{sim}_T^{\theta_0}(y_{\cdot}) = 1 \in \R^{\pldim}
.
\end{align}
Set $\tdim \eqdef \pdim \eqdef \pldim$. Then, let $V$ be the zero matrix and $U$ the identity matrix according to the dimensions in \Cref{defn:PositionalEncoding}\footnote{The point $t_0$ is not sampled since every path $y_{\cdot}$ in $K$ satisfies $y_0=0$ and, thus, there is no need to sample it at time $0$.}. Set $\theta_1=(U,V,\{t_n\}_{n=1}^\pldim)$ and with that, 
\begin{equation}
\label{eq:bijection_samplingB}
        \operatorname{pos}_T^{\theta_1}(y_{\cdot})
    =
        (y_1,\dots,y_\pldim).
\end{equation}
In particular, the map $\operatorname{pos}_T^{\theta_1}$ in \eqref{eq:bijection_samplingB} defines a linear bijection between $K$ and the set $B_{C_K}(0)^\pldim$, where $B_{C_K}(0)  \eqdef  \lbrace z \in \R^{d_Y} : \vert z \vert \leq C_K\rbrace$.  Since $B_{C_K}(0)^\pldim$ is a subset of the finite-dimensional Banach space $(\mathbb{R}^{\pldim \times d_Y},\|\cdot\|_2)$ and all norms on a finite-dimensional normed are equivalent, then, $\operatorname{pos}_T^{\theta_1}$ is a bi-Lipschitz embedding of $K$ into $(\mathbb{R}^{\pldim \times d_Y},\|\cdot\|_2)$.  

Fix $\edim \eqdef \pldim \cdot d_Y$ and let $C$ be the $\edim \times \edim$-dimensional matrix with $C_{i,j}=1$ for $i,j=1,\dots,\edim$.  Set $\theta=(\theta_0,\theta_1,C)$.  By construction $\|\theta\|_0\in \mathcal{O}(\pldim \cdot d_Y)$.  
Then, 
\begin{equation}
\label{eq:bijection_sampling__RelabelingStuff}
        \operatorname{attn}_T^{\theta}(t,y_{\cdot})
    \eqdef 
        \Big(
                t
            ,
                C\,
                \operatorname{vec}\big(
                        \operatorname{sim}_T^{\theta_0}(y_{\cdot}) 
                    \odot
                        \operatorname{post}_T^{\theta_1}(y_{\cdot})
                \big)
        \Big)
    =
        \big(t,(y_1,\dots,y_\pldim)\big)
    .
\end{equation}
Together, \labelcref{eq:bijection_samplingA,eq:bijection_samplingB} imply that $\operatorname{attn}_T^{\theta}$ defines a bi-Lipschitz embedding of $K$ into the $\edim+1$-dimensional Euclidean space.
\end{proof}

\begin{proof}\textbf{of \Cref{prop:Lossless_Encoding}}\hfill\\
Follows directly from \Cref{lem:Good_Feature_Map,lem:Good_Feature_Map__FiniteVariant,lem:biLip_pw_lin}.
\end{proof}

\subsection{Approximability of Locally Lipschitz Maps By FFs - Proof of \texorpdfstring{\Cref{prop:Universal_Approximation_Theorem__PathToNd}}{Proposition 4}}
We now prove our main universal approximation theorem, \Cref{prop:Universal_Approximation_Theorem__PathToNd}.  We show that the target space/codomain of any considered FF is geometrically regular, in the sense of \cite{Acciaio2022_GHT,kratsios2023transfer}.  Using this fact, we then combine and apply the results of \cite{kratsios2021universal,kratsios2023transfer} to deduce the result.

\subsubsection{Polish QAS Space Structure on the space of Positive Semi-Definite Matrices}

\label{s:Proof__ss:Universality___sss:Outputs}
Let $d_X \in \mathbb{N}_+$, and $\operatorname{Sym}_{0,d_X}$ denote the set of $d_X\times d_X$ symmetric positive semi-definite matrices and let $\|\cdot\|_F$ denote the Frobenius norm on the set of $d_X\times d_X$ matrices.\footnote{The $0$ emphasizes that there is no rank-restriction on the matrices in $\operatorname{Sym}_{0,d_X}$ unlike, for example, in \cite{herrera2023Denise_TMLR,neuman2023restricted}.}  Consider the $2$-product metric on $\mathbb{R}^{d_X}\times \operatorname{Sym}_{0,d_X}$ given for any $(m^{(1)},A),(m^{(2)},B)$ by 
$
        \operatorname{d}_{2,F}\big(
            (m^{(1)},A),(m^{(2)},B)
        \big)^2
    \eqdef 
        \|m^{(1)}-m^{(2)}\|^2+\|A-B\|_F^2
$.
We show that $(\mathbb{R}^{d_X}\times \operatorname{Sym}_{0,d_X},\operatorname{d}_{2,F})$ satisfies the conditions of \citep[Theorem 3.7]{kratsios2023transfer}; namely \citep[Setting 3.6 (iii)]{kratsios2023transfer}.
This requires defining a few maps first.  
For $q\in \mathbb{N}$ define the map $\mathcal{Q}_q:\mathbb{R}^{d_X+d_X^2}\rightarrow \mathbb{R}^{d_X}\times \operatorname{Sym}_{0,d_X}$ by sending any $(m,A)\in \mathbb{R}^{d_X}\times \operatorname{Sym}_{0,d_X}$ to
$
        \mathcal{Q}_q((m,A)) 
    \eqdef 
        (m,A^{\top}A)
$. 
The family $\mathcal{Q}\eqdef (\mathcal{Q}_q)_{q\in \mathbb{N}_+}$ quantizes $(\mathbb{R}^{d_X}\times \operatorname{Sym}_{0,d_X},\operatorname{d}_{2,F})$, in the sense of \citep[Definition 3.2]{Acciaio2022_GHT}.

Next, we consider the so-called \textit{mixing function }$\eta:\cup_{N\in \mathbb{N}_+}\,\Delta_N\times (\mathbb{R}^{d_X}\times \operatorname{Sym}_{0,d_X})^{N} \rightarrow 
\mathbb{R}^{d_X}\times \operatorname{Sym}_{0,d_X}$, where $N\in\mathbb{N}_+$ and $\Delta_N$ is the $N$-simplex, defined for any $N\in\mathbb{N}_+$, $w\in \Delta_N$, and $(m^{(1)},A^{(1)}),\dots,(m^{(N)},A^{(N)})\in \mathbb{R}^{d_X}\times \mathbb{R}^{d_X^2}$ by $
    \eta\big(
            w
        ,
            \{(m^{(n)},A^{(n)}\}_{n=1}^N
    \big)
        \eqdef 
            \sum_{n=1}^N\,
                w_n
                \cdot
                (m^{(n)},A^{(n)})
$.  
Note that by the convexity of $\mathbb{R}^{d_X}\times \operatorname{Sym}_{0,d_X}$ and the fact that each $w\in \Delta_N$, for some $N\in \mathbb{N}_+$, then $\eta$ does indeed take values in $\mathbb{R}^{d_X}\times \operatorname{Sym}_{0,d_X}$.  The mixing function $\eta$ will be used to inscribe ``abstract geodesic simplices'' in $(\mathbb{R}^{d_X}\times \operatorname{Sym}_{0,d_X},\operatorname{d}_{2,F})$ thereby endowing it with the structure of an approximately simplicial space, in the sense of \citep[Definition 3.1]{Acciaio2022_GHT}.

\begin{lemma}
\label{lem:QAS_Space_Sym0d}
$(\mathbb{R}^{d_X}\times \operatorname{Sym}_{0,d_X},\mathcal{Q},\eta)$ is a barycentric QAS space, in the sense of \citep[Definition 3.4]{Acciaio2022_GHT}.  In particular, it satisfies \citep[Setting 3.6 (iii)]{kratsios2023transfer}.
\end{lemma}
\begin{proof}\textbf{of Lemma~\ref{lem:QAS_Space_Sym0d}}\hfill\\
We first show that $(\mathbb{R}^{d_X}\times \operatorname{Sym}_{0,d_X},\mathcal{Q},\eta)$ is a QAS space, as defined in \citep[Definition 3.4]{Acciaio2022_GHT}.  We also show that it is barycentric, meaning that it admits a $1$-barycenter map as defined, for example, in \citep[Section 3.2]{basso2018fixed}.
\hfill\\
\textit{QAS Space Structure}
\hfill\\
Since $\operatorname{Sym}_{0,d_X}$ is a closed convex subset of space $(\mathbb{R}^{d\times d},\|\cdot\|_F)$.  Since the Cartesian product of closed subsets is a closed subset of $(\mathbb{R}^{d^2+d},d_{2,F})$ and since the product of convex sets is again convex by \citep[Proposition 3.6]{CombettesBauschke_Book_ConvexAnalysisMonotonOperators_2017} then, $\mathbb{R}^{d_X}\times \operatorname{Sym}_{0,d_X}$ is a closed and convex subset of the normed linear space $(\mathbb{R}^{d+d^2},d_{2,F})$.  
Now, for $q\in \mathbb{N}$ define the map $\mathcal{Q}_q:\mathbb{R}^{d+d^2}\rightarrow \mathbb{R}^{d_X}\times \operatorname{Sym}_{0,d_X}$ as sending any $(m,A)\in \mathbb{R}^{d_X}\times \operatorname{Sym}_{0,d_X}$ to
$
        \mathcal{Q}_q((m,A)) 
    \eqdef 
        (m,A^{\top}A)
    .
$
The map $\mathcal{Q}_q$ is a surjection since every symmetric matrix is the square of some $d\times d$ matrix \cite{TBD}.  
Therefore, $\mathcal{Q}_{\cdot}=(\mathcal{Q}_q)_{q\in\mathbb{N}_+}$ trivially satisfies \citep[Definition 3.2]{Acciaio2022_GHT} with modulus of quantizability, defined on \citep[page 12]{Acciaio2022_GHT}, given by $\mathcal{Q}_{K}(\varepsilon)=D_1\eqdef d(d+1)$ for every compact subset $K$ of $\mathbb{R}^{d_X}\times \operatorname{Sym}_{0,d_X}$ and for every $\varepsilon>0$.  
For any $N\in \mathbb{N}_+$, $w\in \Delta_N$, and $(m^{(1)},A^{(1)}),\dots,(m^{(N)},A^{(N)})\in \mathbb{R}^{d_X}\times \mathbb{R}^{d\times d}$ we have
\allowdisplaybreaks
\begin{align}
\nonumber
        & d_{2,F}\Big(
            \eta\big(
                w
            ,
                \{(m^{(n)},A^{(n)}\}_{n=1}^N
            \big)
            ,
                y_i
        \Big) \\
    \notag
    & \quad = 
    \Big(
            \big\|
                \big(\sum_{n=1}^N\,w_n\,m^{(n)}\big) - m^{(i)}
            \big\|^2
        +
            \big\|
                \big(\sum_{n=1}^N\,w_n\,A^{(n)}\big) - A^{(i)}
            \big\|_F^2
    \big)^{1/2}
    \\
    \label{eq:Because_Simplex}
    & \quad = 
    \Big(
            \big\|
                \sum_{n=1}^N\,w_n\,m^{(n)} 
                    - 
                \sum_{n=1}^N\,w_n\,m^{(i)}
            \big\|^2
        +
            \big\|
                \sum_{n=1}^N\,w_n\,A^{(n)}
                - 
                \sum_{n=1}^N\,w_n\,A^{(i)}
            \big\|_F^2
    \big)^{1/2}
    \\
\nonumber
    & \quad \le
        \Big(\sum_{n=1}^N\,w_n\Big)^2\,
            \big\|
                m^{(n)} 
                    - 
                m^{(i)}
            \big\|^2
        +
        \Big(\sum_{n=1}^N\,w_n\Big)^2\,
            \big\|
                A^{(n)}) 
                - 
                A^{(i)}) 
            \big\|_F^2
    \big)^{1/2}
\\
\nonumber
    & \quad \le
    \sum_{n=1}^N\,w_n
    \,
    \Big(
            \big\|
                m^{(n)} 
                    - 
                m^{(i)}
            \big\|^2
        +
            \big\|
                A^{(n)}) 
                - 
                A^{(i)}) 
            \big\|_F^2
    \big)^{1/2}
\\
\nonumber
    & \quad = 
    1\cdot
    \sum_{n=1}^N\,w_n
    \,
        d_{2,F}\big(
            (m^{(n)},A^{(n)})
            ,
            (m^{(i)},A^{(i)})
            \big)^1
,
\end{align}
for $i=1,\dots,N$,
where~\eqref{eq:Because_Simplex} holds since $\sum_{n=1}^N\,w_n=1$ since $w\in \Delta_N$.  Thus, $\eta$ is a mixing function and therefore $(\mathbb{R}^{d_X}\times \operatorname{Sym}_{0,d_X},d_{2,F})$ is approximately simplicial, as defined in \citep[Definition 3.1]{Acciaio2022_GHT}.
Consequentially, $\smash{(\mathbb{R}^{d_X}\times \operatorname{Sym}_{0,d_X},d_{2,F},\mathcal{Q},\eta)}$ is a QAS space; as defined in \citep[Definition 3.4]{Acciaio2022_GHT}.  
\hfill\\
\textit{Barycentricity}
\hfill\\
Since $(\mathbb{R}^{d_X}\times \operatorname{Sym}_d,d_{2,F})$ is a normed linear space then \cite{BruHeinicheLootgieter1993} shows that the only contracting barycenter map is given by $\mathcal{P}_1(\mathbb{R}^{d_X}\times \operatorname{Sym}_{d},d_{2,F})\ni \mathbb{P}\rightarrow \mathbb{E}_{X\sim \mathbb{P}}[X]\in \mathbb{R}^{d_X}\times \operatorname{Sym}_{d}$. 
Since $\operatorname{Sym}_{0,d_X}$ is a closed convex subset of the normed linear space $\operatorname{Sym}_d$ then Jensen's inequality, as formulated in \citep[Theorem 10.2.6]{DudleyRealAnalysisandPRobabilityBook1989Revised2002}, implies that for each $\mathbb{P}\in \mathcal{P}_1(\mathbb{R}^{d_X}\times \operatorname{Sym}_{0,d_X},d_{2,F})$ we have $\mathbb{E}_{X\sim \mathbb{P}}[X]\in \mathbb{R}^{d_X}\times \operatorname{Sym}_{0,d_X}$; where $\mathbb{E}_{X\sim \mathbb{P}}[X]$ denote the Bochner integral of a random variable with law $\mathbb{P}$.  Consequentially, 
$\mathcal{P}_1(\mathbb{R}^{d_X}\times \operatorname{Sym}_{0,d_X},d_{2,F})\ni \mathbb{P}\rightarrow \mathbb{E}_{X\sim \mathbb{P}}[X]\in \mathbb{R}^{d_X}\times \operatorname{Sym}_{0,d_X}$ is a contracting barycenter map.  Thus, $(\mathbb{R}^{d_X}\times \operatorname{Sym}_{0,d_X},d_{2,F})$ is barycentric metric space.
Since this is a barycentric QAS space then, \citep[Setting 3.6 (iii)]{kratsios2023transfer} is satisfied.
\end{proof}
\begin{lemma}
\label{lem:approximation_Sym}
Fix an activation function $\sigma$ satisfying \Cref{ass:KLCondition}.
For every $K\subset C([0:T],\mathbb{R}^{d_X})$ satisfying \Cref{ass:reg_compact}, each $0<\delta,\alpha\le 1$, $0 \le L$, and every $(L,\alpha)$-H\"{o}lder
\footnote{That is, $f$ is $\alpha$ H\"{o}lder with optimal H\"{o}lder coefficient $L$.} 
function $f:[0:T]\times K \rightarrow \mathbb{R}^{d_X}\times \operatorname{Sym}_{0,d_X}$ there is a map $\hat{g}: [0:T]\times C([0:T],\mathbb{R}^{d_Y}) \rightarrow \mathbb{R}^{d_X}\times \operatorname{Sym}_{0,d_X}$ satisfying the uniform estimate
\begin{equation}
\label{lem:approximation_Sym__approximation}
    \max_{(t,y_{\cdot}) \in [0:T]\times K}\,
        d_{2,F}(\hat{g}(t,y_{\cdot}),f(t, y_{\cdot})) < \delta
\end{equation}
with representation 
$
        \hat{g}(t, y_{\cdot}) 
    =
        \sum_{i=1}^\ddim\,
                P_{\Delta_\ddim}(\hat{f}\circ \operatorname{attn}^\theta(t, y_{\cdot}))_i
                \cdot
                (m^{(i)},(A^{(i)})^{\top}A^{(i)})
,
$
where $\ddim \in \mathbb{N}_+$, $m^{(1)},\dots,m^{(\ddim)}\in \mathbb{R}^{d_X}$, $A^{(1)},\dots,A^{(\ddim)}\in \mathbb{R}^{d\times d}$, $P_{\Delta_\ddim}:\mathbb{R}^\ddim\rightarrow\Delta_\ddim$ is the Euclidean (orthogonal) projection onto the $\ddim$-simplex, and an MLP $\hat{f}:\mathbb{R}^\edim\rightarrow \mathbb{R}^\ddim$ with activation function $\sigma$.  The depth, width, encoding dimension ($\edim$), and decoding dimension ($\ddim$) are recorded in \Cref{tab:approximationrates}.
\end{lemma}

\begin{table}[!ht]
    \centering
    \scriptsize
    \caption{{Complexity Estimates for transformer-type model $\hat{g}$ in \Cref{lem:approximation_Sym}.}}
    \label{tab:approximationrates}
    \begin{tabular}{lllll}
    \toprule
        $\sigma$ Regularity & Depth & Width & Encode ($\edim$) & Decode ($\ddim$)\\
    \midrule
        $\operatorname{ReLU}$ & 
        $
            \mathcal{O}\big(
                (LV( L))^{-\edim-1}\,\varepsilon^{-\edim-1}
            \big)
        $
        & 
        $
            \mathcal{O}\big(
                (LV(L))^{-\edim-1}\,\varepsilon^{-\edim-1}
            \big)
        $
        & $\mathcal{O}(1)$ & $\mathcal{O}\big(L\varepsilon^{-1}\big)$
        \\
        Smooth $\&$ Non-poly.
         & 
         $
        \mathcal{O}\left(
            L^{4\edim+5}\, \varepsilon^{-4\edim-5}
        \right)
         $
         &
         $ \mathcal{O}(L\epsilon^{-1} + \edim +3)$
         & $\mathcal{O}(1)$ & $\mathcal{O}\big(L\varepsilon^{-1}\big)$
         \\
         Poly. $\&$ Non-affine
         & 
         $
        \mathcal{O}\big(
            L^{8\edim+14}\, \varepsilon^{-8m-14} 
        \big)
         $
         & 
         $\mathcal{O}(L\epsilon^{-1} + \edim +4)$
         & $\mathcal{O}(1)$ & $\mathcal{O}\big(L\varepsilon^{-1}\big)$
         \\
         $C(\rr)$ $\&$ Non-poly.
            &
        Finite
        &
        $\mathcal{O}(L\epsilon^{-1} + \edim +3)$
         & $\mathcal{O}(1)$ & $\mathcal{O}\big(L\varepsilon^{-1}\big)$
         \\
    \bottomrule
    \end{tabular}
    \caption*{Where $V(t)$ is the inverse of $s\mapsto s^4\,\log_3(t+2)$ on $[0,\infty)$ evaluated at $131 t$.}
\end{table}

\begin{proof}\textbf{of Lemma~\ref{lem:approximation_Sym}}\hfill\\
We work in the notation of \citep[Theorem 3.7]{kratsios2023transfer}, or rather, its quantitative version \citep[Lemma 5.10]{kratsios2023transfer}.  
Our objective is to apply \citep[Theorem 3.7]{kratsios2023transfer} by verifying each of the conditions of \citep[Setting 3.6]{kratsios2023transfer}.

\begin{enumerate}[label=\emph{Step \arabic*:}, wide]
    \item \emph{Implementing a Bi-Lipschitz Feature Map with the Attention Layer.}
        \begin{remark}
            We first show that the conditions of {\citep[Setting 3.6 (i)]{kratsios2023transfer}} are met, by verifying that the parameters of the attention layer~\eqref{eq:featuremap}  can be chosen such that $\operatorname{attn}^\theta$ is a suitable feature map.  
        \end{remark}

        When convenient, let $\operatorname{attn}^{\theta}$ be as in either of \Cref{lem:Good_Feature_Map,lem:Good_Feature_Map__FiniteVariant,lem:biLip_pw_lin} depending on which assumption of Assumptions~\ref{ass:reg_compact} (i), (ii), or (iii) holds.  For convenience, we denote the map $\operatorname{attn}$ by $\varphi$.  These lemmata show that the map $\varphi$ is a bi-Lipschitz embedding of $([0:T]\times K,\|\cdot\|\times \|\cdot\|_T)$ into a Euclidean space $(\mathbb{R}^N,\|\cdot\|_2)$.
    
        We also observe that every bi-Lipschitz map is a quasi-symmetric map\footnote{See \citep[page 78]{Heinonen_2001_BookAnalMetSpace}.}, as defined on \citep[page 78]{Heinonen_2001_BookAnalMetSpace}.  Thus, \citep[Theorem 12.1]{Heinonen_2001_BookAnalMetSpace} implies that $(\mathcal{M},d_g)$ is a doubling metric space, as defined on \citep[page 81]{Heinonen_2001_BookAnalMetSpace}.  Since $(\mathcal{M},d_g)$ and $(K,\|\cdot\|_T)$ since both are isometric then $(K,\|\cdot\|_T)$ is also a doubling metric space (see \citep[Lemma 9.6 (v)]{RobinsonDimEmbAt_2011Book}).  Thus, $([0:T]\times K,\|\cdot\|\times \|\cdot\|_T)$ is a doubling metric space.  We have thus verified \citep[Setting 3.6 (i)]{kratsios2023transfer}.
    \item \emph{Feature Space Geometry.}
        Since the codomain of $\phi$ is simply a Euclidean space, then, the constant sequence of identity maps $\{T^i\eqdef 1_{\mathbb{R}^{\edim+1}}\}_{i=1}^{\infty}$ are trivially finite-rank linear operators realizing the bounded approximation property on any compact subset of $\mathbb{R}^{\edim+1}$.  That is, for each non-empty compact $A\subseteq\mathbb{R}^{\edim+1}$, 
        \[
            \lim\limits_{i\mapsto \infty}\,\max_{u\in A}\,\|u-T^i(u)\|_2 = \lim\limits_{i\mapsto \infty}\,\max_{u\in A}\,\|u-1_{\mathbb{R}^N}(u)\|_2 = 0
        \]
        and the operator norm $\|1_{\mathbb{R}^{\edim+1}}\|_{\mathrm{op}}=1$. Thus, $(T^i)_{i\in \mathbb{N}}$ implements the $1$-BAP ($1$-bounded approximation property) on $\mathbb{R}^{\edim+1}$ for every $i$.  Therefore, \citep[Setting 3.6 (ii)]{kratsios2023transfer} holds.

\item \emph{Geometry of $(\mathbb{R}^{d_X}\times \operatorname{Sym}_{0,d_X},\mathcal{Q},\eta)$.}
\Cref{lem:QAS_Space_Sym0d} shows that $(\mathbb{R}^{d_X}\times \operatorname{Sym}_{0,d_X},\mathcal{Q},\eta)$ is barycentric and it is a QAS space with quantized mixing function, see \citep[page 7]{kratsios2023transfer}, given for any $\ddim \in\mathbb{N}_+$, $u\in \mathbb{R}^\ddim$, and $(m^{(1)},A^{(1)}),\dots,(m^{(N)},A^{(\ddim)})\in \mathbb{R}^{d_X}\times \mathbb{R}^{d_X^2}$ by
\begin{equation}
\label{eq:QMF}
    \hat{\eta}\big(
            w
        ,
            \{(m^{(n)},A^{(n)}\}_{n=1}^\ddim
    \big)
        \eqdef 
            \sum_{n=1}^\ddim\,
                P_{\Delta_\ddim}(w_n)
                \cdot
                (m^{(n)},(A^{(n)})^{\top}A^{(n)})
.
\end{equation}
This verifies \citep[Setting 3.6 (iii)]{kratsios2023transfer}.

\item \emph{Determining The Euclidean Universal Approximator.}
\begin{remark}
    We now verify that the class of all MLPs with activation function $\sigma$ satisfying \Cref{ass:KLCondition}, thus assumption \citep[Setting 3.6 (iv)]{kratsios2023transfer} holds.  The case where $\sigma=\operatorname{ReLU}$ and $\sigma\neq \operatorname{ReLU}$ are treated separately.
\end{remark}
First, consider the case where $\sigma\neq \operatorname{ReLU}$.
Then, for each $\edim,\ddim,c\in \mathbb{N}_+$ let $\mathcal{F}_{\edim, \ddim, c}$ denote the family of maps $f:\mathbb{R}^\edim\rightarrow \mathbb{R}^\ddim$ with representation~\eqref{eq:MLP} and satisfying
\begin{equation}
\label{eq:GeneralSigma_Version}
        J \le c
    \mbox{ and }
        \max_{j=0,\dots,J}\,
            d_{j} 
        \le \edim+\ddim+3
.
\end{equation}
Since $\sigma$ was assumed to satisfy \Cref{ass:KLCondition} then by \citep[Theorem 9]{kratsios2022universal}, as formulated in \citep[Proposition 53]{kratsios2022universal}, implies that $\mathcal{F}_{\cdot}$ is a universal approximator in the sense of \citep[Definition 2.11]{kratsios2023transfer}.  Moreover, its rate function is recorded in \citep[Proposition 53]{kratsios2022universal}.  Therefore, $\mathcal{F}_{\cdot}$, as defined in~\eqref{eq:GeneralSigma_Version}, verifies \citep[Setting 3.6 (iv)]{kratsios2023transfer}.

Next, suppose that $\sigma=\operatorname{ReLU}$.  Then, for each $\edim,\ddim,c\in \mathbb{N}_+$ let $\mathcal{F}_{\edim,\ddim,c}$ denote the family of maps $f:\mathbb{R}^\edim\rightarrow \mathbb{R}^\ddim$ with representation~\eqref{eq:MLP} and satisfying
\begin{equation}
\label{eq:ReLU_Version}
        J \le c
    \mbox{ and }
        \max_{j=0,\dots,J}\,
            d_{j} 
        \le c
.
\end{equation}
By \citep[Theorem 1.1]{pmlr-v125-kidger20a}, $\mathcal{F}_{\cdot}$ is a universal approximator in the sense of \citep[Definition 2.11]{kratsios2023transfer}.  Moreover, its rate function is given in \citep[Theorem 1]{galimberti2022designing}, as recorded in \citep[Table 1]{galimberti2022designing}.  In either case, $\mathcal{F}_{\cdot}$, as defined in~\eqref{eq:ReLU_Version}, verifies \citep[Setting 3.6 (iv)]{kratsios2023transfer}.

\item \emph{Applying \citep[Theorem 3.7]{kratsios2023transfer}.}
Steps $1$ though $5$ verify that the conditions of \citep[Theorem 3.7]{kratsios2023transfer} are indeed met.  Furthermore, we have just shown that we are in the special case where the feature decomposition $\{([0,T]\times K,\varphi)\}$ of $([0,T]\times K,\|\cdot\|_T)$, as defined in \citep[Definition 3.4]{kratsios2023transfer}, is a singleton.  
By~\citep[Lemma 5.10]{kratsios2023transfer} for every $\varepsilon>0$ there is a map $\hat{F}:[0,T]\times K\rightarrow \mathbb{R}^{d_X}\times \operatorname{Sym}_{0,d_X}$ with representation
\footnote{Since the barycentric decomposition $\{(K,\varphi)\}$ of $(K,\|\cdot\|_T,\mu)$ has exactly one part then the partition of unity \citep[Setting 3.7]{kratsios2023transfer} is trivial and $\psi_1(x_{\cdot})=1$ for each $x_{\cdot}\in K$.}
\begin{align}
\label{eq:transformer_pre_mapped_BEGIN}
        \hat{F}(y_{\cdot})
    = &
        \beta_{\mathbb{R}^{d_X}\times \operatorname{Sym}_{0,d_X}}
        \big(
            \delta_{\hat{\eta}(P_{\Delta_\ddim}\circ \hat{f}_n\circ \phi(y_{\cdot}),(Z_n)_{n=1}^\edim)}
        \big)
    \\
\label{eq:transformer_barycentercancellation}
    = & 
        \hat{\eta}(P_{\Delta_\ddim}\circ \hat{f}_n\circ \phi(y_{\cdot}),(Z_n)_{n=1}^\edim)
    \\
\label{eq:transformer_pre_mapped_END}
        = & 
            \sum_{n=1}^\ddim
            \,
                \big(
                    P_{\Delta_\ddim}\circ \hat{f}_n\circ \phi(y_{\cdot})
                \big)
                \cdot
                (m^{(n)},(A^{(n)})^{\top}A^{(n)})
\end{align}
for each $y_{\cdot}\in K$ satisfying 
\[
    \sup_{y_{\cdot}\in K}\,
        d_{2,F}\big(
            \hat{F}(y_{\cdot})
        ,
            f(y_{\cdot})
        \big)
    <
        \varepsilon_A + \varepsilon_Q + \varepsilon_E = \varepsilon
,
\]
where $\varepsilon_E=0$ and $\varepsilon_A\eqdef \varepsilon_Q\eqdef \varepsilon/2$ and where $\hat{f}\in \mathcal{F}_{d_n,N_n,c_n}$, $Z_n\in \mathbb{R}^{N_n\times D_n}$, for some positive integers $d_n,c_n,D_n$, and $N_1,\dots,N_N\in \mathbb{N}_+$ recorded in \citep[Table 3]{kratsios2023transfer}, and $\smash{\beta_{\mathbb{R}^{d_X}\times \operatorname{Sym}_{0,d_X}}}$ is a $1$-Lipschitz barycenter map on $(\mathbb{R}^{d_X}\times \operatorname{Sym}_{0,d_X},d_{2,F})$ (which exists by \Cref{lem:QAS_Space_Sym0d}).  We observe that that~\eqref{eq:transformer_barycentercancellation} follows from the fact that the barycenter map $\beta_{\mathbb{R}^{d_X}\times \operatorname{Sym}_{0,d_X}}$ is a right-inverse of the map $\mathbb{R}^{d_X}\times \operatorname{Sym}_{0,d_X}\ni (m,B)\mapsto \delta_{(m,B)}\in \mathcal{P}_1(\mathbb{R}^{d_X}\times \operatorname{Sym}_{0,d_X})$, where $\mathcal{P}_1(\mathbb{R}^{d_X}\times \operatorname{Sym}_{0,d_X})$ denotes the $1$-Wasserstein space on $(\mathcal{P}_1(\mathbb{R}^{d_X}\times \operatorname{Sym}_{0,d_X}),d_{2,F})$ and~\eqref{eq:transformer_pre_mapped_END} follows from the expression for $\hat{\eta}$ given in~\eqref{eq:QMF}.  Consequentially, \labelcref{eq:transformer_pre_mapped_BEGIN,eq:transformer_barycentercancellation,eq:transformer_pre_mapped_END} reduces to 
 of $\hat{g}$ in Lemma~\ref{lem:approximation_Sym}
.

\item \emph{Tallying Parameters.}
Since the quantitative version of \citep[Theorem 3.7]{kratsios2023transfer} held, namely \citep[Lemma 5.10]{kratsios2023transfer}, then we obtain the following parameter estimates
\begin{enumerate}[label=(\roman*),leftmargin=1cm]
    \item $\edim=\mathcal{N}_{\operatorname{pack}}\big(K,C_{\ldim}\,\operatorname{Vol}(\mathcal{M},g)^{1/\ldim}\big) \in \mathcal{O}(1)$,
    \item $c$ is recorded in \Cref{tab:approximationrates} as the depth of the network $\hat{f}$,
    \item The expression of $\ddim$ is recorded, in detail, atop \citep[page 46]{kratsios2023transfer} and is
    \[
            \ddim
        \le 
            \big(
                C_{K,1}^{\lceil \frac1{4\alpha}\rceil}
            \big)^{
                \log_2(\operatorname{diam}(K)) - \frac1{\alpha}\log_2(\epsilon_A/(2L C_{K:2}))
            }
        \in 
        \mathcal{O}_{K}\big(
            L/\varepsilon
        \big)
        .
    \]
    for constants $C_{K,1},C_{K,2},C_K>0$ depending only on the compact set $K$ and on the mixing function $\eta$, $\alpha=1$ as $f$ is $1$-H\"{o}lder, and where $\mathcal{O}$ suppresses a constant depending only on $K$ and on the mixing function $\eta$. 
\end{enumerate}
\end{enumerate}
\end{proof}

\subsubsection{Proof of the Main Approximation Lemma}
\label{s:Proof__ss:Universality___sss:Theorem}
For any $d\in \mathbb{N}_+$, let $\overline{\mathcal{N}_d}$ denotes the set of Gaussian measure on $\mathbb{R}^d$ equipped with the $2$-Wasserstein metric $\mathcal{W}_2$.  
The Lemmata in the previous sections, together with the main results of \cite{kratsios2021universal} and \cite{kratsios2023transfer}, are used to deduce our main approximation theoretic tool, namely, \Cref{prop:Universal_Approximation_Theorem__PathToNd}.   

\begin{proof}\textbf{of \Cref{prop:Universal_Approximation_Theorem__PathToNd}}\hfill
\begin{enumerate}[label=\emph{Step \arabic*:}, wide]
\item \emph{Bounded The Local Lipschitz Stability of $\rho$ on $f(K)$.}
By \Cref{lem:LipschitzManifold_Nd}, the map $\varrho:(\mathcal{N}_{d_X},\mathcal{W}_2)\rightarrow (\mathbb{R}^{d_X}\times \operatorname{Sym}_{0,d_X},d_{2,F})$ is continuous.  By definition of the product topology, see \citep[page 114]{MunkresTopology_2000}, the projection map $\pi:(\mathbb{R}^{d_X},\operatorname{Sym}_{0,d_X},d_{2,F})\ni (\mu,\Sigma)\rightarrow \Sigma \in (\operatorname{Sym}_{0,d_X},\|\cdot\|_F)$ is continuous.  
Since the composition of continuous functions is again continuous, then the map
\[
    g: ([0,T]\times K, 
    | \cdot | \times \|\cdot\|_T)\ni (t,y_{\cdot}) \mapsto 
    \pi\circ \varrho\circ f(t,y_{\cdot})
    \in (\mathbb{R},|\cdot|)
\]
is continuous.  By \citep[Theorem 26.5]{MunkresTopology_2000}, $\tilde{K}\eqdef g(K)$ is a compact subset of $(\mathbb{R},|\cdot|)$.

By \citep[Theorem 9.2.6 - page 130]{LaxBookLinearAlgebra_2007} the map $\lambda_{\text{min}}:(\operatorname{Sym}_{0,d_X},\|\cdot\|_F)\rightarrow (\mathbb{R},|\cdot|)$ which sends any $d_X \times d_X$ square matrix $\Sigma$ to its minimal eigenvalue $\lambda_{\text{min}}(\Sigma)$ is continuous.  Since every continuous function with compact domain achieves its minimum on its domain then, there exists some $\Sigma_0\in \tilde{K}$ minimizing $\lambda_{\text{min}}$; by which we mean that 
\begin{equation}
\label{eq:prop:Universal_Approximation_Theorem__PathToNd___minimaleigenvalueachieved}
        \lambda_{\text{min}}(\Sigma_0)
    =
        \min_{\Sigma\in g(K)}\,\lambda_{\text{min}}(\Sigma)
    <
        \infty
.
\end{equation}
Since $f$ takes values in $\mathcal{N}_{d_X}$ then $\pi\circ \varrho(f(x))$ is positive definite, for each $x\in K$.  In particular, $\lambda_{\text{min}}(\Sigma_0)>0$.  Consequentially,~\eqref{eq:prop:Universal_Approximation_Theorem__PathToNd___minimaleigenvalueachieved} implies that
\begin{equation}
\label{eq:prop:Universal_Approximation_Theorem__PathToNd___minimaleigenvalue}
        0
    <
        r
    \eqdef 
        \lambda_{\text{min}}(\Sigma_0)
    =
        \min_{\Sigma\in g(K)}\,\lambda_{\text{min}}(\Sigma)
    <
        \infty
.
\end{equation}
By \Cref{lem:LoewnerBound}, we have that there is an $r>0$ such that $r=\lambda_{\operatorname{min}}(\Sigma_t)$ for all $0\le t\le T$. 
Consequentially, \Cref{lem:LipschitzManifold_Nd} implies that the map $\varrho:(\mathcal{N}_{d_X},\mathcal{W}_2)\rightarrow (\mathbb{R}^{d_X}\times \operatorname{Sym}_{0,d_X},d_{2,F})$ is Lipschitz on $\tilde{K}$ with Lipschitz constant bounded-above by $\operatorname{Lip}\big(\varrho|f(K)\big)$.  Consequentially, we have that
\begin{equation}
\label{eq:prop:Universal_Approximation_Theorem__PathToNd___Lipschitzvarrho}
        \operatorname{Lip}\big(\varrho\circ f | K\big)
    \le 
        \operatorname{Lip}\big(\varrho|f(K)\big)
        \,
        \operatorname{Lip}\big(f|K\big)
    \le 
        \max\big\{1,\frac{\sqrt{d}}{2\sqrt{r}}\big\}
        \,
        L
.
\end{equation}
\item \emph{Approximating $\varrho \circ f$ on $K$.}
Fix $\varepsilon>0$ and fix the ``perturbed approximation error''
\begin{equation}
\label{eq:prop:Universal_Approximation_Theorem__VALUEOFDelta}
        \delta
    \eqdef 
        \min\biggl\{
                1
            ,
                \frac{\varepsilon
                }{
                \sqrt{d}2\sqrt{
                    L
    \,\sqrt{T^2+\operatorname{diam}(K)^2} + 1
                    }}
        \biggr\}
.
\end{equation}

We apply \Cref{lem:approximation_Sym} to deduce that there exists a map $\hat{F}: [0:T]\times K \rightarrow \mathbb{R}^{d_X}\times \operatorname{Sym}_{0,d_X}$ with representation of $\hat{g}$ in Lemma~\ref{lem:approximation_Sym} 
satisfying the uniform estimate\begin{equation}
\label{eq:prop:Universal_Approximation_Theorem__PathToNd___uniformestimate}
    \max_{(t,y_{\cdot}) \in [0:T]\times K }\,
        d_{2,F}(\hat{F}(t, y_\cdot),\varrho\circ f(t,y_\cdot))
    <
        \delta
.
\end{equation}
Since \Cref{lem:LipschitzManifold_Nd} showed that $\varrho$ has a locally bi-Lipschitz homeomorphism then, in particular, $\varrho^{-1}$ exists and it is Lipschitz continuous.  
Consider the $1$-thickening of $\varrho\circ f(K)$ defined by 
\[
        \bar{B}_1
    \eqdef 
        \{(m,\Sigma)\in \mathbb{R}^{d_X}\times \operatorname{Sym}_{0,d_X}:\,(\exists (t,y_{\cdot} )\in [0:T]\times K )\, d_{2,F}(f(t, y_\cdot),(m,\Sigma))\le 1\}
.
\]
Since $\delta$ was defined, in~\eqref{eq:prop:Universal_Approximation_Theorem__VALUEOFDelta}, to be at-most $1$, then~\eqref{eq:prop:Universal_Approximation_Theorem__VALUEOFDelta} implies that $
    \hat{F}(K)\subseteq \bar{B}_1
    $.
From~\eqref{eq:prop:Universal_Approximation_Theorem__PathToNd___uniformestimate} and \Cref{lem:LipschitzManifold_Nd} we deduce the uniform estimate
\allowdisplaybreaks
\begin{align}
\label{eq:Approxim_LEMMA__BEGIN}
    & \max_{(t,y_{\cdot})\in [0:T]\times K}\,
        \mathcal{W}_2\big(
                \rho^{-1}\circ F(t,y_{\cdot})
            ,
                f(t,y_{\cdot})
        \big)
\\
\nonumber
= & 
    \max_{(t,y_{\cdot})\in [0:T]\times K}\,
        \mathcal{W}_2\big(
                \varrho^{-1}\circ \hat{F}(t,y_{\cdot})
            ,
                \varrho^{-1}\circ \varrho\circ f(t,y_{\cdot})
        \big)
\\
\nonumber
\le &
    \operatorname{Lip}\big(
            \varrho^{-1}
        |
            \bar{B}_1
    \big)
    \,
    \max_{(t,y_{\cdot})\in [0:T]\times K}\,
        d_{2,F}\big(
                \hat{F}(t,y_{\cdot})
            ,
                \varrho\circ f(t,y_{\cdot})
        \big)
\\
\nonumber
\le &
    \operatorname{Lip}\big(
            \varrho^{-1}
        |
            \bar{B}_1
    \big)
    \,
    \delta
\\
\label{eq:prop:Universal_Approximation_Theorem__PathToNd___BOUND_A}
\le &
    \operatorname{Lip}\big(
            \varrho^{-1} 
        |
            \bar{B}_1
    \big)
    \,
    \frac{\varepsilon
    }{
    \sqrt{d}2\sqrt{
        L
\,\sqrt{T^2+\operatorname{diam}(K)^2} + 1
    }
    }
\\
\label{eq:Approxim_LEMMA__END}
\le & 
    \varepsilon
,
\end{align}
where~\eqref{eq:prop:Universal_Approximation_Theorem__PathToNd___BOUND_A} followed from \Cref{lem:LipschitzManifold_Nd} and the definition of $\bar{B}_1$ 
and bound~\eqref{eq:Approxim_LEMMA__END} followed from the estimate~\eqref{eq:prop:Universal_Approximation_Theorem__PathToNd___Lipschitzvarrho} for $\operatorname{Lip}( \varrho^{-1}  | \bar{B}_1 )$.  Since $\mathcal{W}_p\le \mathcal{W}_2$ for all $1\le p\le 2$, see \citep[Remark 6.6]{VillaniBook_2009} then, the estimates in~\eqref{eq:Approxim_LEMMA__BEGIN}-\eqref{eq:Approxim_LEMMA__END} imply that
\begin{align}
    \max_{(t,y_{\cdot}) \in [0:T]\times K}\,
        \mathcal{W}_p\big(
                \hat{F}(t,y), f(y)
        \big)
\le & 
    \varepsilon
,
\end{align}
\hfill\\
for all $1\le p\le 2$; as claimed.  
\item \emph{Counting Parameters.}
Using $\delta$, as defined in~\eqref{eq:prop:Universal_Approximation_Theorem__PathToNd___uniformestimate}, in the place of $\varepsilon$ in \Cref{tab:approximationrates} and noting that $\delta^{-1} \in \mathcal{O}(\varepsilon^{-1})$ yields the conclusion.  
\end{enumerate}
\end{proof}

\section*{Acknowledgments}
AK acknowledges financial support from an NSERC Discovery Grant No.\ RGPIN-2023-04482 and their McMaster Startup Funds. AK and XY acknowledge that resources used in preparing this research were provided, in part, by the Province of Ontario, the Government of Canada through CIFAR, and companies sponsoring the Vector Institute.


\begin{appendices}  
\section{Supplementary Material}
\label{S:Supplement}


\subsection{Summary of Notation}
This section serves as a reference, which records the notation used throughout our manuscript.  In what follows, $N,I,J\in \mathbb{N}_+$, $A$ is an arbitrary $I\times J$ matrix, $x,x_1,\dots,x_J\in \mathbb{R}^N$, $f$ is an ar\-bi\-trary real-valued function on $\mathbb{R}$, and $t\in \mathbb{R}$.
\begin{enumerate}[itemsep=3pt]
    \item \textit{Componentwise Composition:} $\smash{f\bullet x\eqdef (f(x_n))_{n=1}^N}$ for any $N\in \mathbb{N}_+$.
    \item \textit{Rectified Linear Unit (ReLU):} $\operatorname{ReLU}:\mathbb{R}\rightarrow\mathbb{R}$ given by $\smash{\operatorname{ReLU}(t)\eqdef \max\{0,t\}}$.
    \item \textit{Rowwise Product:} $\smash{v\odot X\eqdef (v_i\,X_{i,j})_{i=1,\dots,N,\,j=1,\dots,d}}$.
    \item \textit{Softmax Function:} $\smash{\operatorname{Softmax}(x)\eqdef (e^{x_n}/\sum_{i=1}^N\,e^{x_i})_{n=1}^N}$.
    \item \textit{Sparsity:} $\smash{\|(A_{i,j})_{i,j=1}^{I,J}\|_0\eqdef \#\{A_{i,j}\neq 0:\,i,j=1,\dots,I\}}$.
    \item \textit{Vector Concatenation:} $\smash{\oplus_{j=1}^N x_j\eqdef (x_1,\dots,x_J)^{\top}}$ is the $J\times N$-matrix whose $j^{th}$ row is $x_j$.
    \item \textit{Vectorization:} $\smash{\operatorname{vec}((A_{i,j})_{i,j=1}^{I,J})\eqdef (A_{1,1},\dots,,A_{I,1},\dots,A_{1, J}, \dots, A_{I,J})}$.
    \item \textit{Euclidean Norm:} 
    $\smash{\|A\| = (\sum_{i,j=1}^{I,J} |A_{i,j}|^2 )^{1/2}}$. 
    \item \textit{Supremum Norm:} 
    $\smash{\|A\|_\infty=\sup \{ A_{i,j} : \, i=1,..., I, j=1, ..., J \}}$. 
\end{enumerate}

\subsection{Additional Background}
\label{s:AdditionalBackground}
\subsubsection{Target Space - The \texorpdfstring{$2$-Wasserstein}{2-Wasserstein} Space of Probability Measures}
\label{s:Setting__ss:OutputSpace}

Fix $1 \leq p \leq 2$. The $p$-Wasserstein space $\mathcal{P}_p(\mathbb{R}^{d_X})$ consists of all probability measures on $\mathbb{R}^{d_X}$ with finite second moment; i.e.\ $\mathbb{P}\in \mathcal{P}_p(\mathbb{R}^{d_X})$ if
$
\mathbb{E}_{X\sim \mathbb{P}}[\|X\|^p]<\infty
.
$
The $2$-Wasserstein metric $\mathcal{W}_p$ on $\mathbb{P}_p(\mathbb{R}^{d_X})$ is given by minimizing the optimal cost of transporting mass between any two measures $\mathbb{P}$ and $\mathbb{Q}$ via randomized transport plans.  This can be formalized by the Kanotovich problem
\[
        \mathcal{W}_p(\mathbb{P}, \mathbb{Q})^p
    \eqdef
        \inf_{(X_1,X_2)\sim \pi;\,X_1\sim \mathbb{P},\,X_2\sim \mathbb{Q}}\,
            \mathbb{E}_{\pi}[\|X_1-X_2\|^p]
    .
\]
Generally, the $p$-Wasserstein distance between measures can be computationally taxing, requiring a super-quadratic complexity to compute \cite{TBD}.  However, the $2$-Wasserstein distance is not always computationally intractable as for instance in the case where $\mathbb{P}=N(\mu^{(1)},\Sigma^{(1)})$ and $\mathbb{Q}=N(\mu^{(2)},{\Sigma^{(2)}})$ are Gaussian measures on $\mathbb{R}^{d_X}$, \cite{dowson1982frechet} showed that it admits the following closed-form expression
\begin{align}
        \mathcal{W}_2(\mathbb{P},\mathbb{Q})^2
    =
            \| \mu^{(1)} - \mu^{(2)} \|_2^2 
        + 
            \operatorname{tr}(\Sigma^{(1)}) + \operatorname{tr}(\Sigma^{(2)})
        -  
            2
            \,
            \operatorname{tr}
            \big(
                {\Sigma^{(2)}}^{1/2}
                \Sigma^{(1)}
                {\Sigma^{(2)}}^{1/2}
            \big)^{1/2}
    , \label{id:WassersteinExplicit}
\end{align}
provided that $\Sigma^{(1)}$ and $\Sigma^{(2)}$ are invertible; in particular, they are positive-definite.

\subsection{Auxiliary Results}
\label{s:Proof__ss:AuxResults}

\subsubsection{Isometric Copies of Every Compact Riemannian Manifold in the Path Space}

\begin{proposition}[Isometric Copies of Compact Riemannian Manifolds]
\label{prop:Non_Vaccousness}
\sloppy Let $\mathcal{X}$ be a compact metric space.  For every $T>0$ and each $d\in \mathbb{N}_+$, there is a compact subset $K\subseteq C([0:T],\mathbb{R}^{d_Y})$ and an isometry from $\mathcal{X}$ onto $K$.
\end{proposition}
\begin{proof}\textbf{of \Cref{prop:Non_Vaccousness}}\hfill\\
Since $(\mathcal{X},d_{\mathcal{X}})$ is compact then the Kuratowksi embedding $\phi_1:\,x\mapsto d_{\mathcal{X}}(\cdot,x)$ is an isometric embedding of $\mathcal{X}$ into the Banach space $C(\mathcal{X})$ with its uniform norm (since $\mathcal{X}$ is compact); where $d_{\mathcal{X}}$ denotes the metric on $\mathcal{X}$.  By the Banach-Mazur theorem, there exists an isometric embedding of $\phi_2:C(\mathcal{X})\rightarrow C([0,1])$.  
Since the map $\psi_3:C([0:1])\rightarrow C([0:T],\mathbb{R}^{d_Y})$, given by $f\mapsto (f(\cdot /T),0,\dots,0)$, is an isometric embedding; then $\smash{\phi\eqdef \phi_3\circ\phi_2\circ \phi_1}$ is an isometric embedding of $\mathcal{X}$ into $C([0:T],\mathbb{R}^{d_Y})$.
\end{proof}

This section records an auxiliary lemma due to Iosif Pinelis, \cite{LocalLipschitz_Pinelis}.  We include the result and its proof here to keep our manuscript self-contained.  
\begin{lemma}[\cite{LocalLipschitz_Pinelis}]
\label{lem:LipschitzManifold_Nd}
Fix $d\in \mathbb{N}_+$ and $R,r>0$.  For every $m^{(1)},m^{(2)}\in \mathbb{R}^d$ and each $d\times d$ symmetric positive semi-definite matrix $A,B$ satisfying: $\|A\|_F,\|B\|_F\le R$ and $A-r\cdot I_d$ and $B-r\cdot I_d$ are positive semi-definite
\footnote{
    I.e. $A,B\ge r\cdot I_d$ where $\ge $ is the partial ordering on the set of $d\times d$-dimensional symmetric positive-definite matrices given by $A\ge B$ if and only if $A-B$ is positive semi-definite.
} 
then the following lower-bound holds
\begin{align*}
            \frac1{
                \min\{1,\sqrt{d}(2\sqrt{R})\}
            }
            \,
            \sqrt{\|m^{(1)}-m^{(2)}\|^2+\|A-B\|_F^2}
        \le &
            \mathcal{W}_2\big(
                \mathcal{N}(m^{(1)},A)
            ,
                \mathcal{N}(m^{(2)},B)
            \big)
    .
\end{align*}
Moreover, the following upper-bound also holds
\begin{align*}
            \mathcal{W}_2\big(
                \mathcal{N}(m^{(1)},A)
            ,
                \mathcal{N}(m^{(2)},B)
            \big)
        \le &
            \max\Big\{1,\frac{\sqrt{d}}{2\sqrt{r}}\Big\}
                \,
            \sqrt{\|m^{(1)}-m^{(2)}\|^2+\|A-B\|_F^2}
.
\end{align*}
In particular, the map $\varrho:(\overline{\mathcal{N}}_d,\mathcal{W}_2)
\rightarrow
(\mathbb{R}^d\times \operatorname{Sym}_{0,d},\|\cdot\|\times \|\cdot\|_F)
$ is locally-Lipschitz; where $\|\cdot\|\times \|\cdot\|_F$ denotes the product of the Euclidean norm on $\mathbb{R}^d$ and the Fr\"{o}benius norm on the space of $d\times d$-dimensional symmetric positive semi-definite matrices $\operatorname{Sym}_{0,d}$.
\end{lemma}
\begin{proof}\textbf{of Lemma~\ref{lem:LipschitzManifold_Nd}}\hfill\\
By \citep[Proposition 7]{GivensShortt_ClassofWassersteinMetricforProbDists_1984} the $2$-Wasserstein distance between $\mathcal{N}(m^{(1)},A)$ and $\mathcal{N}(m^{(2)},B)$ satisfies
\begin{equation*}
	\mathcal{W}_2(\mathcal{N}(m^{(1)},A),\mathcal{N}(m^{(2)},B))=\sqrt{\|m^{(1)}-m^{(2)}\|^2+W_2(\mathcal{N}(0,A),\mathcal{N}(0,B))^2}.
 \label{10}
\end{equation*}
Thus, $\|m^{(1)}-m^{(2)}\|\le \mathcal{W}_2(\mathcal{N}(m^{(1)},A),\mathcal{N}(m^{(2)},B))$.  Moreover, for any unit vector $u\in\R^n=\R^{n\times1}$ we have
\begin{align*}
	\mathbb{E}_{X\sim \mathcal{N}(0,A),\,Y\sim \mathcal{N}(0,B)}\|X-Y\|^2
    \ge &
        \mathbb{E}_{X\sim \mathcal{N}(0,A),\,Y\sim \mathcal{N}(0,B)}(u^\top X-u^\top Y)^2
    \\
    \ge& 
        (\sqrt{u^\top A u}-\sqrt{u^\top B u})^2,
\end{align*}
where the last inequality holds since $u^\top X$ and $u^\top Y$ have Gaussian law with zero-mean random with respective variances $u^\top A u$ and $u^\top B u$.
Again using the inequality $\|m^{(1)}-m^{(2)}\|\le \mathcal{W}_2(\mathcal{N}(m^{(1)},A),\mathcal{N}(m^{(2)},B))$, for any unit vector $u\in\R^n=\R^{n\times1}$, we have
\begin{equation*}
\begin{aligned}
	\mathcal{W}_2(\mathcal{N}(a,A),\mathcal{N}(b,B))&\ge \mathcal{W}_2(\mathcal{N}(0,A),\mathcal{N}(0,B)) \\ 
	&\ge|\sqrt{u^\top A u}-\sqrt{u^\top B u}| \\ 
	&=\frac{|u^\top A u-u^\top B u|}{\sqrt{u^\top A u}+\sqrt{u^\top B u}} \\ 
	&\ge\frac{|u^\top(A-B)u|}{\sqrt{\|A\|}+\sqrt{\|B\|}} \\  
	&=\frac{\|A-B\|}{\sqrt{\|A\|}+\sqrt{\|B\|}} 
\end{aligned}
\end{equation*}
for some unit vector $u\in\R^n=\R^{n\times1}$, where $\|M\|$ is the spectral norm of a matrix $M$. 
So, 
\begin{equation*}
	\|A-B\|\le(\sqrt{\|A\|}+\sqrt{\|B\|\,})\, \mathcal{W}_2(\mathcal{N}(m^{(1)},A),\mathcal{N}(m^{(2)},B)). \label{30}
\end{equation*}
The conclusion follows upon combining $\|m^{(1)}-m^{(2)}\|\le \mathcal{W}_2(N(m^{(1)},A),N(m^{(2)},B))$ and 
$\|A-B\|\le (\sqrt{\|A\|}+\sqrt{\|B\|}) \mathcal{W}_2(\mathcal{N}(m^{(1)},A),\mathcal{N}(m^{(2)},B))$ together with the observation that $\|\cdot\|_{\lambda:2}\le \|\cdot\|_F\le \sqrt{d}\|\cdot\|_{\lambda:2}$; where 
\[
\|A\|_{\lambda:2}\eqdef \sqrt{\max_{i=1,\dots,d}\,\lambda_{\text{max}}(C^{\top}C})
\]
denotes the spectral norm on the set of $d\times d$ matrices and where $\lambda_{\text{max}}(C^{\top}C)$ denotes the largest eigenvalue $C^{\top}C$ for a given $d\times d$ matrix $C$. 
\end{proof}

\end{appendices}


\bibliography{Bookkeaping/2_References}

\end{document}